\documentclass{article} 
\usepackage{hyperref}
\pdfoutput=1
\usepackage[letterpaper, margin=1in]{geometry}

\usepackage{graphicx}
\usepackage{wrapfig}

\usepackage[T1]{fontenc}
\usepackage{amsfonts}
\usepackage{amsmath}
\usepackage{amssymb}
\usepackage{amsthm}
\usepackage{mathrsfs} 
\usepackage{calligra}
\usepackage{thmtools}
\usepackage{thm-restate}

\DeclareMathAlphabet{\mathcalligra}{T1}{calligra}{m}{n}
\DeclareFontShape{T1}{calligra}{m}{n}{<->s*[2.2]callig15}{}
\newcommand{\scripty}[1]{\ensuremath{\mathcalligra{#1}}}
\newcommand{\rtnd}{\ensuremath{\scripty{r}}}

\newcommand{\aEL}{\ensuremath{\mathrm a}}
\newcommand{\bEL}{\ensuremath{\mathrm b}}
\newcommand{\cEL}{\ensuremath{\mathrm c}}
\newcommand{\dEL}{\ensuremath{\mathrm d}}
\newcommand{\eEL}{\ensuremath{\mathrm e}}
\newcommand{\fEL}{\ensuremath{\mathrm f}}
\newcommand{\gEL}{\ensuremath{\mathrm g}}
\newcommand{\hEL}{\ensuremath{\mathrm h}}
\newcommand{\iEL}{\ensuremath{\mathrm i}}

\newcommand{\modularminA}[1]{\min((#1)^{T}(A),1)}

\newcommand{\lex}[1]{{\ensuremath \breve #1}}

\usepackage[T3,T1]{fontenc}
\DeclareSymbolFont{tipa}{T3}{cmr}{m}{n}
\DeclareMathAccent{\invbreve}{\mathalpha}{tipa}{16}

\usepackage{subcaption}

\usepackage[vlined,ruled]{algorithm2e}
\usepackage{algorithmic}

\newtheorem{theorem}{Theorem}[section]
\newtheorem{corollary}{Corollary}[theorem]
\newtheorem{definition}[theorem]{Definition}
\newtheorem{lemma}[theorem]{Lemma}
\newtheorem{proposition}[theorem]{Proposition}
\newtheorem{example}[theorem]{Example}

\newtheorem{exercise}{Exercise}[section]

\DeclareMathOperator*{\argmax}{argmax}
\DeclareMathOperator*{\argmin}{argmin}

\newcommand{\set}[1]{\ensuremath{{\left\{ #1 \right\}}}}

\newcommand{\lovasz}{Lov\'asz}
\newcommand{\surp}{\ensuremath{\mathscr S}}

\usepackage{authblk}

\title{Deep Submodular Functions}
\author[1,2]{Jeffrey A. Bilmes}
\author[1]{Wenruo Bai}
\affil[1]{Department of Electrical Engineering, University of Washington, Seattle, 98195}
\affil[2]{Department of Computer Science and Engineering, University of Washington, Seattle, 98195}

\usepackage[usenames,dvipsnames,dvinames]{xcolor}
\hypersetup{bookmarks, colorlinks=true, citecolor=Violet,linkcolor=Mahogany,urlcolor=blue, pdftitle={Deep Submodular Functions}, pdfauthor={Jeff A. Bilmes, http://melodi.ee.washington.edu/people/bilmes, Wenruo Bai}, pdfsubject={submodular functions}, pdfkeywords={submodular, discrete optimization, machine learning, combinatorial optimization,Bilmes}}

\begin{document}

\maketitle

\begin{abstract}

  We start with an overview of a class of submodular functions called
  SCMMs (sums of concave composed with non-negative modular functions
  plus a final arbitrary modular).  We then define a new class of
  submodular functions we call {\em deep submodular functions} or
  DSFs.  We show that DSFs are a flexible parametric family of
  submodular functions that share many of the properties and
  advantages of deep neural networks (DNNs), including many-layered
  hierarchical topologies, representation learning, distributed
  representations, opportunities and strategies for training, and
  suitability to GPU-based matrix/vector computing.  DSFs can be
  motivated by considering a hierarchy of descriptive concepts over
  ground elements and where one wishes to allow submodular interaction
  throughout this hierarchy. In machine learning and data science
  applications, where there is often either a natural or an
  automatically learnt hierarchy of concepts over data, DSFs therefore
  naturally apply.  Results in this paper show that DSFs constitute a
  strictly larger class of submodular functions than SCMMs, thus
  justifying their mathematical and practical utility. Moreover, we
  show that, for any integer $k>0$, there are $k$-layer DSFs that
  cannot be represented by a $k'$-layer DSF for any $k'<k$. This
  implies that, like DNNs, there is a utility to depth, but unlike
  DNNs (which can be universally approximated by shallow networks),
  the family of DSFs strictly increase with depth.  Despite this
  property, however, we show that DSFs, even with arbitrarily large
  $k$, do not comprise all submodular functions.  We show this using a
  technique that ``backpropagates'' certain requirements if it was the
  case that DSFs comprised all submodular functions.  In offering the
  above results, we also define the notion of an antitone
  superdifferential of a concave function and show how this relates to
  submodular functions (in general), DSFs (in particular), negative
  second-order partial derivatives, continuous submodularity, and
  concave extensions. To further motivate our analysis, we provide
  various special case results from matroid theory, comparing DSFs
  with forms of matroid rank, in particular the laminar
  matroid. Lastly, we discuss strategies to learn DSFs, and define the
  classes of deep supermodular functions, deep difference of
  submodular functions, and deep multivariate submodular functions,
  and discuss where these can be useful in applications.











\end{abstract}

\tableofcontents

\section{Introduction}
\label{sec:introduction}

Submodular functions are attractive models of many physical processes
primarily because they possess an inherent naturalness to a wide
variety of problems (e.g., they are good models of diversity,
information, and cooperative costs) while at the same time they enjoy
properties sufficient for efficient optimization. For example,
submodular functions can be minimized without constraints in
polynomial time~\cite{fujishige2005submodular} even though they lie
within a $2^n$-dimensional cone in $\mathbb R^{2^n}$ 
and are parameterized,
in their most general form,
with a corresponding $2^n$ independent degrees of freedom.
Moreover, while
submodular function maximization is NP-hard, submodular maximization
is one of the easiest of the NP-hard problems since constant factor
approximation algorithms are often available --- e.g., in the
cardinality constrained case, the classic $1-1/e$ result of
Nemhauser~\cite{nem78} via the
greedy algorithm. Other problems also have guarantees, such as
submodular maximization subject to knapsack or multiple matroid
constraints
\cite{calinescu2011maximizing,buchbinder2014submodular,lee2009non,rishabh2013-submodular-constraints,iyer2013-fast-submodular-semigradient}.

Submodular functions are becoming increasingly important in the field
of machine learning. In recent years, submodular functions have been
used for representing diversity functions for the purpose of data
summarization~\cite{linacl}, for use as structured convex
norms~\cite{bach2010structured}, for energy functions in tree-width
unconstrained probabilistic
models~\cite{gillenwater2012near,Kulesza2012,iyer2015-spps,gotovos15sampling},
useful in computer vision~\cite{kolmogorov2004energy},
feature~\cite{liu2013submodular} and dictionary
selection~\cite{das2011submodular}, viral
marketing~\cite{hartline2008optimal} and influence modeling in social
networks~\cite{kkt03}, information cascades~\cite{leskovec2007cost}
and diffusion modeling~\cite{rodriguez2012submodular},
clustering~\cite{narasimhan2005-q}, and active and semi-supervised
learning~\cite{guillory2011-active-semisupervised-submodular}, to name
just a few.  There also have been significant contributions from the
machine learning community purely on the mathematical and algorithmic
aspects of submodularity. This includes algorithms for optimizing
non-submodular functions via the use of submodularity
\cite{narasimhan2005-subsup,krause08robust,jegelka2011-nonsubmod-vision,rkiyeruai2012},
strategies for optimizing submodular functions subject to both
combinatorial~\cite{rkiyersemiframework2013} and submodular level-set
constraints~\cite{rishabh2013-submodular-constraints}, and so on.


One of the critical problems associated with utilizing submodular
functions in machine learning and data science contexts is selecting
which submodular function to use, and given that submodular functions
lie in such a vast space with $2^n$ degrees of freedom, it is a
non-trivial task to find one that works well, if not optimally. One
approach is to attempt to learn the submodular function based on
either queries of some form or based on data. This has led to results,
mostly in the theory community, showing how learning submodularity can
be harder or easier depending on how we judge what is being
learnt. For example, it was shown that learning submodularity in the
PMAC setting is fairly hard~\cite{balcan10learning} although in some
cases things are a bit easier~\cite{feldman2013optimal}.  Learning can
be made easier if we restrict ourselves to learn within only a
subfamily of submodular functions. For example,
in~\cite{sipos2012large,hui2012-submodular-shells-summarization}, it
is shown empirically that one can effectively learn mixtures of submodular
functions using a max-margin learning framework --- here the
components of the mixture are fixed and it is only the mixture
parameters that are learnt, leading often to a convex optimization
problem. In some cases, computing gradients of the convex problem can
be done using submodular
maximization~\cite{hui2012-submodular-shells-summarization}, while in
other cases, even a gradient requires minimizing a difference of two
submodular functions~\cite{sebastian2014-submod-image-sum}.

Learning over restricted families rather than over the entire cone is
desirable for the same reasons that any form of regularization in
machine learning is useful. By restricting the family over which
learning occurs, it decreases the complexity of the learning problem,
thereby increasing the chance that one finds a good model within that
family. This can be seen as a classic bias-variance tradeoff, where
increasing bias can reduce variance. Up to now, learning over
restricted families has apparently (to the authors' knowledge) been
limited to learning mixtures over fixed components. This can be 
limited if the components are restricted, and if not might require a
very large number of components.  Therefore, there is a need for a
richer and more flexible parametric family of submodular functions
over which learning is not only still possible but ideally relatively
easy. See Section~\ref{sec:learn-deep-subm} for further
discussion on learning submodular functions.

In this paper, we introduce a new family of submodular functions that
we term ``deep submodular functions,'' or DSFs. DSFs strictly
generalize, as we show below, many of the kinds of submodular
functions that are useful in machine learning contexts. These include
the so-called ``decomposable'' submodular functions, namely those
that can be represented as a sum of concave composed with modular
functions~\cite{stobbe10efficient}.

We describe the family of DSFs and place them in the context of the
general submodular family. In particular, we show that DSFs strictly
generalize standard decomposable functions, thus theoretically
motivating the use of deeper networks as a family over which to learn.
Moreover, DSFs can represent a variety of complex submodular functions
such as laminar matroid rank functions. These matroid rank functions
include the truncated matroid rank
function~\cite{goemans2009approximating} that is often used to show
theoretical worst-case performance for many constrained submodular
minimization problems.  We also show, somewhat surprisingly, that like
decomposable functions, DSFs are unable to represent all possible
cycle matroid rank functions. This is interesting in and of itself
since there are laminar matroids that can not be represented by cycle
matroids.  
On the other hand, we show that the more general DSFs share
a variety of useful properties with decomposable functions.  Namely,
that they: (1) can leverage the vast amount of practical work on
feature engineering that occurs in the machine learning community and
its applications; (2) can operate on multi-modal data if the data can
be featurized in the same space; (3) allow for training and testing on
distinct sets since we can learn a function from the feature
representation level on up, similar to the work
in~\cite{hui2012-submodular-shells-summarization}; and (4) are useful
for streaming
\cite{badanidiyuru2014streaming,kumar2015fast,chekuri2015streaming}
and parallel
\cite{mirzasoleiman2015distributed,barbosa2015power,barbosa2016new}
optimization since functions can be evaluated without requiring
knowledge of or access to the entire ground set.  These advantages are
made apparent in Section~\ref{sec:background}.


Interestingly, DSFs also share certain properties with deep neural
networks (DNNs), which have become widely popular in the machine
learning community. For example, DNNs with weights that are strictly
non-negative correspond to a DSF. This suggests, as we show in
Section~\ref{sec:learn-deep-subm}, that it is possible to develop a
learning framework over DSFs leveraging DNN learning
frameworks. Unlike standard deep neural networks, which typically are
trained either in classification or regression frameworks, however,
learning submodularity often takes the form of trying to adjust the
parameters so that a set of ``summary'' data sets are offered a high
value. We therefore extend the max-margin learning framework
of~\cite{sipos2012large,hui2012-submodular-shells-summarization} to
apply to DSFs. Our approach can be seen as a max-margin learning
approach for DNNs but restricted to DSFs.

We offer a list of applications for DSFs in machine learning and data
science in Section~\ref{sec:applications}.

\section{Background and Motivation}
\label{sec:background}


Submodular functions are discrete set functions that have the {\em
  property of diminishing returns}.  Given a finite size-$n$ set of
objects $V$ (the {\em ground set}), where each $v \in V$ is a distinct
element.  A valuation set function $f:2^V \rightarrow \mathbb{R}$ that
returns a real value for any subset $X \subseteq V$ is said to be {\em
  submodular} if for all $X \subseteq Y$ and $v \notin Y$ the
following inequality holds: $f(X \cup \{ v \}) - f(X) \geq f(Y \cup \{
v \}) - f(Y)$.  This means that the incremental value (or gain) of
adding another sample $v$ to a subset decreases when the context in
which $v$ is considered grows from $X$ to $Y$. We can define the {\em
  gain of $v$ in the context of $X$} as $f(v | X) \triangleq f(X \cup
\{ v \}) - f(X)$. Thus, $f$ is submodular if $f(v|X) \geq f(v|Y)$.  If
the gain of $v$ is identical for all different contexts i.e., $f(v|X)
= f(v|Y), \forall X,Y \subseteq V$ and $\forall v \in V$, then the function is said to be
{\em modular}.  A function might also have the property of being
normalized ($f(\emptyset) = 0$) and monotone non-decreasing ($f(X)
\leq f(Y)$ whenever $X \subseteq Y$).  If $f$ is a normalized monotone
non-decreasing function, then it is often referred to as a {\em
  polymatroid function}
\cite{cunningham1985optimal,cunningham1983decomposition,lovasz1980matroid}
\footnote{\lovasz{} in 1980 uses the same definition, but also asked
  for integrality which Cunningham did not require.}  because it
carries identical information to that of a polymatroidal
polyhedron. If the negation of $f$, $-f$, is submodular, then $f$ is
called {\em supermodular}.  If $m$ is a normalized modular function,
it can be written as a sum of singleton values $m(X) = \sum_{x \in X}
m(x)$ and, moreover, is seen simply as a vector $m \in \mathbb R^V$.

A very simple example of a submodular function can be described using
an urn containing a set of balls and a valuation function that counts
the number of colors present in the urn. Such a function, therefore,
measures only the diversity of ball colors in the urn, rather than
ball quantity. We are motivated by applications where we wish to build
models of information and diversity over data sets, in which case $V$
is a ground set of data items. Each $v \in V$, in such case, might be
a distinct data sample --- for example, either a word, n-gram,
sentence, document, image, video, protein, genome, sensor reading, a
machine learning system's input-output training pair, or even a highly
structured irregularly sized object such as a tree or a graph. It is
also desirable for $V$ to be a set of heterogeneous data objects, such
where $v_1 \in V$ may be an image and $v_2 \in V$ may be a document.

There are many useful classes of submodular functions. One of the more
widely used such function are those that, for the present purposes, 
we refer to a ``graph based,'' since they are parameterized by a
weighted graph. Graph-based methods have a long history in many
applications of machine learning and natural language processing
(NLP), e.g.,
\cite{mihalcea05,navigli07,alexandrescu09,silberer10,subra10,lang11,liu13,rasmara13,yong2014}.
Work in this field is relevant to any graph-based submodular functions
parameterized by a weighted graph $G=(V,E,w)$, where $V$ is a set of
nodes (corresponding to the ground set), $E$ is a set of edges, and
$w: E \to \mathbb R_+$ is a set of non-negative edge weights
representing associations (e.g., affinity or similarity) between the
corresponding elements. Graph-based submodular functions include the
classic {\em graph cut} function $f(X) = \sum_{x \in X, y \in V
  \setminus X} w(x,y)$, but also the {\em monotone graph cut} function
$f(X) = \sum_{x \in X, y \in V } w(x,y)$, the {\em saturated graph
  cut} function \cite{Lin2011} $f(X) = \sum_{v \in V} \min(C_v(X),
\alpha C_v(V))$ where $\alpha \in (0,1)$ is a hyperparameter and where $C_v(X) = \sum_{x
  \in X} w(v,x)$. Another widely used graph-based function is the {\em facility
  location} function
\cite{mirchandani1990discrete,cornuejols1977uncapacitated,nem78,fisher1978analysis}
$f(X) = \sum_{v \in V} \max_{x \in X} w(x,v)$, the maximization of
which is related to the $k$-median problem
\cite{badanidiyuru2014streaming,kaufman2009finding}. It is also useful 
and learn conic mixtures of graph based functions as done in
\cite{hui2012-submodular-shells-summarization}.

An advantage of graph-based submodular functions is that they can be
instantiated very easily, using only a similarity score between two
objects $v_1,v_2 \in V$ that does not require metricity or any
property (such as non-negative definiteness of the associated matrix,
required for using a determinantal point process
(DPP)~\cite{Gillenwater2012,Kulesza2012,gillenwater2012near,Agarwal2014,Gillenwater2014}
other than non-negativity.  A drawback of graph-based functions is
that building a graph over $n$ samples has complexity $O(n^2)$ as has
querying the function itself, something that does not scale to very
large ground set sizes (although there are many approaches to more
efficient sparse graph construction
\cite{chen09,jebara09,chen09,ozaki11,wang12,zhang13} to improve upon
this complexity).  Moreover, it is difficult to add elements to $V$ as
it requires $O(n)$ computation for each addition. For machine learning
applications, moreover, it is difficult with these functions to train
on a training set that may generalize to a test set
\cite{hui2012-submodular-shells-summarization}.



\section{Sums of Concave Composed with Modular Functions (SCMMs)}
\label{sec:sums-conc-comp}

A class of submodular functions~\cite{stobbe10efficient} used in
machine learning are the so-called ``decomposable functions.''.  Given
a set of non-negative modular functions $m_i: V \to \mathbb R_+$, a
corresponding set of non-negative monotone non-decreasing normalized
(i.e., $\phi(0) = 0$) concave functions $\phi_i: [0, m_i(V)] \to
\mathbb R_+$, and a final normalized but otherwise arbitrary modular
function $m_\pm: V \to \mathbb R$, consider the class of functions $g
: 2^V \to \mathbb R_+$ that take the following form:
\begin{align}
g(A) = \sum_{i} \phi_i(m_i(A)) + m_\pm(A) = \sum_i
\phi_i\left(\sum_{a \in A} m_i(a)\right) + m_\pm(A).
\label{eq:scmm_general}
\end{align}
This class of functions is known to be
submodular~\cite{fujishige1999minimizing,fujishige2005submodular,stobbe10efficient}.
While such functions have been called ``decomposable'' in the past, in
this work we will refer to this class of functions as ``Sums of
Concave over non-negative Modular plus Modular'' (or SCMMs) in order
to avoid confusion with the term ``decomposable'' used to describe
certain graphical models
\cite{lauritzen1996graphical,golumbic2004algorithmic}.\footnote{In
  fact, the notion of decomposition used
  in~\cite{lauritzen1996graphical,golumbic2004algorithmic}, the
  graphical models community, and related to the notion of the same
  name used in \cite{cunningham1983decomposition}, can also be used to
  describe a form of decomposability of a submodular function in that
  the submodular function may be expressed as a sum of terms each one
  of which corresponds to a clique in a graph, and where the graph is
  triangulated, but where the terms need not be a concave composed
  with a modular function.  Hence, without this switch of terminology,
  one reasonably could speak of ``decomposable decomposable submodular
  functions.''}

SCMMs have been shown to be quite flexible~\cite{stobbe10efficient},
being able to represent a surprisingly diverse set of functions. For
example, consider the bipartite neighborhood function, which is
defined using a bipartite graph $G=(V,U,E,w)$ with
$E \subseteq V \times U$ being a set of edges between elements of $V$
and $U$, and where $w: U \to \mathbb R_+$ is a set of weights on
$U$. For any subset $Y \subseteq U$ we define
$w(Y) = \sum_{y \in Y} w(y)$ as the sum of the weights of the elements
$Y$.  The bipartite neighborhood function is then defined as
$g(X) = w(\Gamma(X))$, where the neighbors function is defined as
$\Gamma(X) = \{ u \in U : \exists (x,u) \in E \text{ having } x \in X
\} \subseteq U$ for $X \subseteq V$.  This can be easily written as
an SCMM as follows:
$g(X) = \sum_{u \in U} w(u) \min( | X \cap \delta u | , 1 )$ where
$\delta u \subseteq V$ are the neighbors of $u$ in $V$ --- hence
$m_u(X) = | X \cap \delta u |$ is a modular function and
$\phi_u(\alpha) = \min(1,\alpha)$ is concave.  When all the
weights are unity, this is also equivalent to the set cover function
$g(X) = | \bigcup_{x \in X} \Gamma(x) |$ where 
the operation $\min |X|$ s.t.\ $g(X) = |U|$ 
attempts to cover a set
$U$ by a small set of subsets $\set{ \Gamma(x) : x \in X }$. With such
functions, it is possible to represent graph cut as follows: $g(X) =
f(X) + f(V \setminus X) - f(V)$, a sum of an SCMM and a complemented
SCMM. 
It is shown in~\cite{jegelka2011-fast-approx-sfm} that any SCMM can be
represented with a graph cut function that might optionally utilize
additional auxiliary variables that are first minimized over.

SCMMs can represent other functions as well, such as multiclass
queuing system functions~\cite{itoko2007computational,stobbe13thesis},
functions of the form $f(A) = m_1(A) \phi(m_2(A))$ where
$m_1,m_2 : V \to \mathbb R_+$ are both non-negative modular functions,
and $\phi : \mathbb R \to \mathbb R$ is a non-increasing concave
function.  Another useful instance is the probabilistic coverage
function~\cite{el2009turning} where we have a set of topics, indexed
by $i$, and $V$ is a set of documents.  The function, for topic $u$,
takes the form $f_u(A) = 1 - \prod_{a \in A} (1 - p(u|a))$ where
$p(u|a)$ is the probability of topic $u$ for document $a$ according to
some model. This function can be written as
$f_u(A) = 1 - \exp( - \sum_{a \in A} \log(1/(1 - p(u|a))))$ where
$\phi_u(\alpha) = 1 - \exp(-\alpha)$ is a concave function and
$m_u(A) = \sum_{a \in A} \log(1/(1 - p(u|a)))$ is modular. Hence,
probabilistic coverage is an SCMM. Indeed, even the facility location
function can be related to SCMMs. If in the facility location function
we sum over a set of concepts $U$ rather than the entire ground set
$V$ (which can be achieved, say by first clustering $V$ into
representatives $U$), the function takes the form
$g(A) = \sum_{u \in U} \max_{a \in A} w(a,u)$. A soft approximation to
the max function (softmax) can be obtained as follows:
\begin{align}
\label{eq:softmax}
\phi_{\text{smax}(\gamma,w)}(A)
\triangleq \frac{1}{\gamma}\log(\sum_{a \in
  A} \exp(\gamma w_a )).
\end{align}
We have that $\max_{a \in A} w_a = \lim_{\gamma \to \infty} \phi_{\text{smax}(\gamma,w)}(A)$
and for any finite $\gamma$, $\phi_{\text{smax}(\gamma,w)}(A)$ is a concave
over modular function. Hence, a soft concept-based facility
location function would take the form
$g_\gamma(A) = \sum_{u \in U} \phi_{\text{smax}(\gamma,w_u)}(A)$ which is
also an SCMM.

\begin{figure}[tb]
\centerline{\includegraphics[page=1,width=1.0\textwidth]{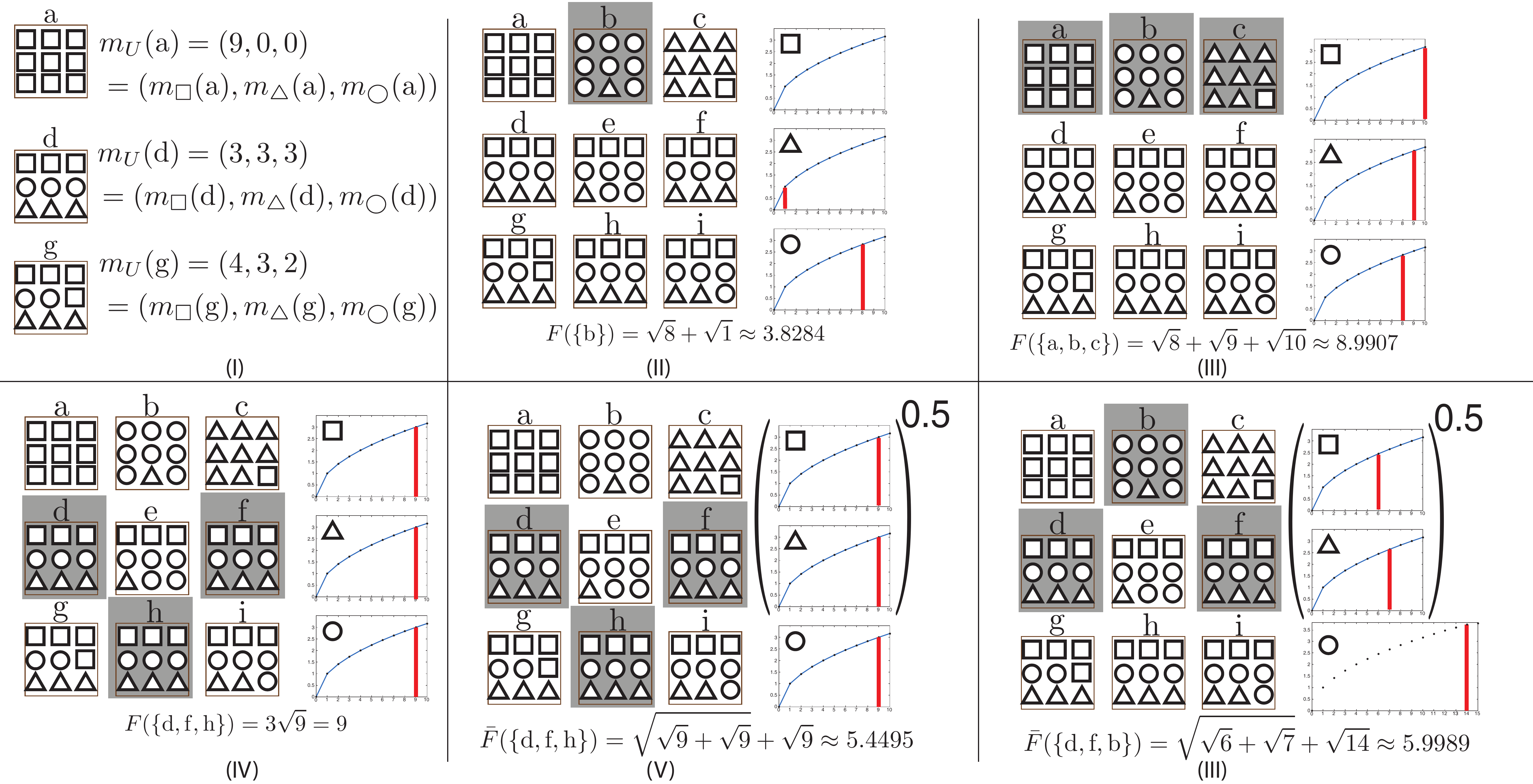}}

\caption{Illustration of SCMMs and their lack of higher-level
  interaction amongst concepts.  I: Three objects, each consisting of
   a set of shapes (one or more of $\Box$, $\triangle$, and
   $\bigcirc$), and indexed by $\set{ \aEL, \dEL, \gEL }$.  II: Nine
   objects $\set{ \aEL, \bEL, \cEL, \dEL, \eEL, \fEL, \gEL, \hEL, \iEL}$
   and selection of $\set{ \bEL }$ and valuation of $f(\bEL)$.
 III: Selection of $\set{ \aEL, \bEL, \cEL }$ and valuation of $f(\text{ \aEL,\bEL,\cEL})$.
 IV: Selection of $\set{ \dEL, \fEL, \hEL }$ and valuation of 
 $f(\set{ \dEL,\fEL,\hEL})$,
 which is the maximum value for $f$ amongst all sets of size three.
 V: Interaction amongst the non-smooth shapes causes a reduced
 valuation of $\set{ \dEL, \fEL, \hEL }$ .
 IV: With interaction amongst the non-smooth shapes, a
 new size-three maximum is achieved
with set $\set{ \bEL, \dEL, \fEL }$.
}
\label{fig:scmm_fig1}
\end{figure}

Equation~\eqref{eq:scmm_general} allows for a final arbitrary modular
function $m_\pm$ without which the function class would be strictly
monotone non-decreasing and trivial to unconstrainedly minimize.
Allowing an arbitrary modular function to apply at the end means the
function class need not be monotone and hence finding the minimizing
set is non-trivial. Because of their particular form, however, SCMMs
yield efficient algorithms for fast
minimization~\cite{stobbe10efficient,
  jegelka2013reflection,nishihara2014convergence}. Moreover, it
appears that there is little loss of generality in handling the
non-monotonicty separately from the polymatroidality, as any
non-monotone submodular function can easily be written as a sum of a
totally normalized polymatroid function plus a modular
function~\cite{cunningham1983decomposition,cu84}. To see this,
consider any arbitrary submodular function $f$ and write it as
$f(A) = \bigl( f(A) - \sum_{a \in A} f(a | V \setminus \set{a} )\bigr)
+ \sum_{a \in A} f(a | V \setminus \set{a} )$, the first term
$f(A) - \sum_{a \in A} f(a | V \setminus \set{a} )$ is a polymatroid
function and the second term is modular.



\subsection{Feature Based Functions}
\label{sec:feat-based-funct}

A particularly useful way to view SCMMs for machine learning and
data science applications is when data objects are embedded in a
``feature'' space indexed by a finite set $U$.  Suppose we have a set
of (possibly multi-modal) data objects $V$ each of which can be
described by an embedding into feature space $\mathbb R_+^U$ where
each $u \in U$ can be thought of as a possible feature, concept, or
attribute of an object. Each object $v \in V$ is represented by a
non-negative feature vector $m_U(v) \triangleq (m_{u_1}(v),
m_{u_2}(v), \dots, m_{u_{|U|}}(v)) \in \mathbb R_+^U$. Each feature $u
\in U$ also has an associated normalized monotone non-decreasing concave
function $\phi_u : [0, m_u(V)] \to \mathbb R_+$ and a
non-negative importance weight $w_u$. These then yield the class of ``feature
based functions''
\begin{align}
f(X) = \sum_{u \in U} w_u \phi_u(m_u(X)) + m_\pm(X)
\label{eq:feature_based}
\end{align}
where $m_u(X) = \sum_{x \in X} m_u(x)$.
A feature based function then is an SCMM.

In a feature-based function, $m_u(v) \geq 0$ is a non-negative score
that measures the degree of feature $u$ that exists in data object $v$
and the vector $m_U(v)$ is the entirety of the object's representation
in feature space. The quantity $m_u(X)$ measures the $u$-ness in a
collection of objects $X$ that, when the concave function
$\phi_u(\cdot)$ is applied, starts diminishing the contribution of
this feature for that set of objects. The importance of each feature
is given by the feature weight $w_u$.  From the perspective of
applications, $U$ can be any set of features.

As an example in NLP, let $V$ be a set of sentences. For $s \in
V$ and $u \in U$, define $m_u(v)$ to be the count of n-gram
feature $u$ in sentence $s$.  For the sentence $s = \text{{\em
    Whenever I visit New York City, I buy a New York City map.}}$,
$m_{\text{"the"}}(s) = 1$ while $m_{\text{"New York City"}}(s) = 2$.
There are many different ways to produce the scores $m_u(s)$ other
than raw n-gram counts. For example, they can be TFIDF-based
normalized counts, or scaled in various ways depending on the nature
of $u$. The weight $w_u$ can be the desired relative frequency of $u$,
the length of $u$, and so on.


Feature engineering is the study of techniques for transforming raw
data objects into feature vectors and is an important step for many
machine learning~\cite{zheng2006,wiki:feature_eng,web:feature_eng} and
structured prediction problems \cite{taskar2005learning}. Good feature
engineering allows for potentially different size and type of data
items (either within or across modalities) to be embedded within the
same space and hence considered on the same playing field. Proper
feature representation is often therefore a crucial for many machine
learning systems to perform well. In the case of NLP, for example,
features requiring annotation tools (e.g., parse-based features
\cite{zhang08,massung13,post13}) and unsupervised features such as
n-gram and word distribution features (e.g.,
\cite{xue09,bergsma10,bansai12,lahiri13,pietro14}) are available.  For
computer vision, this includes visual bag-of-words features (e.g.,
\cite{filliat2007visual,yang2007evaluating,li2011contextual,nicosevici2012automatic,csurka2004visual,tirilly2008language,deselaers2008bag}).
Any type of data can have automatically learned features using
representation learning via, say, deep models (e.g.,
\cite{turian10,mikolov13,mnih13,penning14,kalch14,liu15}) --- this is
essentially the main message in the name ICLR (International
Conference on Learning Representations), one of the main venues for
deep model research today.

One of the advantages of feature based submodular functions for
machine learning and data science applications is that they can
leverage this vast amount of available work on feature
engineering. Feature transformations can be developed separately from
the resulting submodularity and can still be easily incorporated into
a feature based function without loosing the submodularity property.

Figure~\ref{fig:scmm_fig1} gives another illustrative but contrived
example, that demonstrates how feature functions, when maximized,
attempt to achieve a form of uniformity, and hence diversity, over
feature space.  The figure also helps to motivate deep submodular
functions in the next section.  We have $|V|=9$ data objects each of
which is an image containing a set of shapes, some number of circles,
squares, and triangles. For example, Figure~\ref{fig:scmm_fig1}-(I)
shows that object $\aEL$ contains nine squares while object $\dEL$
contains three each of squares, circles, and triangles. To the right
of these shapes is the corresponding vector $m_U(v)$ for that object
(e.g., $m_U(\gEL)$ shows four squares, three triangles, and two
circles). On these shapes we can define a submodular function as
follows:
$g(A) = \sum_{u \in \{ \triangle, \Box, \bigcirc \} } \sqrt{m_u(A)}$
where $m_u(A) = \sum_{a \in A} \text{count}_u(a)$ counts the total
number of objects of type $u$ in the set of images $A$.
Figure~\ref{fig:scmm_fig1}-(II) shows
$g(\set{\bEL}) = \sqrt{8} + \sqrt{1}$.
Figure~\ref{fig:scmm_fig1}-(III) shows $g(\set{\aEL,\bEL,\cEL})$ which
has a greater diversity of objects and hence is given a greater value,
while Figure~\ref{fig:scmm_fig1}-(IV) shows
$g(\set{\dEL,\hEL,\fEL}) = 9$ which is the maximum valued size-three
set (and is also the solution to the greedy algorithm in this case),
and is the set having the greatest diversity. Diversity, therefore,
corresponds to uniformity and maximizing this submodular function,
under a cardinality constraint, strives to find a set of objects with
as even a histogram of feature counts as possible. When using
non-uniform weights $w_u$, then maximizing this submodular function
attempts to find a set that closely respect the feature weights.

In fact, maximizing feature based functions can be seen as a form of
constrained divergence minimization.  Let $p = \set{ p_u }_{u \in U}$
be a given probability distribution over features (i.e., $\sum_u p_u =
1$ and $p_u \geq 0$ for all $u \in U$). Next, create
an $X$-dependent distribution over features:
\begin{align}
      0 \leq \bar p_u(X) \triangleq \frac{ m_u(X) }{ \sum_{u' \in U} m_{u'}(X) }
      = \frac{ m_u(X) }{ m(X) } \leq 1
\end{align}
where $m(X) \triangleq \sum_{u' \in U} m_{u'}(X)$.  Then $\bar p_u(X)$
can also be seen as a distribution over features $U$ since $\bar
p_u(X) \geq 0$ and $\sum_{u \in U} \bar p_u(X) = 1$ for any $X
\subseteq V$.  Consider the KL-divergence between these two
distributions:
\begin{align}
      D( p || \bar p(X)  ) 
      &= - H(p) + \log m(X) - \sum_{u \in U} p_u \log (m_u(X))
\end{align}
Hence, the KL-divergence is merely a constant plus a difference of
feature-based functions. Maximizing $\sum_{u \in U} p_u \log (m_u(X))$
subject to $\log m(X) = \text{const}$ (which can be seen as a data
quantity constraint) therefore is identical to finding an $X$ that
minimizes the KL-divergence between $\bar p(X)$ and $p$.
Alternatively, defining $g(X) \triangleq \log m(X) - D( p || \set{
  \bar m_u(X) } ) = \sum_{u \in U} p_u \log (m_u(X))$ as done in
\cite{shinohara2014submodular}, we have a {submodular function $g$}
that represents a combination of its quantity of $X$ via its features
(i.e., $\log m(X)$) and its distribution closeness to $p$. The concave
function in the above is $\phi(\alpha) = \log(\alpha)$ which is
negative for $\alpha < 1$.  We can rectify this situation by defining
an extra object $v' \notin V$ having $m_u(v') = 1$ for all $u$. Then
$g(X | v') = \sum_{u \in U} p_u \log (1+ m_u(X))$ is also a feature
based function on $V$.

The KL-divergence can be generalized in various ways, one of which is
known as the $f$-divergence, or in particular the
$\alpha$-divergence~\cite{shun1985differential,amari1993methods}.
Using the reparameteriation $\alpha = 1 - 2\delta$
\cite{kass1984canonical}, the $\alpha$-divergence (or now
$\delta$-divergence \cite{zhu1995information}) can be expressed as
\begin{align}
D_{\delta}(p,q) = \frac{1}{\delta(1-\delta)} (1-\sum_{u\in U} p_u^\delta q_u^{1-\delta}).
\end{align}
For $\delta \to 1$ we recover the standard KL-divergence above. 
For $\delta \in (0,1)$ we see that the 
optimization problem $\min_{X \subseteq V: m(X) \leq b} D_{\delta}(p,\bar p(X))$ where
$b$ is a budget constraint is the same
as the constrained submodular maximization problem
$\max_{X \subseteq V: m(X) \leq b} g(X)$ 
where $g(X) = \sum_{u \in U} p_u^\delta (m_u(X))^{1-\delta}$
is a feature-based function since $\phi_u(\alpha) = \alpha^{1-\delta}$
is concave on $\alpha \in [0,1]$ for $\delta \in (0,1)$. Hence,
any such constrained submodular maximization problem can be seen
as a form of $\alpha$-divergence minimization.


Indeed, there are many useful concave functions one could employ in
applications and that can achieve different forms of submodular
function. Examples include the following: 
(1) the power functions, such as $\phi(\alpha) = \alpha^{1-\delta}$
that we just encountered ($\delta = 1/2$ in
Figures~\ref{fig:scmm_fig1} (I)-(IV));
(2) the other non-saturating non-linearities such as
$\phi(x)=\nu^{-1}(x)$ where $\nu(y)=y^3 / 3 + y$
\cite{galen2013-deep-cca} and the log functions
$\phi_\gamma(\alpha) = \gamma \log(1+\alpha/\gamma)$ with
$\gamma > 0$ is a parameter;
(3) the saturating functions such as
$\phi(\alpha) = 1 - \exp(-\alpha)$, the logistic function $\phi(\alpha) =
1/(1+\exp(-\alpha))$ and other ``s''-shaped sigmoids 
(which are concave over the non-negative reals)
such as the hyperbolic tangent,
or $\phi(\alpha) = \Bigl[ 1 - \frac{1}{\ln(b)}
\ln \Bigl( 1 + \exp \bigl( - \alpha \ln(b) \bigr)\Bigr) \Bigr]$ as
used in \cite{bicici11,kirch14};
(4) and the hard truncation functions such as
$\phi(\alpha) = \min(\alpha, \gamma)$ for some constant
$\gamma$.  There are also parameterized concave
functions that get as close to the 
hard truncation functions as we wish,
such as $\phi_{a,c}(x) = ((x^{-a}+c^{-a})/2)^{-1/a}$ where 
$a\geq -1$, and $c>0$ are parameters --- it is straightforward to show that 
$\phi_{-1,c}(x)$ is linear, that $\lim_{a \to \infty} \phi_{a,c}(x) = \min(x,c)$,
and that for $-1 < a < \infty$ we have a form of soft min. Also recall
the parameterized soft max mentioned above in relationship to the facility
location function.
In other cases, is useful for the concave function
to be linear for a while before a soft or nonsaturating concave part kicks in,
for example $\phi(\alpha) = \min( \sqrt{\alpha/\gamma}, \alpha/\gamma)$ for
some constant $\gamma > 0$.
These all can have their uses, depending on the application,
and determine the nature of how the returns of a given feature $u \in U$ should diminish.
Feature based submodular functions, in particular, have been useful
for tasks in speech
recognition~\cite{wei2014-unsupervised-icassp}, machine
translation~\cite{kirch14}, and computer vision
\cite{jegelka2011-nonsubmod-vision}.


We mention a final advantage of SCMMs is that they do not require the
construction of a pairwise graph and therefore do not have quadratic
cost as would, say a facility location function (e.g.,
$f(X) = \sum_{v \in V} \max_{x \in X} w_{xv}$), or any function based
on pair-wise distances, all of which have cost $O(n^2)$ to evaluate.
Feature functions have an evaluation cost of $O(n|U|)$, linear in the
ground set $V$ size and therefore are more scalable to large data set
sizes. Finally, unlike the facility location and other graph-based
functions, feature-based functions do not require the use of the
entire ground set for each evaluation and hence are appropriate for
streaming
algorithms~\cite{badanidiyuru2014streaming,chekuri2015streaming} where
future ground elements are unavailable at the time one needs a
function evaluation, as well as parallel submodular optimization
\cite{mirzasoleiman2015distributed,barbosa2015power,barbosa2016new}.
For example, the vectors $m_U(v)$ for a newly encountered object $v$
can be computed on the fly (or in parallel) whenever the object $v$
is available and wherever it is located on a parallel machine.

\section{Deep Submodular Functions}
\label{sec:deep-subm-funct}
\label{sec:dsf_definition}


\begin{figure}[tbh]
\centerline{\includegraphics[page=2,width=0.8\textwidth]{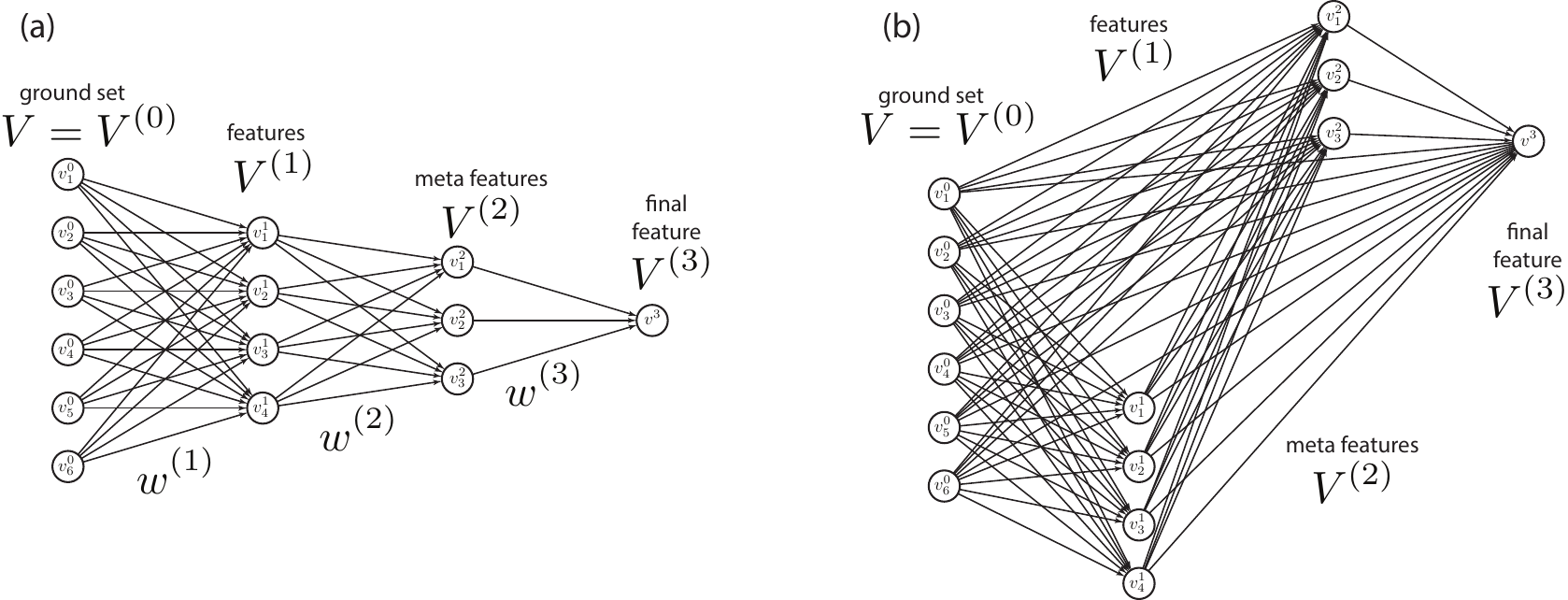}}
\caption{
Left: A layered 
DSF 
with $K=3$ layers.
Right: a 3-block DSF 
allowing layer skipping.
}
\label{fig:deep_submodular}
\end{figure}

While feature-based submodular functions are indisputably
useful, their weakness lies in that features themselves may not
interact, although one feature $u'$ might be partially redundant with
another feature $u''$. For example, when describing a sentence via its
component n-grams features, higher-order n-grams always include
lower-order n-grams, so some n-gram features can be partially
redundant. For example, in a large collection of documents about ``New
York City'', it is likely there will be some instances of ``Chicago,''
so the feature functions for these two features should likely negatively
covary. One way to reduce this redundancy is to subselect the features
themselves, reducing them down to a subset that tends not to interact
in any way. This can only work in limited cases, however, namely when
the features themselves can be reduced to an ``independent'' set that
looses no information about the data objects, and this only happens
when redundancy is an all-or-nothing property (as in a matroid). 

Most real-world features, however, involve partial redundancy.  The
presence of ``New York City'' shouldn't completely remove the
contributing of ``Chicago'', rather it should only discount its
contribution.  A better strategy, therefore, is to allow the feature
scores to interact, say, when measuring redundancy at some
higher-level concept of a ``big city.''  

Figure~\ref{fig:scmm_fig1} offers a further pictorial example.
Figure~\ref{fig:scmm_fig1}-(IV) shows that the most diverse set of
size three is $\set{\dEL, \hEL, \fEL}$ since it has an even
distribution over the set of features, square, triangle,
circle. Suppose, however, the non-smooth shapes are seen to be
partially redundant with each other, so that the presence of a square
should discount, to some degree, the value of a triangle, but should
not discount the value of a circle. The feature based function $g(A) =
\sum_{u \in \{ \triangle, \Box, \bigcirc \} } \sqrt{m_u(A)}$ does not
allow these three features to interact in any way to achieve this form
of discounting. The contribution of ``square'' is measured
combinatorially independently of ``triangle'' --- feature-based
functions therefore fail for features that themselves should be
considered partially redundant.  We can address this issue by using an
additional level of concave composition
\begin{align}
g(A) = 
 \sqrt{\sum_{u \in \{ \triangle, \Box \} } \sqrt{m_u(A)} }
+  \sqrt{m_\bigcirc(A)},
\end{align}
where the nested square-root over the two features, square and
triangle, allow them to interact and discount each other.
Figure~\ref{fig:scmm_fig1}-(V) shows the new value of the formally
maximum set $\set{\dEL, \hEL, \fEL}$ 
is no longer the maximum size-three
set. Figure~\ref{fig:scmm_fig1}-(VI) shows the new maximum sized-three
set, where the number of squares and circles together is roughly the
same as the number of circles.

In general, to allow feature scores to interact and discount each
other, we can utilize an additional ``layer'' of nested concave
functions as follows:
\begin{align}
f(X) = 
\sum_{s \in S}
\omega_s \phi_s (
\sum_{u \in U} w_{su} \phi_u(m_u(X))),
\label{eq:two_level_feature_based}
\end{align}%
where $S$ is a set of meta-features, $\omega_s$ is a meta-feature
weight, $\phi_s$ is a non-decreasing concave function associated
with meta-feature $s$, and $w_{s,u}$ is now a meta-feature
specific feature weight. With this construct, $\phi_s$
assigns a discounted value to the set of features in $U$, which can be
used to represent feature redundancy.
Interactions between the meta-features might be needed as well, and
this can be done via meta-meta-features, and so on, resulting in a
hierarchy of increasingly higher-level features.  Such a hierarchy
could correspond to semantic hierarchies for NLP applications (e.g.,
WordNet~\cite{miller1995wordnet}), or a visual hierarchy in computer
vision (e.g., ImageNet~\cite{deng2009imagenet}). Alternatively, in the
spirit of modern big-data efforts in deep learning, such a hierarchy
could be learnt automatically from data.


We propose a new class of submodular
functions that we call {\em deep submodular functions} (DSFs). 
They may make use
of a finite-length series of disjoint sets (see Figure~\ref{fig:deep_submodular}-(a)):
$V=V^{(0)}$, which is the function's ground set,
and additional sets
$V^{(1)}, V^{(2)}, \dots,V^{(K)}$.
$U=V^{(1)}$ can be seen as a set of ``features'', $V^{(2)}$ as a set of meta-features,
$V^{(3)}$ as a set of meta-meta features, etc.~up to $V^{(K)}$.
The size of $V^{(i)}$ is $d^i = |V^{(i)}|$. Two successive sets 
(or ``layers'') $i-1$ and $i$ are connected by a matrix 
$w^{(i)} \in \mathbb R^{d^{i} \times d^{i-1}}_+$, for $i \in \{1,\dots, K\}$. 
Hence, rows of $w^{(i)}$ are indexed by
elements of $V^{(i)}$ and columns of $w^{(i)}$ are indexed by elements
of $V^{(i-1)}$. 
Given $v^{i} \in V^{(i)}$, define
$w^{(i)}_{v^{i}}$ to be the row of $w^{(i)}$ corresponding to element
$v^{i}$, and $w^{(i)}_{v^{i}}(v^{i-1})$ is the element of matrix
$w^{(i)}$ at row $v^{i}$ and column $v^{i-1}$.  We may think of
$w^{(i)}_{v^{i}}: V^{(i-1)} \to \mathbb R_+$ as a modular function
defined on set $V^{(i-1)}$. Thus, this matrix contains $d^{i}$ such
modular functions.
Further, let $\phi_{v^k}: \mathbb R_+ \to \mathbb R_+$ be a non-negative non-decreasing concave function. 
Then, a $K$-layer DSF $f: 2^V \to \mathbb R_+$ 
can be expressed as follows, for any $A \subseteq V$,
\begin{align}
f(A) = \bar f(A) + m_\pm(A) 
\label{eq:deep_submodular}
\end{align}
where,%
\begingroup\makeatletter\def\f@size{10}\check@mathfonts
\def\maketag@@@#1{\hbox{\m@th\small\normalfont#1}}%
\begin{align}
\label{eq:deep_submodular_polymatroid}
\!\!\!\!\bar f(A)
= 
\phi_{v^K}
\Biggl(
\sum_{v^{K-1} \in V^{(K-1)}}
w^{(K)}_{v^K}(v^{K-1})
\phi_{v^{K-1}}
\biggl(
\dots
\!\!\!\!
\sum_{v^2 \in V^{(2)}}
w_{v^3}^{(3)}(v^2)
\phi_{v^2}
\Bigl(
\sum_{v^1 \in V^{(1)}}
w_{v^2}^{(2)}(v^1)
\phi_{v^1}
\bigl(
\sum_{a \in A} w_{v^1}^{(1)}(a)
\bigr)
\Bigr)
\biggr)
\Biggr),
\end{align}\endgroup
and where $m_\pm: V \to \mathbb R$ is an arbitrary modular function.
Equation~\eqref{eq:deep_submodular} defines a class of submodular
functions. Submodularity follows since a composition of a monotone
non-decreasing function $h$ and a monotone non-decreasing concave
function $\phi$ ($g(\cdot) = \phi(h(\cdot))$) is submodular (Theorem 1
in \cite{Lin2011} and repeated, with proof, in
Theorem~\ref{thm:concave_over_polymatroidal}) --- a DSF is submodular
via recursive application and since submodularity is closed under
conic combinations.


\subsection{Recursively Defined DSFs}
\label{sec:recursive-dsfs}

A slightly more general way to define a DSF and that is useful for the
theorems below uses recursion. This section also defines the notation
that will be often used later in the paper.

We are given a directed acyclic graph
(DAG) $\mathbf G=(\mathbf V, \mathbf E)$ where for any given node
$v \in \mathbf V$, we say $\text{pa}(v) \subset \mathbf V$ are the
parents of (or vertices pointing towards) $v$.  A given size $n$ subset
of nodes $V \subset \mathbf V$ corresponds to the ground set of a
submodular function and for any $v \in V$, $\text{pa}(v) = \emptyset$.
A unique ``root'' node $\rtnd \in \mathbf V \setminus V$ has the
distinction that $\rtnd \notin \text{pa}(q)$ for any
$q \in \mathbf V$.  Given a non-ground node
$v \in \mathbf V \setminus V$, we define the concave function
$\psi_v : \mathbb R^V \to \mathbb R_+$ where
\begin{subequations}
\begin{align}
\psi_v(x) 
& =  \phi_v(\varphi_v(x)),
\label{eq:gen_deep_submodular_recursive_a}
\intertext{and}
\varphi_v(x) &= \sum_{u \in \text{pa}(v) \setminus V}
w_{vu} \psi_u(x) + \langle  m_v, x \rangle.
\label{eq:gen_deep_submodular_recursive_b}
\end{align}
\label{eq:gen_deep_submodular_recursive}
\end{subequations}
In the above, $\phi_v : \mathbb R_+ \to \mathbb R_+$ is a normalized
non-decreasing univariate concave function, $w_{vu} \in \mathbb R_+$
is a non-negative weight indicating the relative importance of
$\psi_u$ to $\varphi_v$, and
$m_v: \mathbb R^{\text{pa}(v) \cap V} \to \mathbb R_+$ is a
non-negative linear function that evaluates as
$\langle m_v, x \rangle = \sum_{u \in \text{pa}(v) \cap V}
m_v(u)x(u)$.  In other words, $\langle m_v, x \rangle$ is a sparse
dot-product over ground elements $\text{pa}(v) \cap V$.  There is no
additional additive bias constant added to the end of
Equation~\eqref{eq:gen_deep_submodular_recursive_b} as this is assumed
to be part of $\phi_v$ (as a shift) if needed (alternatively, for one
of the $u \in \text{pa}(V) \setminus V$, we can set $\psi_u(x)=1$ as a
constant, and the bias may be specified by a weight, as is
occasionally done when specifying neural networks).  The base case,
where $\text{pa}(v) \subseteq V$ therefore has
$\psi_v(x) = \phi_v(\langle m_v, x \rangle$), so $\psi_v(\mathbf 1_A)$
is a concave composed with a modular function. The notation
$\mathbf 1_A$ indicates the characteristic vector of set $A$, meaning
$\mathbf 1_A(v) = 1$ if $v \in A$ and is otherwise zero.

A general DSF is defined as follows: for all $A \subseteq V$,
$f(A) = \psi_\rtnd(\mathbf 1_A) + m_{\pm}(A)$, where
$m_{\pm} : V \to \mathbb R$ is an arbitrary modular function (i.e., it
may include positive and negative elements).
For all $v \in \mathbf V$, we also for convenience, define
$g_v(A) = \psi_v(\mathbf 1_A)$.
To be able to treat all $v \in \mathbf V$ similarly, we say, for
$v \in V$, that $\text{pa}(v) = \emptyset$, and use the
identity $\phi_v(a) = a$ for $a \in \mathbf R$, and set
$m_v = \mathbf 1_v$, so that $\psi_v(x) = \varphi_v(x) = x(v)$ and
$g_v(A) = \mathbf 1_{v \in A}$ which is a modular function on $V$.

By convention, we say that a zero-layer DSF function is an arbitrary
modular function, a one-layer DSF is an SCMM, and a two-layer DSF is,
as we will soon see, something different. By $\text{DSF}_k$, we mean
the family of DSFs with $k$ layers.

As mentioned above, from the perspective of defining a submodular
function, there is no loss of generality by adding the final modular
function $m_{\pm}$ to a polymatroid
function~\cite{cunningham1983decomposition,cu84}. The degree to which DSFs
comprise a subclass of submodular functions corresponds to the degree
to which $g_\rtnd$ comprise a subclass of all polymatroid functions.

The recursive form of DSF is more convenient than the layered approach
mentioned above which, in the current form, would partition
$\mathbf V = \{ V^{(0)}, V^{(1)}, \dots, V^{(K)} \}$ into layers, and
where for any $v \in V^{(i)}$, $\text{pa}(v) \subseteq V^{(i-1)}$.
Figure~\ref{fig:deep_submodular}-(a) corresponds to a layered graph
$\mathbf G = (\mathbf V, \mathbf E)$ where $r = v_1^3$ and
$V = \{v^0_1, v^0_2, \dots, v^0_6\}$.
Figure~\ref{fig:deep_submodular}-(b) uses the same partitioning but
where units are allowed to skip by more than one layer at a time.
More generally, we can order the vertices in $\mathbf V$ with order
$\sigma$ so that $\{ \sigma_1, \sigma_2, \dots, \sigma_n \} = V$ where
$n=|V|$, $\sigma_m = \rtnd = v^K$ where $m = |\mathbf V|$ and where
$\sigma_i \in \text{pa}(\sigma_j)$ iff $i < j$. This allows an
arbitrary pattern of skipping while maintaining submodularity. The
additional linear function in
Equation~\eqref{eq:gen_deep_submodular_recursive_b} is strictly not
necessary (e.g., there could be paths of linearity along subsets of
the $\phi_v v \in \mathbf A$ for some $\mathbf A \subset \mathbf V$ thereby
achieving the same result) but we include it to stress that at each layer
there may be a modular function and a bias.

\subsection{DSFs: Practical Benefits and Relation to Deep Neural Networks}
\label{sec:dsfs-deep-neural}

The layered definition in Equation~\eqref{eq:deep_submodular} is
reminiscent of feed-forward deep neural networks (DNNs) owing to its
multi-layered architecture. Interestingly, if one restricts the
weights of a DNN at every layer to be non-negative, then for many
standard hidden-unit activation functions the DNN constitutes a
submodular function when given Boolean input vectors. The result
follows for any activation function that is monotone non-decreasing
concave for non-negative reals, such as the sigmoid, the hyperbolic
tangent, and the rectified linear functions. In the rectified linear
case, however, the entire network would be linear so the model becomes
interesting only with hidden activations that are strictly concave
(since the weights can be arbitrarily scaled, perhaps
$\phi(x) = \min(x,1)$ is a reasonable concave analogy in a DSF to the
rectified linear function in a DNN).  More importantly, this suggests
that DSFs can be trained in a fashion similar to DNNs ---
specifically, training DSFs and can take advantage of the many
successful training techniques and software libraries for training
DNNs (many of the toolkits make it easy to project weights into the
positive orthant). Further discussion on this point is given in
Section~\ref{sec:learn-deep-subm}. The recursive definition of DSFs,
in Equation~\eqref{eq:gen_deep_submodular_recursive} is useful for the
analysis in Section~\ref{sec:family-comp-subm}.


DSFs should be useful for many applications in machine
learning. First, they retain the advantages of SCMMs in that they
require neither $O(n^2)$ computation nor access to the entire ground
set for a set evaluation. The underlying DSF computation is
matrix-vector multiplication that, like DNNs, can be performed very
quickly using modern GPU computing. Hence, DSFs can be both fast, and
useful for parallel and/or streaming applications.  Second, DSFs allow
for a nested hierarchy of features, similar to advantages a deep model
has over a shallow model. For example, a one-layer DSF must construct
a valuation over a set of objects from a large number of low-level
features which can lead to fewer opportunities for feature sharing
while a deeper network fosters distributed representations, also
analogous to DNNs \cite{bengio09,bengio13}.  It can be argued that a
deep neural network is more efficient, in terms of the number of
possible functions represented per weight, than a shallow neural
network and perhaps DSFs share this advantage.  Hence, even if the DSF
and SCMM families were to be found to be same (but that
Theorem~\ref{sec:dsfs-extend-class} shows to be false), there could be
advantages to applications and learning paradigms thanks to this
natural hierarchical decomposition of concepts.

DSFs have been used occasionally in some applications.  In one
instance~\cite{lin2010-pp-corpus-creation}, a square root was applied
to a subset of the right hand nodes in a bipartite neighborhood
function in order to offer reduced cost for these nodes being
indirectly selected in the graph.
In~\cite{wei2014-unsupervised-icassp} a two-layer DSF was used to
introduce higher-level interaction between features, an act that
yielded benefits in speech data summarization. Lastly,
laminar matroid rank functions, which are instances of DSFs as shown
in Section~\ref{sec:matroids-case}, have been used to show worst case
performance of various constrained submodular minimization problems
\cite{goemans2009approximating,svitkina2011submodular,rishabh2013-submodular-constraints}.

\section{Relevant Properties and Special Cases}
\label{sec:family-comp-subm}

DSFs represent a family that, at the very least, contain the family of
SCMMs. Above, we argued intuitively that DSFs might extend SCMMs as
they allow components themselves to directly interact, and the
interactions may propagate up a many-layered hierarchy.  In this
section, we start off (in Section~\ref{sec:prop-conc-subm})
discussing preliminaries regarding concave functions.
Section~\ref{sec:antit-maps-superd} then covers specific properties of
the multivariate concave function associated with a DSF, in particular
the antitone gradient superdifferential property which is a a
sufficient condition for submodularity. This section also compares
this condition with the negativity of the off-diagonal Hessian matrix
condition for submodular functions.  Section~\ref{sec:matroids-case}
discusses matroid rank special cases, including the laminar matroid
rank function which can be seen, in the light of this paper, as a form
of deep matroid rank.  This section also discusses special cases of
the results shown later in the paper, in particular, that: (1) cycle
matroid rank functions cannot represent all partition matroid rank
functions; (2) laminar matroid rank functions strictly generalize
partition matroid rank functions; (3) laminar matroid rank functions
cannot express all cycle matroid rank functions; (4) DSFs generalize
laminar matroid rank functions; and (5) SCMMs generalize partition
matroid rank functions.  Lastly, section~\ref{sec:defic-absol-redund}
introduces various analysis tools (in particular the ``surplus'') that
are used later in the paper.

\subsection{Properties of Concave and Submodular Functions}
\label{sec:prop-conc-subm}

Many of the results in the sections below rely on a number of
properties of concave functions. Since we wish to consider
non-differentiable concave functions, the theorems below consider this
more general case where we may assume only that the concave functions
have superdifferentials. It is, in general, more work to show that the
properties of concave functions hold in this non-differential case,
but since there seem to be no consolidated published proofs of these
properties, we offer them here in full.

Let $\phi : \mathbf R \to \mathbf R$ be a normalized ($\phi(0) = 0$)
monotone non-decreasing concave function.  In any such function, there
may be an initial linear part where $\phi(x) = \gamma x$ for
$x \in [0,\alpha_\phi]$ where $\gamma > 0$ and where
$\alpha_\phi \geq 0$ is the largest point where $\phi$ is still
linear. Larger than $\alpha_\phi$, there may be a middle part
consisting of a series of concave curves and line segments all
situated to ensure concavity. Larger than this, there finally might be
a saturation point where $\phi(x) = c$ for all
$x \geq \alpha_\text{sat}$, where
$c,\alpha_\text{sat} \in \mathbb R_+ \cup \set{ \infty }$. The middle
region ($x \in [\alpha_\text{lin},\alpha_\text{sat}]$) might or might
not be smooth. It is useful sometimes in applications (e.g.,
\cite{jegelka2011-nonsubmod-vision}) to formulate submodular functions
from concave functions that have an initial linear part followed by
either a saturation or by a smooth concave part.

\begin{definition}[Superdifferential]
Let $\phi : \mathbb R^n \to \mathbb R$ be a concave function. The superdifferential
of $\phi$ at $x$ is the set of vectors defined as follows:
\begin{align}
\partial \phi(x) = 
\set{
s \in \mathbb R^n :
f(y) - f(x)
\leq
\langle s, y - x \rangle,
\forall y \in \mathbb R^n
}
\end{align}
\end{definition}
\noindent
The superdifferential of a concave function is guaranteed always to exist
\cite{rockafellar1966characterization,rockafellar1970maximal,hiriart1993convex,nesterov2004introductory}.
When $\phi$ is differentiable at $x$, 
the superdifferential corresponds
to the gradient, so that $\partial \phi(x) = \set{ \nabla \phi(x) }$
and otherwise members of $\partial \phi(x)$ are called subgradients.
In general, we have the following:
\begin{lemma}
The superdifferential of a concave function is a monotone operator, i.e.,
\begin{align}
\langle u - v, x - y \rangle  \leq 0,
\forall x, y \in \mathbb R^n, u \in \partial \phi(x), v \in \partial \phi(y)
\end{align}
\end{lemma}
\begin{proof}
We have that
\begin{align}
f(y) \leq f(x) + \langle u, y - x \rangle,
\text{ and } f(x) \leq  f(y) + \langle v , x - y \rangle
\end{align}
Adding the two inequalities yields monotonicity.
\end{proof}
This means in particular that, in the one-dimensional case when $n=1$,
if $x \leq y$ then for any $u \in \partial \phi(x)$ and any
$v \in \partial \phi(y)$, we must have $u \geq v$. In the below, we
offer a number of properties of concave superdifferentials in the 1D
case. While statements of these results are intuitively clear, the
authors were unable to find published proofs, so they are also
included herein.

\begin{theorem}
Let $\phi : \mathbb R \to \mathbb R$ be a 
continuous function. 
Then $\phi$ is 
concave
if and only if 
for all $a,b \in \mathbb R$ with $a \leq b$,
and $\Delta \in \mathbb R_+$, we have that
\begin{align}
\phi(a + \Delta) - \phi(a) \geq \phi(b + \Delta)  - \phi(b).
\label{eq:concave_cond}
\end{align}
Also, $\phi$ is monotone non-decreasing concave if and only if for all
$a,b \in \mathbb R$ with $a \leq b$, and
$\Delta, \epsilon \in \mathbb R_+$, we have that
\begin{align}
\phi(a + \Delta + \epsilon) - \phi(a) \geq \phi(b + \Delta)  - \phi(b)
\label{eq:mon_concave_cond}
\end{align}
\end{theorem}
\begin{proof}
The result is vacuous if $a = b$, or $\Delta = 0$ so assume $a < b$
and $\Delta > 0$.

If part: Assume
Equation~\eqref{eq:concave_cond} is true and consider
\begin{align}
\frac{ \phi(a + \Delta) - \phi(a) } { \Delta }
\geq \frac{ \phi(b + \Delta) - \phi(b) } { \Delta }
\end{align}
If $\phi$ is differentiable at $a$ and $b$, then taking $\Delta \to 0$
gives us $\phi'(a) \geq \phi'(b)$ for all $a \leq b$, and this is a
sufficient condition for concavity (see Nesterov 2.13, page 54,
\cite{nesterov2004introductory}). If $\phi$ is not differentiable at
either $a$ or $b$, we resort to its continuity. A function is concave
if and only if it is continuous and midpoint concave
\cite{niculescu2006convex} (or midconcave \cite{roberts1974convex}),
defined as for any $x,y \in \mathbb R$ $f((x+y)/2) \geq
(f(x)+f(y))/2$).  This condition is immediate from
Equation~\eqref{eq:concave_cond} by setting $x = a$, $y = b + \Delta$,
and $b = a + \Delta = (x+y)/2$.

Only if part: Assume $\phi$ is concave and $a < b$ and $\Delta > 0$ are
given. If $\phi$ is differentiable, then by the mean value theorem,
there exists an $a^+$ with $a \leq a^+ \leq a + \Delta$ and a $b^+$
with $b \leq b^+ \leq b + \Delta$ where
\begin{align}
\phi'(a^+) = \frac{ \phi(a + \Delta) - \phi(a) } { \Delta } \\
\intertext{and}
\phi'(b^+) = \frac{ \phi(b + \Delta) - \phi(b) } { \Delta }
\end{align}

If $a+ \Delta \leq b$ then $a^+ \leq b^+$ and hence $\phi'(a^+) \geq
\phi'(b^+)$ by concavity (Nesterov) which immediately gives $\phi(a +
\Delta) - \phi(a) \geq \phi(b + \Delta) - \phi(b)$.  If $\phi$ is not
differentiable at either $a$ or $b$, then consider $d_a
\in \partial \phi(a)$ and $d_b \in \partial \phi(b)$, so that $\forall
y_a,y_b$, $\phi(y_a) \leq \phi(a) + \langle d_a, y_a - a \rangle$ and
$\phi(y_b) \leq \phi(b) + \langle d_b, y_b - b \rangle$.  Taking $y_a
= a + \Delta$ and $y_b = b + \Delta$ gives $(\phi(a+\Delta) -
\phi(a))/\Delta = d_a \geq d_b = (\phi(b + \delta) - \phi(b))/\Delta$
which follows from the monotonicity of the superdifferential operator.

If $a+ \Delta > b$ then 
$a < b < a + \Delta < b + \Delta$. Again when $\phi$ 
is differentiable, by the mean
value theorem, there
exists $a_b^+$ 
with $a \leq a_b^+ \leq b$
and $a_\Delta^+$ 
with $a + \Delta 
\leq
a_\Delta^+ \leq b + \Delta$ with
\begin{align}
\phi'(a^+_b) = \frac{ \phi(b) - \phi(a) } { b - a } \\
\intertext{and}
\phi'(a_\Delta^+) = \frac{ \phi(b + \Delta) - \phi(a + \Delta) } { b - a },
\end{align}
and since $a_b^+ < a_\Delta^+$,
$\phi'(a^+_b) \geq \phi'(a_\Delta^+)$. This immediately
gives 
$\phi(b) - \phi(a) \geq 
\phi(b + \Delta) - \phi(a + \Delta)$
or 
$\phi(a + \Delta) - \phi(a)
\geq \phi(b + \Delta) - \phi(b)$.
If $\phi$ is not
differentiable, then taking supergradients $d_a \in \partial \phi(a)$
and $d_{a +\Delta} \in \partial \phi(a + \Delta)$ again gives the result.

The second part of the theorem is immediate if we take $a=b$, and
define $\delta = a+\Delta$ leading to $\phi(\delta + \epsilon) \geq
\phi(\delta)$, i.e., monotonicity.
\end{proof}

The above proof considers the smooth and non-smooth varieties
separately where the non-smooth case utilizes only the existence of
the superdifferential of a concave function. Since the
superdifferential always exists for a concave function, smooth or
otherwise, in the below we consider only the most general case where
we assume only a superdifferential exists. As a result, the proofs are
a bit more involved, but when constructing DSFs and considering the
resultant submodular families in Section~\ref{sec:dsfs}, we wish to
allow for the most general class concave functions.

We next restate Theorem 1 from \cite{Lin2011} but also provide a proof
which was missing.
\begin{theorem}
\label{thm:concave_over_polymatroidal}
Suppose that $h : 2^V \to \mathbb R$ is
a monotone non-decreasing submodular function
and $\phi$ is a monotone non-decreasing concave
function. Then $g(A) = \phi(h(A))$ is monotone
non-decreasing submodular.
\end{theorem}
\begin{proof}
  Consider any $A \subseteq B \subset V$ and $v \notin B$.  Define
  quantities $a,b, \Delta, \epsilon$ so that: $a = h(A) \leq b = h(B)$,
  $a+ \Delta + \epsilon = h(A + v)$, and $b + \Delta = h(B + b)$.
  I.e., $h(v | A) = \Delta + \epsilon \geq h(v|B) = \Delta$.
  Then 
we have
\begin{align}
\phi(a + \Delta + \epsilon) - \phi(a)
\geq \phi(b + \Delta) - \phi(b)
\end{align}
or 
\begin{align}
\phi(h(A + v)) - \phi(h(A))
\geq \phi(h(B + v)) - \phi(h(B)).
\end{align}
\end{proof}

The slope of the linear interpolation between two points on a concave
function puts a connecting relationship on the corresponding
superdifferentials at each of the two points, as the following result
shows.
\begin{lemma}
  Given a concave function $\phi: \mathbb R \to \mathbb R$ and two
  points $a,b$ with $a < b$ that define the value
  $d_{ab} = (\phi(b) - \phi(a))/(b-a)$. 
  Then $\min_{d \in \partial \phi(a)} d > d_{ab}$
  if and only if $\max_{d \in \partial \phi(b)} d < d_{ab}$.
\label{thm:simultaneously_strict}
\end{lemma}
\begin{proof}
  From the monotonicity of the
  supergradient~\cite{hiriart1993convex,nesterov2004introductory}, we
  always have
\begin{align}
d_a^{\text{min}} \triangleq
 \min_{d \in \partial \phi(a)} d \geq d_{ab}
\geq \max_{d \in \partial \phi(b)} d
\triangleq
d_b^\text{max} 
\label{eq:dab_sandwich}
\end{align}
since otherwise, say if
$d_a^{\text{min}} < d_{ab}$, then
$\phi(a) + d_a^{\text{min}}(b-a) < \phi(a) + d_{ab}(b-a) = \phi(b)$
which contradicts $d_a^\text{min}$ being a supergradient. 
We must show that the inequalities in 
Equation~\eqref{eq:dab_sandwich} can be only simultaneously strict.
Let
$d_a^\text{min}$ be given such that $d_a^\text{min} > d_{ab}$,
and suppose that $d_b^\text{max} = d_{ab}$. Then
\begin{align}
\phi(y) 
  &\leq \phi(b) + d_b^\text{max}(y-b)  \\
  &= \phi(b) + d_b^\text{max}(y - a + a -b) \\
  &= \phi(b) + d_b^\text{max}(a-b) + d_b^\text{max}(y-a) \\
  &= \phi(a) + d_b^\text{max}(y-a)
\end{align}
and hence we have found a supergradient $d_b^\text{max} \in \partial \phi(a)$
with  $d_a^\text{min} > d_b^\text{max}$ contradicting
the minimality of $d_a^\text{min}$. Hence,
we must have $d_b^\text{max} < d_{ab}$. A similar
argument shows that $d_b^\text{max} < d_{ab}$
and 
$d_a^\text{min} = d_{ab}$ leads to a contradiction of
the maximality of $d_b^\text{max}$.
\end{proof}

The next result identifies a condition that, if true, tells us
about the extent of the initial linear region of a monotone non-decreasing
concave function.
\begin{theorem}
  Given a monotone non-decreasing concave function
  $\phi : \mathbb R \to \mathbb R$ that is normalized ($\phi(0) = 0$)
  and any $a, b \in \mathbb R_+$ with $0 < a \leq b$. Then
  $\phi(a+b) = \phi(a) + \phi(b)$, if and only if $\phi(x)$ is
  linear in the region from $0$ to $a+b$ (that is,
  there exists $\gamma \in \mathbb R$ with
  $\phi(x) = \gamma x$ for $x \in [0,a+b]$.
\label{thm:concave_additivity}
\end{theorem}
\begin{proof}
If case: immediate.

Only if case: 
Any violations of the following inequalities
would violate the superdifferential property of
$\partial \phi(y)$ at $0$, $a$, $b$, or $a+b$:
\begin{align}
\min_{d \in \partial \phi(0)}d \geq \phi(a)/a,
&& \max_{d \in \partial \phi(a)} d \leq \phi(a)/a, \\
\min_{d \in \partial \phi(a)} d \geq \frac{\phi(b)-\phi(a)}{b-a},
&& \max_{d \in \partial \phi(b)} d \leq \frac{\phi(b)-\phi(a)}{b-a}, \\
\min_{d \in \partial \phi(b)} d \geq \frac{\phi(a+b)-\phi(b)}{(a+b)-a}  = \phi(a)/a,
&& \max_{d \in \partial \phi(a+b)} d \leq \phi(a)/a.
\end{align}
This leads to the series of inequalities:
\begin{align}
\min_{d \in \partial \phi(0)} d 
&\stackrel{\text{(a)}}{\geq}
\phi(a)/a
\stackrel{\text{(b)}}{\geq}
\max_{d \in \partial \phi(a)} d
\geq 
\min_{d \in \partial \phi(a)} d
\geq
\frac{\phi(b) - \phi(a)}{b-a}
\geq 
\max_{d \in \partial \phi(b)} d \\
&\geq 
\min_{d \in \partial \phi(b)} d
\stackrel{\text{(c)}}{\geq}
\phi(a)/a
\stackrel{\text{(d)}}{\geq}
\max_{d \in \partial \phi(a+b)} d
\end{align}
From Lemma~\ref{thm:simultaneously_strict}, if (a) is strict, then so
is (b), leading to the contradiction $\phi(a)/a > \phi(a)/a$.
Also from Lemma~\ref{thm:simultaneously_strict}, if (d)
is strict, then so is (c), leading to the same contradiction.
Hence, all inequalities are equalities.
By the monotonicity of the superdifferential of a concave function, we have
that for any $x < y < z$ and $d_y \in \partial \phi(y)$
that
\begin{align}
  \min_{d \in \partial \phi(x) } d \geq d_y \geq \max_{d \in \partial \phi(z) } d
\end{align}
Hence, for all $y \in [0,a+b]$, we have $\partial \phi(y) = \set{ \phi(a)/a }$,
meaning that $\phi$ is linear in this region with $\gamma = \phi(a)/a = \phi(b)/b = \phi(a+b)/(a+b)$.
\end{proof}

It is known that any normalized submodular function is subadditive, in
that for any $A \subseteq V$, $\sum_{a \in A} f(a) \geq f(A)$. A
similar property is true of normalized monotone non-decreasing concave
functions.
\begin{theorem}[Subadditivity]
  Given a normalized monotone non-decreasing concave function $\phi$,
  a set of non-negative points $\set{ x_i }_{i=1}^\ell$,
  $x_i \in \mathbb R_+$, then we have
\begin{align}
\sum_i \phi(x_i) \geq \phi(\sum_i x_i)
\end{align}
and where the inequality is strict if and only if $\sum_i x_i$ is past
any linear part of $\phi$.
\label{thm:subadd_concave}
\end{theorem}
\begin{proof}
It is sufficient to show that it is true for
$x^{1:\ell-1} = \sum_{i=1}^{\ell - 1}x_i$ and $x_\ell$ that
\begin{align}
\phi(x^{1:\ell-1}) + \phi(x_\ell) \geq \phi(\sum_i x_i)
= \phi(x^{1:\ell-1} + x_\ell)
\end{align}
then apply it inductively with $x^{1:\ell-2} = \sum_{i=1}^{\ell-2}x_i$ and 
$x_{\ell - 1}$. Hence, we only need to show that $\phi(x_1) + \phi(x_2) \geq \phi(x_1+x_2)$, and we get this immediately setting $a = 0$,
$\Delta = x_1$, $b= x_2$ in Equation~\eqref{eq:concave_cond}.

The strictness part follows from Theorem~\ref{thm:concave_additivity},
where is states that equality
in $\phi(x^{1:\ell-1}) + \phi(x_\ell) = \phi(\sum_i x_i)$
holds if and only if $\phi$
is linear from $0$ through $x^{1:\ell-1} + x_\ell = \sum_i x_i$.
\end{proof}

The next result shows that when an SCMM has only one term, the
addition of the final modular function $m_\pm$ extends the family. We
in show that this is the case, even when $m_\pm$ is non-negative.

\begin{theorem}
\label{theorem:scm_in_scmm}
The family of an SCMM with one concave over modular term is enlarged
by an additional modular term $m_\pm$.
\end{theorem}

\begin{proof}
  Consider a three-element ground set $V = \left\{ \aEL, \bEL,
    \cEL\right\}$ and a function $g$,
\begin{equation}
  g(A) = \min(|A|, 1) + \mathbf 1_{\cEL \in A},
\end{equation}
thus $g$ is monotone non-decreasing.  Suppose $g(A) =
\phi(m(A))$ for some non-negative modular function $m$ and normalized
non-decreasing concave function $\phi$. Then by
Equation~\eqref{eq:dab_sandwich},
we have:
\begin{align}
\min_{d \in \partial \phi(m(\aEL))} d
\stackrel{\text{(\romannumeral 1)}}{\geq}
\frac{ \phi(m(\aEL,\bEL)) - \phi(m(\aEL))} { m(\aEL,\bEL) - m(\aEL) } = 0
\stackrel{\text{(\romannumeral 2)}}{\geq}
\max_{d \in \partial \phi(m(\aEL,\bEL))} d 
\stackrel{\text{(\romannumeral 3)}}{\geq}
0
\end{align}
where the $(\romannumeral 3)$ follows 
since $\phi$ is monotone. Hence, $(\romannumeral 2)$
is an equality and by Lemma~\ref{thm:simultaneously_strict}
so is $(\romannumeral 1)$.
Hence $0 \in \partial \phi(m(\aEL))$. Then
we have that $\phi(y) \leq \phi(m(\aEL,\bEL)) + 0(y - m(\aEL,\bEL))$.
This means that $\phi(m(\aEL,\bEL,\cEL)) \leq \phi(m(\aEL,\bEL)) = 1 < 2 = g(\aEL,\bEL,\cEL)$,
a contradiction.
\end{proof}
An immediate corollary is that SCMMs are a larger class of submodular
functions than just one concave over modular function. All SCMMs,
however, can be represented as a sum of modular truncations as the
following lemma states:
\begin{lemma}[Sums of Modular Truncations~\cite{stobbe10efficient}]
\label{lemma:sums_mod_truncs}
  If $f$ is an SCMM, then $f$ may be written as
  $f(A) = \sum_i \min(m_i(A),\beta_i) + m_\pm(A)$ where for all $i$,
  $m_i$ is a non-negative modular function, $\beta_i \geq 0$ is a
  non-negative constant, and where the sum is over a finite number of
  terms.
\end{lemma}
Truncating modular function is important, as it is not sufficient to
truncate only cardinality functions.  In other words, SCMMs also
generalize the family of weighted cardinality truncations, as the next
result shows.
\begin{lemma}[Sums of Weighted Cardinality Truncations]
\label{lemma:sums_weighted_cardinality_truncations}
We define the class of sums of weighted cardinality truncations as
\begin{align}
G = \set{ g : \forall A, g(A) = \sum_{B\subseteq V}\sum_{i=1}^{|B|-1}\alpha_{B,i}\min(|A\cap
  B|,i), \text{ where } \forall B, i, \alpha_{B,i}\geq 0 }.
\end{align}
Then there exists an $f \in \text{SCMM}$ that
is not in $G$.
\end{lemma}
Lemma~\ref{lemma:sums_weighted_cardinality_truncations} is proven
in Appendix~\ref{sec:sums-weight-card}.


\subsection{Antitone Maps and Superdifferentials}
\label{sec:antit-maps-superd}

%

Thanks to 
concave composition closure rules~\cite{boyd2004convex},
the root function $\psi_\rtnd(x) : \mathbb R^n \to \mathbb R$ in
Eqn.~\eqref{eq:gen_deep_submodular_recursive} is a monotone
non-decreasing multivariate concave function that, by the
concave-submodular composition rule
(Theorem~\ref{thm:concave_over_polymatroidal}) yields a submodular
function $\psi_\rtnd(\mathbf 1_A)$.  It is widely known that {\bf any}
univariate concave function composed with non-negative modular
functions yields a submodular function. However, given an arbitrary
multivariate concave function this is not the case. Consider, for
example, any concave function $\psi$ over $\mathbb R^2$ that offers
the following evaluations: $\psi(0,0) = \psi(1,1) = 1$,
$\psi(0,1) = \psi(1,0) = 0$.  Then $f(A) = \psi(\mathbf 1_A)$ is not
submodular, and hence the guarantee of submodularity when composing a
concave with a linear function does not extend to dimensions higher
than one. In this section, we discuss a limited form of such a
generalization, one that ensures submodularity and that, moreover,
does not even always rely on concavity in higher dimensions. 
Here and below, for $x,y \in \mathbb R^V$,
then $x \leq y \Leftrightarrow x(v) \leq y(v), \forall v \in V$.


\begin{definition}
A concave function is said to have an {\em antitone superdifferential}
if for all $x \leq y$ we have that $h_x \geq h_y$ for all
$h_x \in \partial \psi(x)$ and $h_y \in \partial \psi(y)$.
\end{definition}

The antitone superdifferential is an apparently straightforward
multidimensional generalization of a defining characteristic of univariate concave
functions.
Theorem~\ref{thm:submod_antitone_super_on_vector_poly} below generalizes
Theorem~\ref{thm:concave_over_polymatroidal} when $k=1$ --- this is
because $\phi : \mathbb R \to \mathbb R$ being concave is, in the 
univariate case, synonymous with it having an antitone
superdifferential (which is synonymous with monotone
supergradients~\cite{hiriart1993convex,nesterov2004introductory}).  
\begin{theorem}
Let $\psi : \mathbb R^k \to \mathbb R$ be a monotone non-decreasing concave function
and let $\vec g: 2^V \to \mathbb R^k$ be a vector of
polymatroid functions, where $\vec g(A) = (g_1(A), g_2(A), \dots, g_k(A))$.
Then
if $\psi$ has an antitone superdifferential, then 
the set function $f: 2^V \to \mathbb R$ defined as
$f(A) = \psi(\vec g(A))$ for all $A \subseteq V$ is submodular.
\label{thm:submod_antitone_super_on_vector_poly}
\end{theorem}
\begin{proof}
Given two points $x,y \in \mathbb R^n$ with $x \leq y$,
then the 
fundamental theorem of calculus for line integrals
states that for any smooth relative path $\bf p$ from $x$ to $y$, the
integral through the vector field $\nabla \psi(z)$ yields
$\psi(y) - \psi(x) = \int_{\bf p} \nabla \psi (x + z) dz$.  If $\psi$
is not differentiable, we may assume, with a slight abuse of notation,
that $\nabla \psi(x)$ is any gradient map for all $x \in \mathbb R^n$
(i.e., $\nabla \psi(x)$ maps from $x$ to some element within $\partial \phi(x)$).
Given an arbitrary $A \subseteq B$ and $v \notin B$, and let
$\mathbf p(t)$ be any relative and parametric curve
from a point $\vec g(A) \in \mathbf R^k$ when $t=0$
to a point $\vec g(A + v) \in \mathbf R^k$ when $t=1$.  Hence,
$\vec g(A) + \mathbf p(0) = \vec g(A)$
and  
$\vec g(A) + \mathbf p(1) = \vec g(A+v)$.
Since $\vec g$ is a vector
of polymatroid functions, we have $\vec g(A) \leq \vec g(B)$
and $\vec g(A) \leq \vec g(A+v)$, and hence, the path $\mathbf p(t)$
can be taken to be monotone, so that
$\vec 0 \leq \mathbf p(t_1) \leq \mathbf p(t_2)$ whenever $0 \leq t_1 \leq t_2 \leq 1$.
Other than monotonicity, the path may be arbitrary. 
By monotonicity and submodularity, 
$\vec 0 \leq \vec g(B+v) - \vec g(B) \leq \vec g(A+v) - \vec g(A)$,
and hence 
we may choose the relative 
path that starts at $\vec 0$, and at some point $t' \in (0,1)$,
goes through the point $\mathbf p(t') = \vec g(B+v) - \vec g(B)$,
and ends up at $\mathbf p(1) = \vec g(A+v) - \vec g(A)$.
Then,
\begin{align}
f(A + v) - f(A) &= \psi(\vec g(A + v)) - \psi(\vec g(A)) 
=     \int_{0}^1 \nabla \psi (\vec g(A) + \mathbf p(t)) \cdot d \mathbf p(t) \\
&\geq \int_{0}^{t'} \nabla \psi (\vec g(A) + \mathbf p(t)) \cdot d \mathbf p(t)
\geq \int_{0}^{t'} \nabla \psi ( \vec g(B) + \mathbf p(t)) \cdot d \mathbf p(t)  \\
&= \psi(\vec g(B + v)) - \psi(\vec g(B)) 
= f(B + v) - f(B), 
\end{align}
where the inequality follows from the monotonicity of $\psi$, the
pointwise antitonicity of the gradient map, the non-negativity of the
path, and by linearity of the integral. Hence, $f$ is submodular.
\end{proof}

We
also fairly quickly get a partial corollary where we need not 
assume that $\phi$ is monotone non-decreasing. In the below, let
$\mathbf b \in \mathbf R^V_+$ be a non-negative real vector and for
any set $A \subseteq V$, $\mathbf b_A$ is a vector such that
$\mathbf b_A(v) = \mathbf b(v)$ if $v \in A$ and otherwise
$\mathbf b_A(v) = 0$ (e.g., when $\mathbf b = \mathbf 1$ then
$\mathbf b_A = \mathbf 1_A$ is the characteristic vector of set $A$).

\begin{corollary}
  Let $\psi : \mathbb R^n \to \mathbb R$ be any concave function and
  $\mathbf b \in \mathbf R^V_+$ be a non-negative real vector.  Then
  if $\psi$ has an antitone superdifferential, then the set function
  $f: 2^V \to \mathbb R$ defined as $f(A) = \psi(\mathbf b_A)$ for all
  $A \subseteq V$ is submodular.
\label{thm:submod_antitone_super}
\end{corollary}
\begin{proof}
  The proof is practically the same as that of
  Theorem~\ref{thm:submod_antitone_super_on_vector_poly} except we
  cannot use the monotonicity of $\psi$. Here the path $\mathbf p$ is
  any relative path from a point $x \in \mathbf R^V_+$ with $x(v) = 0$
  to a point $x + \mathbf b_{v}$.  Given an arbitrary $A \subseteq B$ and $v \notin B$,
  we then get $f(A + v) - f(A) =
 \psi(\mathbf b_{A + v}) - \psi(\mathbf b_A)
 = \int_{\mathbf p} \nabla \psi (\mathbf b_A + z) \cdot dz
 \geq \int_{\bf p} \nabla \psi (\mathbf b_B + z) \cdot dz
 = \psi(\mathbf b_{B + v}) - \psi(\mathbf b_B)
 = f(B + v) - f(B)$.
\end{proof}
Alternatively, we can set $k=n$ in
Theorem~\ref{thm:submod_antitone_super_on_vector_poly} and for all
$v \in V$, set $g_v(A) = \mathbf b(v) \mathbf 1_{v \in A}$ which is a
modular function. Then, the same relative path can be used to move
from $\mathbf b_A$ to $\mathbf b_{A+v}$ as from $\mathbf b_B$ to
$\mathbf b_{B+v}$, so only antotonicity of $\psi$ is needed in the
integral.

Given the above, the following result is not surprising.
\begin{lemma}
  Let $\psi : \mathbb R^n \to \mathbb R$ be a concave function formed
  by the sum of compositions of a scalar concave function and a
  linear function, i.e.,
  $\psi(x) = \sum_i w_i \phi_i(\langle m_i, x \rangle ) + \langle m_\pm, x \rangle $ where
  $m_i \in \mathbb R^n_+$, $w_i \geq 0$ for all $i$, and $m_\pm \in \mathbb R^n$ (i.e., an SCMM).  Then
  $\psi(x)$ has an antitone superdifferential.
\end{lemma}
\begin{proof}
  From the chain rule, we get that
  $\nabla \psi (x) = \sum_i w_i \phi_i'(\langle m_i, x \rangle ) m_i^T
  + m_\pm^T$, and since $\phi_i$ is concave and $m_i$ is non-negative,
  $w_i \phi_i'(\langle m_i, x \rangle) m_i^T$ is monotone
  non-increasing in $x$ ($m_\pm^T$ is constant). In the
  non-differentiable case, $\phi_i$ being monotone-concave implies
  that the same is true for any supergradient map.  Closure over sums
  is immediate.
\end{proof}
\begin{corollary}
Any linear function has an antitone superdifferential.
\end{corollary}

\begin{lemma}
  Composition of monotone non-decreasing scalar concave 
  and antitone superdifferential concave functions
  preserves superdifferential antitonicity.
\end{lemma}
\begin{proof}
  Let $\phi: \mathbb R \to \mathbb R$ be a monotone non-decreasing
  concave functions and $\chi : \mathbb R^n \to \mathbb R$ be a
  monotone non-decreasing concave function with an antitone
  superdifferential, and define $\psi(x) = \phi(\chi(x))$.  Then by
  the chain rule, $\nabla \psi(x) = \phi'(\chi(x)) \nabla \chi (x)$.
  Since $\chi(x)$ is monotone non-decreasing in
  $x$, the first factor $\phi'(\chi(x))$ is
  monotone non-increasing. The second factor is also
  monotone non-increasing, hence so is the product.
 \end{proof}


\begin{corollary}
  The root concave function $\psi_\rtnd$ associated with a DSF has an
  antitone superdifferential.
\end{corollary}
\begin{proof}
The proof follows immediately from the fact
that a DSF function (Equation~\eqref{eq:gen_deep_submodular_recursive})
is a recursive 
application of composition of monotone concave functions,
non-negative sums of monotone concave functions,
and the addition of a final linear function associated with $m_\pm$.
\end{proof}

While having an antitone superdifferential is sufficient to yield a
submodular function, it is not necessary. Consider the following
concave extension of a monotone non-decreasing submodular
function~\cite{vondrak2007submodularity,nem78,fisher1978analysis},
$\psi(x) = \min_{S \subseteq V} [f(S) + \sum_{v \in V} x(v) f(v | S)
]$. This function is concave and is tight
$f(A) = \psi(\mathbf 1_A), \forall A$ at the vertices of the unit
hypercube, but is not the concave closure of $f$
\cite{vondrak2007submodularity}. The superdifferential
is given by
\begin{align}
\partial \psi(x) =
\left\{
(f(v_1|S_x), f(v_2|S_x), \dots, f(v_n|S_x))
: S_x \in \argmin_{S \subseteq V} 
[ f(S) + \sum_{v \in V} x(v) f(v|S)]
\right\}
\end{align}
and when evaluating at $x = \mathbf 1_A$ we have
$\mathcal M_A$
$\triangleq 
\argmin_{S \subseteq V} [ f(S) + \sum_{v \in V} \mathbf 1_A f(v|S)]$
$= \set{ A } \cup \set{ A' : A' = A - v, \forall v \in A }$.
To have an antitone supergradient, we need $\forall x \leq y$
and $g_x \in \partial \psi(x)$, $g_y \in \partial \psi(y)$,
that $g_x \geq g_y$. Taking $x = \mathbf 1_A$ and 
$y = \mathbf 1_{A + v}$ for some $v \notin A$, we
can choose $A \in \mathcal M_A$
and $A' = (A+v-v') \in \mathcal M_{A+v}$ with
$v' \in A$. In this case, we can find a monotone 
submodular function with 
$f(v_i|A) < f(v_i | A + v - v')$ which violates 
antitonicity. 

In order to explore this further, we consider the case where the
function $\psi$ is twice differentiable. In this case, if
$\psi$ is concave, then an antitone superdifferential means
for all $x \leq y$, we have for all $i$, 
$\frac{\partial \psi}{\partial x_i}(x) \geq 
\frac{\partial \psi}{\partial x_i}(y)$. 
Setting $y = x + \epsilon \mathbf 1_{v_j}$, we get for all $i,j$
\begin{align}
\frac{ \partial ^2 \psi }{\partial x_i \partial x_j} (x)
= \lim_{\epsilon \to 0}
\frac{ \frac{\partial \psi}{\partial x_i}(x + \epsilon \mathbf 1_{v_j}) - 
\frac{\partial \psi}{\partial x_i}(x) } { \epsilon} \leq 0,
\end{align}
which is thus also a sufficient condition
for $f(A) = \psi(\mathbf 1_A)$ being submodular.
The condition is stricter than necessary, however.  Consider the quadratic
$\psi : \mathbb R^2 \to \mathbb R$ with
$\psi(x) = \ x^T \left(\begin{smallmatrix}
    1 & -2 \\
    -2 & 1
\end{smallmatrix}
\right) x + 4 {\mathbf 1}^T x$. Since $\phi(0,0) = 0$, $\phi(0,1) = 5$,
$\phi(1,0) = 5$, and $\phi(1,1) = 6$, $f(A) = \phi(\mathbf 1_A)$ is
monotone submodular. Here, we have 
$\frac{ \partial ^2 \psi }{\partial x_1 \partial x_2} = -4$
but 
$\frac{ \partial ^2 \psi }{\partial x_i^2} = 2$ for $i \in \set{1,2}$.
Being submodular does not require the non-positivity of the
diagonal elements of the Hessian matrix.
In fact, the following weaker sufficient
condition for submodularity 
(an old result, going back
more than a hundred years~\cite{auspitz1889untersuchungen,
  edgeworth1897, samuelson1947foundations, lorentz1953inequality,
  samuelson1974complementarity, topkis1978minimizing,
  topkis1998supermodularity})
is well established:
\begin{theorem}
\label{thm:old_cross_2nd_deriv_neg}
Let $\phi : \mathbb R^n \to \mathbb R$ be a twice differentiable
function. If for all $i\neq j$ we have
$\partial^2 \phi/\partial x_i \partial x_j \leq 0$
then the function $f: 2^V \to \mathbb R$ where $f(A) = \phi(\mathbf 1_A)$
is submodular.
\end{theorem}

%

The above result is equivalent to $\partial \phi(x)/\partial x_j$
being decreasing in $x_i$ for all $i \neq j$.
This suggests that the antitone superdifferential condition can also
be weakened while still ensuring submodularity. Define
$d_i^\epsilon \psi (x) = \psi(x + \epsilon \mathbf 1_{v_i}) - \psi(x)$.
Then an antitone superdifferential is the same
as, for all $x \leq y$ having $d_i^\epsilon \phi(x) \geq d_i^\epsilon \psi(y)$
for all $i$ and $\epsilon > 0$. This implies
that $d_j^\epsilon d_i^\epsilon \psi(x) \leq 0$ for all $i,j$.
The weaker condition asks
that $d_j^\epsilon d_i^\epsilon \psi(x) \leq 0$ for all $i\neq j$, and $\epsilon > 0$,
and this is the same as
\begin{align}
\psi(x + \epsilon \mathbf 1_{v_i})
+ \psi(x + \epsilon \mathbf 1_{v_j})
\geq 
\psi(x + \epsilon \mathbf 1_{v_i} + \epsilon \mathbf 1_{v_j})
+ \psi(x)
\end{align}
which essentially is a restatement of the property of submodularity
but on the reals. Note that when $i=j$, this (and
$\partial^2 \phi/\partial x_i^2 \leq 0$ in the twice differentiable
case) asks for the function to be concave in the direction of each
axis, but submodularity, as Theorem~\ref{thm:old_cross_2nd_deriv_neg}
states, does not require this.  Indeed, submodularity is a
relationship between distinct variables, not a criterion on any one
particular variable.


The weaker condition (Theorem~\ref{thm:old_cross_2nd_deriv_neg}) is also not necessary for
concavity, as the aforementioned quadratic is neither concave nor convex.
Concavity requires non-positive definiteness of the Hessian matrix,
something that antitone maps do not ensure.  A map is any function
$h : \mathbb R^V \to \mathbb R^V$ and is antitone if for all
$x,y \in \mathbb R^V$, $(x-y)^T(h(x)-h(y)) \leq 0$ for all $x,y$.  Not
only does an antitone map alone not ensure concavity (a result
established originally in
\cite{rockafellar1966characterization,rockafellar1970maximal}), an
antitone map need not be a gradient field (a property that, if true,
would make it a conservative field). For an example related to
submodular functions, the multilinear extension~\cite{owen1972multilinear},
defined as:
\begin{align}
\tilde f(x) = \sum_{ S \subseteq V } f(V) \prod_{i \in S} x_i \prod_{j \in V \setminus S} (1 - x_j)
\end{align}
has the property that $\tilde f(\mathbf 1_A) = f(A)$ for all
$A \subseteq V$. It has been used as a extension of a submodular function, surrogate to the
true concave envelope, for use in submodular maximization
problems~\cite{feige2011maximizing,chekuri2014submodular,badanidiyuru2014fast}.
When $f$ is submodular, it has $\partial^2 \tilde f(x)/\partial x_i \partial x_j \leq 0$ for all
$i,j$, not only abiding Theorem~\ref{thm:old_cross_2nd_deriv_neg} but
also for $i=j$ it has $\partial \phi^2/\partial x_i^2 = 0$ since it is
multilinear. Hence, multilinear extension also has an antitone map,
but is also neither convex nor concave and hence has neither a
subdifferential nor a superdifferential.  Indeed, concavity is not at
all required for an extension of a submodular function, another well
known example being the \lovasz{} extension of
$\lex f : \mathbf R^V \to \mathbf R$ of $f$ which is a convex, has
$f(A) = \lex f(\mathbf 1_A)$, is defined as
$\lex f(x) = \sum_{i=1}^n x_{\sigma_i}f(\sigma_i | \sigma_1, \sigma_2,
\dots, \sigma_{i-1})$ where
$\sigma = (\sigma_1,\sigma_2, \dots, \sigma_n)$ is an $x$-dependent
order ensuring
$x_{\sigma_1} \geq x_{\sigma_2} \geq \dots \geq x_{\sigma_n}$.
$\lex f$ is not twice differentiable but it has a subgradient
$g \in \partial \lex f(x)$ where
$g(i) = f(\sigma_i | \sigma_1, \sigma_2, \dots, \sigma_{i-1})$.  Given
$x \leq y$, a decreasing order of $y$ can be arbitrarily different
than for $x$ implying $\partial \lex f(x)$ is neither antitone nor
monotone, so $d_i^\epsilon d_j^\epsilon\lex f(x) \leq 0$ is not a
property of the \lovasz{} extension.
Also, any function defined only on the vertices of the unit hypercube has
an infinite number of both concave and convex
extensions~\cite{crama2011boolean}.  The approach above shows that
antitone superdifferentials involves both concavity and submodular
functions. Since Theorem~\ref{thm:old_cross_2nd_deriv_neg} does not
require concavity, however, this suggests that there may be a way to
define submodular functions using generalized line integrals of
antitone maps without needing concavity~\cite{romano1993potential}.

We also note that Theorem~\ref{thm:old_cross_2nd_deriv_neg} is given
as a sufficient condition, but not a necessary condition, for
submodularity when we consider $\phi$ as a function used to produce
$f(A) = \phi(\mathbf 1_A)$.  Let $\phi$ be any function satisfying
Theorem~\ref{thm:old_cross_2nd_deriv_neg} and $\chi$ be any other
function having $\chi(\mathbf 1_A) = 0$ for all $A \subseteq V$.  Then
$f(A) = \phi(\mathbf 1_A) + \chi(\mathbf 1_A)$ is submodular while
$\phi(x) + \chi(x)$ need not satisfy the theorem.
Theorem~\ref{thm:old_cross_2nd_deriv_neg} is typically stated as both
necessary and sufficient conditions for submodularity
\cite{edgeworth1897, samuelson1947foundations,
  samuelson1974complementarity,topkis1978minimizing,topkis1998supermodularity},
as it is used to define submodularity on those lattices, including the
reals (and hence this is sometimes called continuous submodularity),
where twice differentiability everywhere is well defined. For example,
defining
$\partial_i f(A) = f(A \cup \set{i}) - f(A \setminus \set{i})$ for
$i \in V$, we have that a function $f: 2^V \to \mathbb R$ is
submodular if and only if for $i\neq j$,
$\partial_i \partial_j f(A) \leq 0$.  This is in contrast to how we
use it above, which to define a submodular function only on the unit
hypercube vertices starting from a function defined on $\mathbb R^n$.


Getting back to DSFs, since the concave function associated with a DSF
has an antitone superdifferential, and since this is sufficient but
not necessary for submodularity, this suggests (but does not
guarantee, since DSFs evaluate $\psi$ only at hypercube vertices
$\mathbf 1_A$) that the family of DSFs might not comprise all
submodular functions.  While in Section~\ref{sec:dsfs} we show that
DSFs generalize SCMMs, and in Section~\ref{sec:textdsf_k-1-subset} we
show that increasing the layers in a DSF increases the size of the
family, Section~\ref{sec:dsfs-cannot-do} shows, by giving an example,
that not all submodular function can be represented by DSFs.








%
%


In closing this section, we state an additional potential advantage of
DSFs.  Ordinarily the concave closure of a submodular function is
computationally hard to evaluate \cite{vondrak2007submodularity} and
this is disappointing since such a construct would be useful for
relaxation schemes for maximizing submodular functions (and as result
surrogates, such as the multilinear extension are used).  In the DSF
case, however, a particular concave extension is very easy to get,
namely $\psi_\rtnd(x) + \langle m_\pm , x \rangle$. This extension perhaps
could be useful for maximizing DSFs, possibly constrainedly, using
concave maximization followed by appropriate rounding methods.


\subsection{The Special Matroid Case and Deep Matroid Rank}
\label{sec:matroids-case}

We discuss in this section the special case of matroids and matroid
ranks as they motivate and offer insight to the results later in the
paper.

A matroid $M$~\cite{fujishige2005submodular} is a set system
$M=(V,\mathcal I)$ where $\mathcal I = \set{I_1, I_2, \dots }$ is a
set of subsets $I_i \subseteq V$ that are called independent. A
matroid has the property that $\emptyset \in \mathcal I$, that
$\mathcal I$ is subclusive (i.e., given $I \in \mathcal I$ and $I'
\subset I$ then $I' \in \mathcal I$) and that all maximally
independent sets have the same size (i.e., given $A, B \in \mathcal I$
with $|A| < |B|$, there exists a $b \in B \setminus A$ such that $A +
b \in \mathcal I$). The rank of a matroid, a set function $r: 2^V \to
\mathbb Z_+$ defined as $r(A) = \max_{I \in \mathcal I} |I \cap A|$,
is a powerful class of submodular functions. All matroids are defined
uniquely by their rank function as $\mathcal I = \set{ A : r(A) = |A|
}$ and therefore, we can reason about if two matroids are equivalent
or not based on if their ranks are equal, and vice verse.  All
monotone non-decreasing non-negative integral submodular functions can
be
exactly represented by grouping and then evaluating grouped ground
elements in a matroid~\cite{fujishige2005submodular}.

A useful matroid in machine learning applications
\cite{lin2011-submodular-word-alignment,bai-sgm-acm-bcb-2016} is the
partition matroid, where a partition $(V_1, V_2, \dots, V_\ell)$ of
$V$ is formed, along with a set of capacities $k_1, k_2, \dots, k_\ell
\in \mathbb Z_+$. It's rank function is defined as: $r(X) =
\sum_{i=1}^\ell \min( |X \cap V_i|, k_i)$ and, therefore, is an SCMM.

A cycle matroid is a different type of matroid based on a graph
$G=(V,E)$ where the rank function $r(A)$ for $A \subseteq E$ is
defined as the size of the maximum ``spanning forest'' (i.e., a
spanning tree for each connected component) in the edge-induced
subgraph $G_A = (V,A)$. From the perspective of matroids, we can
consider classes of submodular functions via their rank. If a given
type of matroid cannot represent another kind, their ranks lie in
distinct families. To study where DSFs are situated in the space of
all submodular functions, it is useful first to study results
regarding matroid rank functions.
\begin{lemma}
There are partition matroids that are not cycle matroids.
\end{lemma}
\begin{proof}
Consider the partition matroid over $|V|=4$ elements and
consider a partition with one block and a capacity
of two, so $r(X) = \min(|X|,2)$, so any two elements
has rank 2. For this matroid
to be a cyclic matroid, we must have a graph with 4 edges where
every set of three (out of those 4) must contain a cycle. Lets
name the edges $\aEL,\bEL,\cEL,\dEL$, then $\aEL,\bEL,\cEL$ contains a cycle
and so does $\aEL,\bEL,\dEL$, while $\aEL,\bEL$ does not contain a cycle ($\set{\aEL,\bEL}$ has rank 2).
The only way this can happen is if either $c,d$ are parallel edges,
or of $\cEL$ is parallel to one of $\aEL$ or $\bEL$,
and $\dEL$ is also parallel to one of $\aEL$ or $\bEL$,
or if $\cEL$ and $\dEL$ are loops. 
In any of the
above cases, we now have two edges that are parallel, or that contain loops, but they
must have rank 2, which is a contradiction.
\end{proof}

\begin{figure}[tb]
\centerline{\includegraphics[page=1,width=0.9\textwidth]{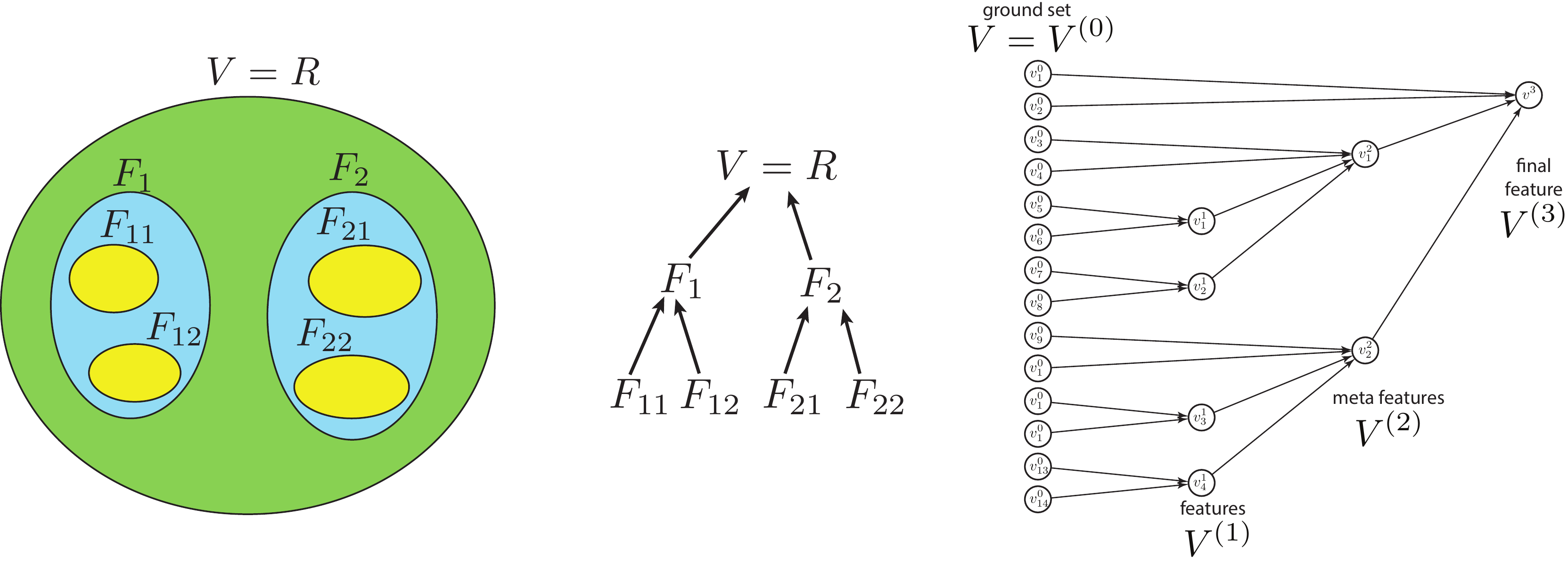}}
\caption{Visualization of a laminar matroid rank function: Left, the
  laminar family of subsets $\mathcal F = \set{ V, F_1, F_{11},
    F_{12}, F_2, F_{21}, F_{22} }$; Middle: the tree structure of the
  laminar family; Right: a possible corresponding DSF DAG associated
  with the laminar matroid rank function when $|V| = 14$.}
\label{fig:laminar_visualizations}
\end{figure}

In a {\em laminar matroid}, a generalization of a partition matroid,
we start with a set $V$ and a family $\mathcal F = \{
F_1,F_2, \dots, \}$ of subsets $F_i \subseteq V$ that is {\em laminar},
namely that for all $i\neq j$ either $F_i \cap F_j = \emptyset$ or
$F_i \subseteq F_j$ or $F_j \subseteq F_i$ (i.e., sets in $\mathcal F$
are either non-intersecting or comparable). In a laminar matroid, we
also have for every $F \in \mathcal F$ an associated capacity $k_F \in
\mathbb Z_+$. A set $I$ is independent 
if $|I \cap F| \leq k_F$ for all $F \in \mathcal F$. 
A laminar family of sets can be organized in a tree, where
there is one root $R \in \mathcal F$ in the tree that, w.l.o.g., can be 
$V$ itself. Then the immediate parents $\text{pa}(F) \subset \mathcal F$ of a 
set $F \in \mathcal F$ in the tree 
are the set of maximal subsets of $F$ in $\mathcal F$, i.e., 
$\text{pa}(F) = \set{ F' \in \mathcal F : F' \subset F \text{ and } \not\exists F'' \in \mathcal F \text{ s.t. } F' \subset F'' \subset F }$.
We then define the following for all $F \in \mathcal F$:
\begin{align}
\label{eq:laminar_matroid_rank}
r_F(A) = \min( \sum_{F' \in \text{pa}(F) } r_{F'}(A \cap F') +
| A \setminus \bigcup_{ F' \in \text{pa}(F) } F |, k_F ).
\end{align}
A laminar matroid rank has a recursive definition
$r(A) = r_R(A) = r_V(A)$. Hence, if the family $\mathcal F$ forms a
partition of $V$, we have a partition matroid. More interestingly,
when compared to Eqn.~\eqref{eq:gen_deep_submodular_recursive}, we see
that a laminar matroid rank function is an instance of a DSF with a
tree-structured DAG as shown in Figure~\ref{fig:deep_submodular}.
Thus, within the family of DSFs lie the truncated matroid rank
functions 
used to show information theoretic hardness for many constrained
submodular optimization problems~\cite{goemans2009approximating},
i.e., start with the partition matroid rank
$r(A)= \min( |A \cap R |, a ) + \min( |A \cap \bar R|, |\bar R| )
= \min( |A \cap R |, a ) + |A \cap \bar R|$ and
then truncate it as follows:
\begin{align}
f_R(A) = \min \set{ r(A), b } = \min \set{ |A|, a + |A \cap \bar R|,
  b } 
\label{eq:truncated_matroid_rank}
\end{align}
with $a < b$. This is a function where $f_R(R) = a$ and $f_R(A) > a$ for $A \neq R$ and $|A|=|R|$
and can be set up to have most size $\geq\! |R|$ sets $A$ valued at $f_R(A) = b$.  Since this
function is used to show hardness for many constrained submodular
minimization problems, and since DSFs generalize laminar matroid
ranks, this portends poorly for 
algorithms of the kind found in~\cite{jegelka2013reflection,nishihara2014convergence}
to achieve fast DSF minimization.

Laminar matroids are more general than partition matroids.
From the perspective of matroid rank, we have:
\begin{lemma}
Laminar matroids strictly generalize partition matroids
\label{lemma:laminar_matroids_bigger_than_partition_matroids}
\end{lemma}
\begin{proof}
Consider a simple laminar family $\mathcal F = \set{ V, B }$
where $k_V = 2$, $B \subset V$ with $k_B = 1$, and
$|B| \geq 2$ and $|V| \geq |B| + 2$ giving rank function
\begin{align}
  r(X) = \min ( \min(|X \cap B|,1) + |X \setminus B|, 2).
\end{align}
Suppose we are given any set of subsets $\set{C_i}_i$ 
of $V$ and corresponding integer capacities $\set{ k_i }_i$ giving
the submodular function:
\begin{align}
r_s(X) = \sum_i \min(|X \cap C_i|,k_i).
\end{align}
and suppose that $r_s(X) = r(X)$ which means
$r_s(X)$ must be a matroid rank function. Note
that $k_i \geq 1$ otherwise term $i$ is vacuous. 
The $C_i$ must be disjoint, for if not let
$C_i \cap C_j \neq \emptyset$, $i\neq j$ and pick $v \in C_i \cap C_j$, which
gives $r_s(v) \geq 2$ implying $r_s$ is not a matroid rank function.
Hence the sets $C_i$ must be disjoint and $r_s$ is a partition
rank function over $\cup_i C_i$. Choose two 
elements $b_1, b_2 \in B$. If
$b_1 \in C_i$ and $b_2 \in C_j$ for $i\neq j$ this
gives $r_s(\set{b_1,b_2}) = 2 \neq r(\set{b_1,b_2}) = 1$. Hence,
there is a unique $i$ such that $B \subseteq C_i$. Thus,
$k_i = 1$ since if not we would get $r_s(\set{b_1,b_2}) = 2$. If
there exists a $v \in C_i \setminus B$ then for any $b \in B$,
$r_s(v,b) = 1 \neq 2 = r(v,b)$. Hence, we must have $C_i = B$. Now
take $v_1, v_2 \notin B$ so that $r(\set{v_1,v_2}) = 
r_s(\set{v_1,v_2}) = 2$, but the term of $r_s$ involving
$B$ does not involve $v_1,v_2$ so that for $b \in B$, 
$r_s(\set{v_1,v_2,b}) = 3$ which is a contradiction. Hence,
a laminar matroid is a strict generalization of a partition
matroid.
\end{proof}

Since a laminar matroid generalizes a partition matroid, this augurs
well for DSFs generalizing SCMMs (a result we provide in
Theorem~\ref{sec:dsfs-extend-class}).  Before considering that, we
already are up against some limits of laminar matroids, i.e.:
\begin{lemma}
Laminar matroid cannot represent all cycle matroids.
\label{lemma:laminar_matroid_cannot_do_cycle_graphic}
\end{lemma}
\begin{proof}
Consider the cycle matroid over edges on $K_4$, hence
$M=(V,\mathcal I)$ with $|V| =6$, $V$ being the
set of edges, where
$r(X)=|X|$ for $|X| \leq 2$,
$r(X) = 2$ when $X$ is any 3-cycle,
$r(X) = 3$ for any acyclic $X$ with $|X| = 3$, 
and $r(X) = 3$ for $|X| > 3$. 
Consider the form
of the laminar matroid in Eqn.~\eqref{eq:laminar_matroid_rank}
and suppose $r_V(X) = r(X)$ for all $X$. 
W.l.o.g.,
we may assume $k_V = 3$. 
Suppose $\exists e \in V \setminus \cup_{F \in \text{pa}(V)} F$. Then
consider any 3-cycle $C$ involving $e$, and $r_V(C - e) = 2$ but since
no element of $\text{pa}(V)$ contains $e$, there is no truncation,
giving $r_V(C) = 3$, a contradiction. Hence, $V = \cup_{F \in
  \text{pa}(V)} F$. Given a 3-cycle $C = \set{\aEL,\bEL,\cEL}$, suppose there
exists an $F \in \text{pa}(V)$ with $\aEL \in F$ and $\bEL \notin F$ and $\cEL
\notin F$. Since we must have $r_V(\set{\bEL,\cEL}) = 2$ and $r_V(\set{\aEL}) =
1$, this implies $r_V(\set{\aEL,\bEL,\cEL}) = 3$, also a contraction. Hence,
any three cycle must be in one element of $\text{pa}(V)$, and by
transitive closure over the four intersecting three-cycles, all elements of $V$ must be in only one member of
$\text{pa}(V)$. This implies that $|\text{pa}(V)| = 1$ and the only
way to represent the 3-cycles is within that one term, $r_F(X)$. This
process then is applied recursively until we are left with the base
case, where the entire recursion boils down to the form $r_V(X) =
\min(r_F(X),3) = \min(\min(|X|,k_F),3) = \min(|X|,\min(k_F,3))$.  This
clearly cannot represent the cycle matroid rank function for any value
of $k_F \in \mathbb Z_+$.
\end{proof}
The proof technique is reminiscent of the back propagation method
used to train DNNs and hence we call it ``backprop proof'' --- it
recursively backpropagates required properties from the root though
each layer (in a DSF sense) of a laminar matroid rank until it boils
down to a partition matroid rank function, where the base case is
clear. The proof is elucidating since it motivates the proof of
Theorem~\ref{sec:dsfs-extend-class} showing that DSFs extend SCMMs. We
also have the immediate corollary.

\begin{corollary}
Partition matroids cannot represent all cycle matroids.
\end{corollary}
%

\subsection{Surplus and Absolute Redundancy}
\label{sec:defic-absol-redund}

In this section, we introduce and study the notion of the surplus of a
set as measured by a submodular function. The surplus is a useful
concept and will be used extensively to show, in
Section~\ref{sec:dsfs}, various properties of the DSF family.

\begin{definition}[Surplus and Absolute Redundancy]
For a function $f: 2^V \to \mathbf R$, we define
$\surp_f(A)$ as the surplus (or absolute redundancy) of a set $A\subseteq V$ by $f$ as follows:
\begin{align}
\surp_f(A) = \sum_{a \in A} f(a) - f(A)
\end{align}
\end{definition}

We call $\surp_f(A)$ the surplus of $A$ by $f$.  We use the term
``surplus'' under an interpretation where $A$ is a set of agents that
can perform their action either independently of each other, or may
perform their actions jointly and
cooperatively~\cite{topkis1998supermodularity}. If an agent $a \in A$
performs the action independently, the cost is $f(a)$ with
an overall cost of $\sum_{a \in A} f(a)$, 
while if the
agents $A$ perform the action cooperatively, the overall cost is
$f(A)$.  The difference $\surp_g(A)= \sum_{a \in A} f(a) - f(A)$ is
the surplus 
obtained by performing the actions $A$ cooperatively
rather than individually.
When $g$ is submodular, surplus is
never negative. Hence, performing the actions jointly leads overall to
profit.\footnote{In~\cite{topkis1998supermodularity}, surplus
is defined as $f(A) - \sum_{a \in A} f(a)$ where $f$ is a supermodular
function, but the same idea still applies.}

The idea of surplus has occurred before in the field of information
theory but under a different name --- in this case, $f(A) = H(X_A)$ is
the entropy function of a set of random variables indexed by the set
$A$. The quantity $\surp_f(A) = \sum_{a \in A} H(X_a) - H(X_A)$
is the average bit-length penalty between optimally coding the random variables in $A$
separately (as if they were independent) vs.\ optimally coding them
jointly. This can, thus, be called the absolute redundancy of the set
$A$. For the entropy function, this idea was first defined
in~\cite{mcgill1954multivariate}.\footnote{Incidentally, in 1954,
  \cite{mcgill1954multivariate} was also the first, to the authors
  knowledge, to provide inequalities on the entropy function that are
  identical to the submodularity condition.}  Absolute redundancy is
also called ``total correlation'' \cite{watanabe1960information} and
also the ``multi-information'' function
\cite{studeny1998multiinformation}.  Our notion of surplus
is not the same as \cite{ore1956studies} where they
define a quantity called ``deficiency'' the negative of which
may be considered a kind of surplus.
Since there may neither be a
statistical, information theoretic, nor economic interpretation, we
actually prefer the terms ``total interaction'' or ``combinatorial
interaction.''  In the below, if only for the sake of brevity,
we utilize the term ``surplus,'' but stress that it
applies to any submodular function whatever its interpretation.  We say that the function $g$
``gives surplus'' to a set $A$ whenever $\surp_g(A) > 0$ and otherwise
$A$ has ``no surplus.''

In the below, we explore a number of properties and introduce a number
of variants of surplus, all of which are useful later in the paper.

\begin{lemma}[Linearity of Surplus]
Let $f_1,f_2$ be two 
functions and $\alpha_1, \alpha_2 \in \mathbb R_+$. Then for any $A \subseteq V$
\begin{align}
\surp_{\alpha_1 f_1 + \alpha_2 f_2}(A) = \alpha_1\surp_{f_1}(A) + \alpha_2 \surp_{f_2}(A)
\end{align}
\label{thm:lindef}
\end{lemma}

\begin{lemma}[Surplus is Immune to Modularity]
Modular functions do not change surplus, i.e., when $m: V \to \mathbb R$
is a normalized modular function and $f$ is any set function:
\begin{align}
\surp_{f + m}(A) = \surp_{f}(A)
\end{align}
\end{lemma}
That modular functions do not influence surplus is useful to be able
to ignore the final modular function $m_{\pm}$ in a DSF when studying
its properties.

\begin{lemma}[Non-negativity of Surplus]
When $f$ is normalized ($f(\emptyset) = 0$) and submodular, then for all $A\subseteq V$,
$\surp_f(A) \geq 0$.
\end{lemma}
\begin{proof}
For any $A \subseteq V$, with $A = \set{a_1,a_2, \dots, a_k}$, 
\begin{align}
f(A) = \sum_{i=1}^k f(a_i | a_1, a_2, \dots, a_{i-1})
\leq \sum_{i=1}^k f(a_i)
\end{align}
\end{proof}
Thus, with a submodular function in such a context, therefore, there
can never be any deficit (negative surplus) and it is always
beneficial to act cooperatively. How fairly to redistribute surplus
back to the individual agents is called the ``surplus sharing problem''
and is studied in~\cite{topkis1998supermodularity}.

\begin{lemma}[Mixtures Preserve Surplus]
Let $f_1, f_2, \dots$ be a set of submodular functions
and $\alpha_1, \alpha_2, \dots $ be a set of positive real-valued weights,
and define $f = \sum_i \alpha_i f_i$ as their conic combination.
Then we have $\surp_f (A) > 0$ if and only if
$\exists i$ with $\surp_{f_i}(A) > 0$.
\label{thm:mixpresdef}
\end{lemma}
\begin{proof}
This follows when one considers
that $\forall i, \surp_{f_i}(A) \geq 0$ for all $A$, that $\forall i, \alpha_i > 0$, and
that
$\surp_f(A) = \sum_i \alpha_i \surp_{f_i}(A)$.
\end{proof}

The next theorem is particularly important for showing certain
properties of DSFs, in particular,
Corollary~\ref{thm:presdef}.
\begin{theorem}[Concave Composition Preserves Surplus]
  Let $h : 2^V \to \mathbb R$ be a polymatroid function and $\phi:
  \mathbb R \to \mathbb R$ be a normalized monotone non-decreasing
  concave function that is not identically zero.
  Define $g: 2^V \to \mathbb R$ as $g(A) =
  \phi(h(A))$. Then $\surp_h(A) > 0$ implies $\surp_g(A) > 0$.
\label{thm:comp_pres_surp}
\end{theorem}
\begin{proof}
Since $g(\cdot)$ is polymatroidal
(by Theorem~\ref{thm:concave_over_polymatroidal}), $\surp_g(A) \geq 0$ for all $A$.
Order $A$ arbitrarily
as $A = \set{ a_1, a_2, \dots, a_k }$ with $k=|A|$.
Then since $\sum_{i=1}^k h(a_i) > h(A)$,
\begin{align}
\sum_{i=1}^k \phi( h(a_i)) 
\stackrel{\text{(a)}}{\geq} \phi( \sum_{i=1}^k h(a_i)) 
\stackrel{\text{(b)}}{\geq}
\phi( h(A)),
\end{align}
where (a) follows from Theorem~\ref{thm:subadd_concave}
and (b) follows from the monotonicity of $\phi$. If
$\sum_{i=1}^k h(a_i)$ is still in the linear part of
$\phi(\cdot)$ then (b) is strict, while
if $\sum_{i=1}^k h(a_i)$ is greater than
the linear part of $\phi(\cdot)$ then,
from the second part of Theorem~\ref{thm:subadd_concave},
(a) is
strict. In either case, $\surp_g(A) > 0$.
\end{proof}


\begin{proposition}[Concave Composition Increases Surplus]
  Let $h : 2^V \to \mathbb R$ be a polymatroid function 
  with $h(v) = 1$ for all $v \in V$, and $\phi:
  \mathbb R \to \mathbb R$ be a normalized monotone non-decreasing
  concave function that is not identically zero and
  where $\phi(1) = 1$. Define $g: 2^V \to \mathbb R$ as $g(A) =
  \phi(h(A))$. Then for any $A$, 
  $\surp_g(A) \geq \surp_h(A)$.
\end{proposition}

\begin{definition}[Grouped Surplus]
We define a form of grouped surplus as follows. Given a set of
$m$ disjoint sets $A_1, A_2, \dots, A_m \subseteq V$, we define:
\begin{align}
I_f^{(m)}(A_1; A_2; \dots; A_m) \triangleq \sum_{i=1}^m f(A_i) - f(\bigcup_{i=1}^m A_i)
\end{align}
\label{def:pairdef}
\end{definition}
When $f(A) = H(X_A)$ is the entropy function, then the pairwise
surplus $I_f^{(2)}(A;B)$ is the well-known mutual information
\cite{cover2012elements} between random variable sets $X_A$ and $X_B$.
The grouped surplus can be defined in terms of standard surplus via
$I_f^{(m)}(A_1; A_2; \dots; A_m) = \surp_f( \set{ A_1 }, \set {A_2},
\dots, \set { A_m } )$ where we treat each of the sets $\set{ A_i }_i$
as a singleton element groups in the standard surplus.  Thus, for any
$m$, we have $I_f^m(A_1; A_2; \dots; A_m) \geq 0$ for any normalized
submodular function $f$.  We also have the following:

\begin{proposition}
\label{thm:zero_shatter_surplus_means_zero_grouped}
Given a submodular function $f$ and a set $A \subseteq V$,
if $\surp_f(A) = 0$ then $I_f^{(m)}(A_1; A_2; \dots ; A_m) = 0$ for any
$m$ and proper $m$-partition $A_1,A_2, \dots, A_m \subseteq A$ of $A$. Moreover,
we have:
\begin{align}
\surp_f(\bigcup_{i=1}^m A_i) > I_f^{(m)}(A_1;A_2; \dots; A_m)
\label{eq:foobarbaz}
\end{align}
\end{proposition}
For example, if $I_f^{(2)}(A;B) > 0$ then $\surp_f(A \cup B) > 0$.
The converse
is not true in general, i.e., we can have $I_f^{(2)}(A;B) = 0$ while
still having $\surp_f(B) > 0$.  
Of particular interest in this paper will be pairwise surplus of the
form $I_f^{(2)}(e'; C)$ where $C$ is a three-cycle of a graphic
matroid, and $e' \notin C$. When it is clear from the context, we will
drop the superscript $m$ and state
$I_f(A_1; A_2; \dots; A_m) \triangleq I_f^{(m)}(A_1; A_2; \dots;
A_m)$ for any $m$.
Considering Proposition~\ref{thm:zero_shatter_surplus_means_zero_grouped}
and Definition~\ref{def:pairdef}, we immediately obtain the following:
\begin{proposition}[Concave Composition Preserves Grouped Surplus]
  Let $h : 2^V \to \mathbb R$ be a polymatroid function and $\phi:
  \mathbb R \to \mathbb R$ be a normalized monotone non-decreasing
  concave function that is not identically zero.
  Define $g: 2^V \to \mathbb R$ as $g(A) =
  \phi(h(A))$. Then for any $m$ and any set of $m$ disjoint sets
  $A_1,A_2, \dots, A_m$, we have
  $I_h^{(m)}(A_1; A_2; \dots; A_m) > 0$
  implies
  $I_g^{(m)}(A_1; A_2; \dots; A_m) > 0$.
\label{thm:comppres_gsurp}
\end{proposition}

\begin{definition}[Modular at $B$]
\label{def:modular_at}
We say a function $h: 2^V \to \mathbb R$ is modular at $B \subseteq V$
if $h(B) = \sum_{b \in B} h(b)$.%
\end{definition}
When $h$ is modular at $B$, it does not necessarily 
mean that it is modular at some $A \subset B$. However, we
do have the following:
\begin{lemma}
If $h: 2^V \to \mathbb R$ is a submodular function. Then
$h$ is modular at all $A \subseteq B$ if and only if $\surp_h(B) = 0$.
\label{def:surplus_zero_modular_at_subsets}
\end{lemma}
\begin{proof}
If $h$ is modular for all $A \subseteq B$, then $h(A) = \sum_{a \in A} h(a)$,
and $\surp_h(B) = 0$. Conversely, suppose $h$ is submodular
and $\surp_h(B) = 0$ and let $A \subseteq B$ be given. Then
\begin{align}
h(B) = \sum_{b \in B} h(b) \geq h(A) + \sum_{b \in B \setminus A} h(b) \geq h(B)
\end{align}
Hence, all inequalities are equalities. Subtracting $\sum_{b \in B \setminus A} h(b)$
from both sides of the first inequality gives the result.
\end{proof}

\begin{lemma}[Forced Separation]
  Let $h: 2^V \to \mathbb R$ be a polymatroid function and $A,B,C$ be
  disjoint subsets where
  $I_h(A;B) = I_h(B;C) = I_h(C;A) = 0$.
  Then if $h(A) = 0$ then
  $I_h(A;B;C) = 0$.
\label{thm:forced_modular_on_cycle_early}
\end{lemma}
\begin{proof}
Consider the following:
\begin{align}
h(A) + h(B) + h(C)
&= h(B) + h(C) = h(B \cup C)
\leq h(A \cup B \cup C)  \\
&\leq h(A) + h(B \cup C) = h(B \cup C),
\end{align}
where the first equality is because $h(A) = 0$, the next is since
$I_h(B;C) = 0$,
the next (an inequality) is due to
monotonicity, the subsequent inequality is due to submodularity, and the final one is
since $h(A) = 0$. Hence, all inequalities are equalities, and
$I_h(A;B;C) = 0$. 
\end{proof}
\noindent As an example, if $A = \set{ \aEL }$,
$B = \set{ \bEL }$,
$C = \set{ \cEL }$,
then the consequence of the lemma is that
$h$ would be modular at the set $\set{ \aEL, \bEL, \cEL }$.

The next lemma shows how we can hold the surplus of a set accountable
either to the concave function of a concave composition function or
to somewhere else internal in the polymatroid function.
\begin{lemma}[When Concave Composition Is Linear]
  Let $h : 2^V \to \mathbb R$ be a polymatroid function and $\phi:
  \mathbb R \to \mathbb R$ be a normalized monotone non-decreasing
  concave function that is not identically zero.
  Define $g: 2^V \to \mathbb R$ as $g(X) =
  \phi(h(X))$ for any $X \subseteq V$. 
  Given two disjoint sets $A,B \subseteq V$
  where $g(A)>0$, $g(B) > 0$, and $I_g(A;B) = 0$,
  then any surplus $\surp_g(A) >0$ given to $A$
  is not due to any non-linearity in $\phi(\cdot)$
  but rather is due entirely to $h(\cdot)$. Moreover,
  $g(X) = \gamma h(X)$ for all $X \subseteq A \cup B$
  for some $\gamma > 0$.
\label{thm:when_concave_is_lin_early}
\end{lemma}
\begin{proof}
By Theorem~\ref{thm:comp_pres_surp},
$I_{g}(A;B) = 0$ implies that $I_{h}(A;B) = 0$.
Then we have
\begin{align}
\phi(h(A)) + \phi(h(B)) = \phi(h(A \cup B)) = \phi(h(A) + h(B)).
\end{align}
Also, $g(A) > 0 \Rightarrow h(A) > 0$
and $g(B) > 0 \Rightarrow h(B) > 0$. Hence, 
by Theorem~\ref{thm:concave_additivity}, 
$\phi(\cdot)$ is
linear in the range $[0,h(A) + h(B)]$.
\end{proof}

\section{The Family of Deep Submodular Functions}
\label{sec:dsfs}

\begin{figure}
\vspace{-1.9ex}
\centerline{\includegraphics[page=1,width=0.6\textwidth]{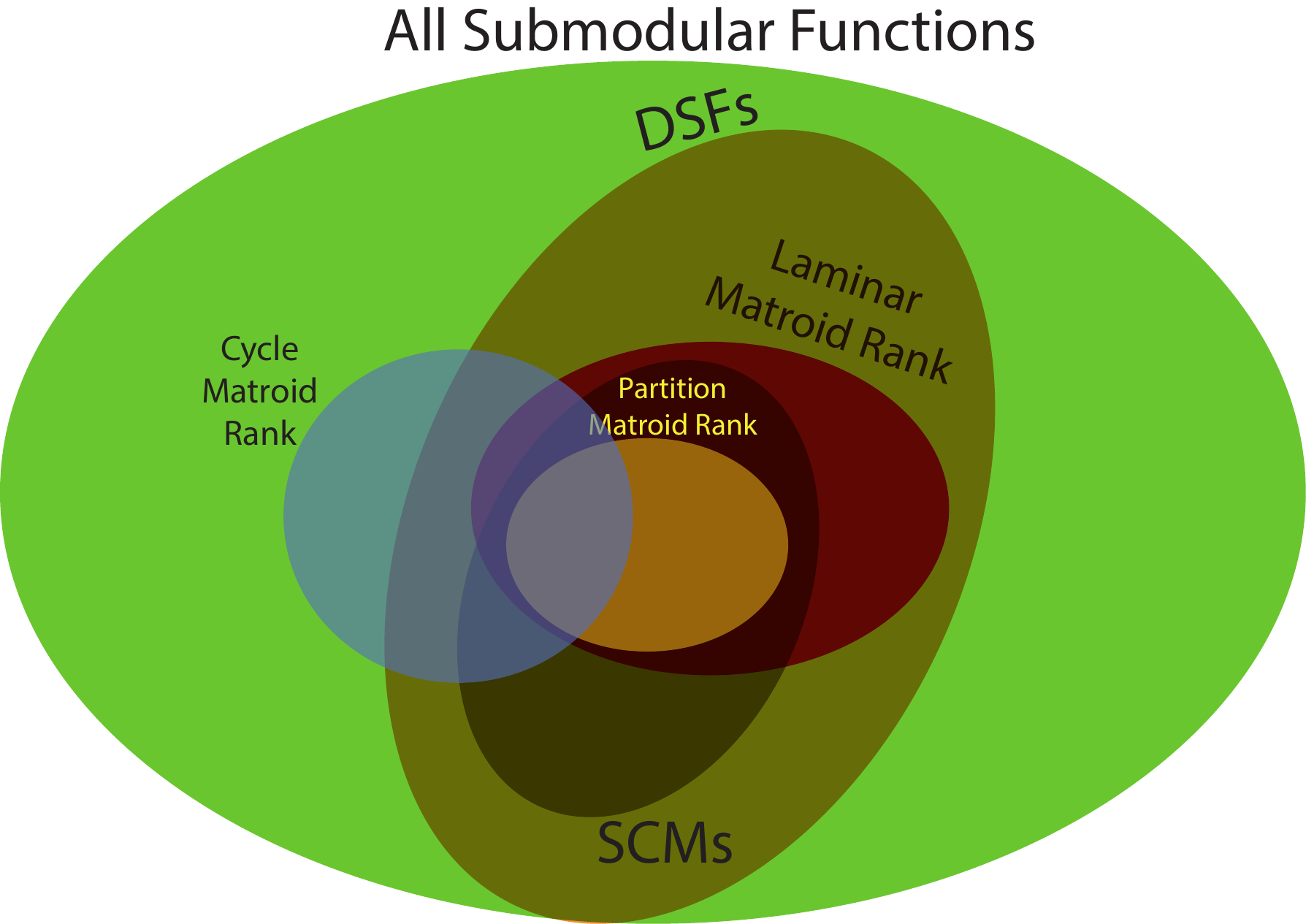}}
\caption{Containment properties of the set of functions
studied in this paper.}
\label{fig:families}
\end{figure}

We have seen that SCMMs generalize partition matroid rank functions
and DSFs generalize laminar matroid rank functions. We might expect,
from the above results, that DSFs might strictly generalize SCMMs ---
this is not immediately obvious since SCMMs are significantly more
capable than partition matroid rank functions because: (1) the concave
functions need not be simple truncations at integers, (2) each term
can have its own non-negative modular function, (3) there is no
requirement to partition the ground elements over terms in an SCMM, and
(4) we are allowed with SCMMs to add an additional arbitrary modular
function. We also have already seen Theorem~\ref{theorem:scm_in_scmm}
showing that SCMMs are a larger class of submodular functions than
just one concave over modular function and, in
Lemma~\ref{lemma:sums_weighted_cardinality_truncations}, that they
generalize weighted cardinality truncations.  SCMMs seem therefore to
be quite dexterous. The next several sections show, however, that DSFs
strictly generalize SCMMs.

More specifically, we formally place DSFs within the context of
general submodular functions.  We show in
Section~\ref{sec:dsfs-gener-scmms} that DSFs strictly generalize SCMMs
while preserving many of their attractive attributes (i.e.,
featurization, multi-modal, and amenability to learning, streaming,
and parallel optimization).  Then in
Section~\ref{sec:textdsf_k-1-subset}, we show that the family of DSFs
strictly grow with the number of layers uses. In
Section~\ref{sec:dsfs-cannot-do}, however, we show that the family of
DSFs still do not comprise all submodular functions.  We summarize the
results of this section in Figure~\ref{fig:families}, and that
includes familial relationships amongst other classes of submodular
functions (e.g., various matroid rank functions mentioned in
Section~\ref{sec:matroids-case}).

\subsection{DSFs generalize SCMMs}
\label{sec:dsfs-gener-scmms}

It is clear that DSFs contain at least the class of SCMMs since any
one-layer DSF is an SCMM. 
We next show that $\text{SCMM} \subset \text{DSF}$ holds, or that DSFs
strictly generalize SCMMs, thus providing justification for using DSFs
over SCMMs and, moreover, generalizing
Lemma~\ref{lemma:laminar_matroids_bigger_than_partition_matroids} to
the non matroid case.
The first DSF we choose is a laminar matroid, so SCMMs are unable to
represent laminar matroid rank functions even given their additional
flexibility over partition matroid rank functions. Since DSFs
generalize laminar matroid rank functions, the result follows.

It is not immediately apparent that DSFs generalize SCMMs as the following
example demonstrates. Consider the DSF $f: 2^V \to \mathbb R$ where
$V = \set{\aEL,\bEL,\cEL,\dEL,\eEL,\fEL}$:
\begin{align}
f(A) = 
\min\Bigl(
\min(|A \cap \set{\aEL,\bEL,\cEL}|,1)
+
\min(|A \cap \set{\dEL,\eEL,\fEL}|,1),
1.5\Bigr)
\end{align}
The reader is encouraged to ponder, for a moment, how one might
represent this DSF as an SCMM. Indeed, this is one case where
it is possible, as seen by the following SCMM $g: 2^V \to \mathbb R$
\begin{align}
g(A)
= 
\phi(|A \cap \set{\aEL,\bEL,\cEL}|)
+ \phi(|A \cap \set{\dEL,\eEL,\fEL}|) 
+ \min(|A|,0.5)  - 0.5|A|
\end{align}
where $\phi: \mathbf R \to \mathbf R$ is concave, with
$\phi(\alpha) = \min(\alpha, 0.5 + 0.5 \alpha)$. It can be verified
that $g(A) = f(A)$ for all $A \subseteq V$. In fact, an even simpler
SCMM does not use a modular function at all and puts
$g(A) = \frac{1}{2}(\min(|A \cap \set{\aEL,\bEL,\cEL}|,1) + \min(|A
\cap \{\dEL,\eEL,\fEL\}|,1)+\min(|A|,1))$.
From this example, one might naturally surmise that the DSFs unable to
be represented by SCMMs are obscure, contrived, and complicated. In
the next two sections, however, we show two fairly simple DSFs and
show that no SCMM can represent them.  Then in
Section~\ref{sec:more-gener-cond}, we provide more general conditions
describing when 2-layer DSF do or do not generalize SCMMs.

\subsubsection{The Laminar Matroid Rank Case}
\label{sec:laminar-matroid-rank}

Our first example DSF we choose is a simple laminar matroid on six
elements.  We show that SCMMs cannot express this laminar rank
function and since DSFs generalize laminar matroid ranks, the result
follows. Consider the following function $f: 2^V \to \mathbb R$
where $V = \set{\aEL,\bEL,\cEL,\dEL,\eEL,\fEL}$:
\begin{align}
f(A) = 
\min\Bigl(
\min(|A \cap \set{\aEL,\bEL,\cEL}|,2)
+
\min(|A \cap \set{\dEL,\eEL,\fEL}|,2),
3\Bigr)
\label{eq:laminar_rank_no_scmm}
\end{align}
The function is a laminar matroid rank function with 
$\mathcal F = \{ V, \set{\aEL,\bEL,\cEL}, \set{\dEL,\eEL,\fEL} \}$
and limits $k_V = 3$, $k_\set{\aEL,\bEL,\cEL} = 2$,
$k_\set{\dEL,\eEL,\fEL} = 2$.

In the following results, we assume that $g: 2^V \to \mathbb R$ is an
SCMM of the form $g(A) = \sum_{i\in \mathcal M} g_i(A) + m_\pm(A)$
where $g_i(A) = \phi_i(m_i(A))$ is a normalized monotone
non-decreasing concave function composed with a non-negative modular
function, $m_\pm(A)$ is an arbitrary normalized modular function, and
$\mathcal M$ is an index set.  Since $f$ itself is normalized, then
we must also have $g(\emptyset) = 0$ as well.  Also define
$B_1 = \set{\aEL,\bEL,\cEL}$ and $B_2 = \set{\dEL,\eEL,\fEL}$.

\begin{lemma}
  Suppose $f(A) = g(A)$ for all $A$. Then there does not exist
  an $i \in \mathcal M$
  where $g_i$ offers surplus both to $B_1$ and $B_2$
  (i.e., there exists no $i$ such that $\surp_{g_i}(B_1) > 0$
  and $\surp_{g_i}(B_2) > 0$.
\label{thm:partition_of_surp_laminar}
\end{lemma}
\begin{proof}
Suppose to the contrary that there exists such an $i$,
Then both $m_i(B_1)$ and $m_i(B_2)$ must both be past the
last linear point of $\phi_i$, say $\alpha_i$. We have that
\begin{align}
m_i(B_1) + m_i(B_2) 
&= m_i(\set{\aEL,\bEL,\cEL} \cup \set{\dEL,\eEL,\fEL})
= m_i(\set{\aEL,\bEL,\dEL} \cup \set{\cEL,\eEL,\fEL}) \\
&= m_i(\set{\aEL,\bEL,\dEL})+ m_i(\set{\cEL,\eEL,\fEL})
\end{align}
Since $m_i(B_1) > \alpha_i$ and
$m_i(B_2) > \alpha_i$ we must have
at least one of 
$m_i(\set{\aEL,\bEL,\dEL}) > \alpha_i$
or $m_i(\set{\cEL,\eEL,\fEL}) > \alpha_i$,
w.l.o.g., say $\set{\aEL,\bEL,\dEL}$.
This implies that $\surp_{g_i}(\set{\aEL,\bEL,\dEL}) > 0$
giving $g$ an unrecoverable surplus which is
a contradiction since $\surp_{f}(\set{\aEL,\bEL,\dEL}) = 0$.
\end{proof}

The next result is our first instance of a DSF that cannot be
represented by an SCMM.

\begin{lemma}
\label{lemma:laminar_rank_no_scmm}
No SCMM can represent the DSF in Equation~\eqref{eq:laminar_rank_no_scmm}.
\end{lemma}
\begin{proof}
For clarity, we offer the proof as a series of numbered statement groups.
\begin{enumerate}
\item 
Lemma~\ref{thm:partition_of_surp_laminar}
means that we can write $g$ as follows:
\begin{align}
g(A) = \sum_{i \in \mathcal M_{1}} g_i(A)
+ \sum_{i \in \mathcal M_{2}} g_i(A)
+ \sum_{i \in \mathcal M_{0}} g_i(A)
\end{align}
where $\mathcal M_0, \mathcal M_1, \mathcal M_2$ is a partition
of $\mathcal M$, and
where for all $i \in \mathcal M_1$, $g_i$ gives 
surplus to $B_1$ but not to $B_2$,
for all $i \in \mathcal M_2$, $g_i$ gives 
surplus to $B_2$ but not to $B_1$, and for
all $i \in \mathcal M_0$, $g_i$ gives surplus neither to $B_1$ nor
$B_2$. 
Hence, 
for all $i \in \mathcal M_1$
and $v \in B_1$ we have $g_i(v) > 0$,
and 
for all $i \in \mathcal M_2$
and $v \in B_2$ we have $g_i(v) > 0$ 
by Lemma~\ref{thm:forced_modular_on_cycle_early}.
Furthermore, since $B_1$ and $B_2$ are the only sets of size
three that are given a surplus, then for all $i \in \mathcal M_0$,
$\surp_{g_i}(A) = 0$ for all $A$ with $|A| \leq 3$.


\item We also need to have zero pairwise surplus such as:
\begin{align}
I_g(\eEL; \set{\aEL,\bEL,\cEL}) =
I_f(\eEL; \set{\aEL,\bEL,\cEL}) = 0
\end{align}
This implies that for $i \in \mathcal M$,
$I_{g_i}(\eEL; \set{\aEL,\bEL,\cEL}) = 0$.  Since we know that
$m_i(B_1)$ is past the non-linear part of $\phi_i$ for
$i \in \mathcal M_1$ and $m_i(B_2)$ is past the non-linear part of
$\phi_i$ for $i \in \mathcal M_2$, the only way to achieve this 
(and corresponding values such as 
$I_g(\bEL; \set{\dEL,\eEL,\fEL}) = 0$)
is if
both: (1) for $i \in \mathcal M_1$, $g_i(v) = 0$ when $v \in B_2$;
 and (2) for $i \in \mathcal M_2$, $g_i(v) = 0$ when $v \in B_1$.  

In other words,
$g_i$ with $i \in \mathcal M_1$ not only offers no surplus for
$B_2$ but also give zero valuation for any $v \in B_2$ (and vice verse).

\item 
  Consider the following set of size-four sets 
  $\mathcal A = \set{ A \subseteq V : |A \cap B_1| = |A \cap B_2| = 2 }$.
  Note that $|\mathcal A| = 9$.
  For any $A \in \mathcal A$, we have 
\begin{align}
\surp_f(A) = \surp_g(A) = \sum_{i \in \mathcal M} \surp_{g_i}(A) = 1.
\label{eq:needed_deficiit}
\end{align}
For $i \in \mathcal M_1 \cup \mathcal M_2$,
we have $\surp_{g_i}(A) = 0$ since two
elements of $A$ are given zero value to every such $g_i$.

Hence, the only terms that can achieve
Equation~\eqref{eq:needed_deficiit} are those $i$ within $\mathcal
M_0$ having $\surp_{g_i}(A) > 0$, where $m_i(A) > \alpha_i$, and where
$\alpha_i$ is the last linear part of $\phi_i$.  Also, to ensure no
unrecoverable surplus occurs, we must have that $m_i(C) \leq \alpha_i$ for any
$C$ having the following properties: (1) any size-three set; 
(2) any
size-four set $C$ with $|C \cap B_1| = 3$ and $|C \cap B_2| = 1$
(because $\surp_{g_i}(C \cap B_1) = 0$ and
$I_{g_i}(C \cap B_2; C \cap B_1) = 0$); and
(3) any size-four set $C$ with $|C \cap B_1| = 1$ and $|C \cap B_2| =
3$. For example, with $A = \set{ \aEL, \bEL, \dEL, \eEL }$
and $C = \set{ \aEL, \bEL, \cEL, \dEL }$, 
we  have that
\begin{align}
m_i(\aEL, \bEL, \cEL, \dEL) \leq
\alpha_i < m_i(\aEL, \bEL, \dEL, \eEL) = m_i(A)
\end{align}
implying that $m_i(\cEL) <
m_i(\eEL)$.

For any $A \in \mathcal A$, define $\mathcal A_2(A) = \set{ A' \in \mathcal A : |A' \triangle A|  = 2 }$
and $\mathcal A_4(A) = \set{ A' \in \mathcal A : |A' \triangle A|  = 4 }$.
Then $|\mathcal A_2(A)| = 4$,
$|\mathcal A_4(A)| = 4$, 
and $\mathcal A = \set{ A } \cup A_2(A) \cup A_4(A)$.
Suppose $m_i(A') > \alpha_i$ where 
$A' \in \mathcal A_4(A)$.
For example, with
$A = \set{ \aEL, \bEL, \dEL, \eEL }$ as above, and
$A' = \set{ \bEL,\cEL, \dEL,\fEL } \in \mathcal A_4(A)$ , this
implies that $m_i(\dEL, \eEL, \fEL, \bEL) \leq \alpha_i < m_i(\bEL, \cEL, \dEL,
\fEL)$ implying that $m_i(\eEL) < m_i(\cEL)$, a
contradiction with the above. Hence, we must have $m_i(A') \leq
\alpha_i$. 

More generally, $g_i$ offering surplus to more than one member of
$A_4(A)$ leads to a contradiction. Also, if $A' \in \mathcal A_4(A)$,
then $\exists A'' \in \mathcal A_4(A')$ with $A'' \neq A$ and $A'' \in
\mathcal A_4(A)$. For example,
with $A$ and $A'$ given as above,
$A'' = \set{a,c,e,f}$.
No more than one of this trio $\set{ A, A', A'' }$
can be offered surplus by the same $g_i$ for $i \in \mathcal
M_0$. This means that we may partition the indices $\mathcal M_0 =
\set{ \mathcal M_0^{(0)}, \mathcal M_0^{(1)}, \mathcal M_0^{(2)},
  \mathcal M_0^{(3)} }$ so that $i \in \mathcal M_0^{(1)}$ may give
surplus to $A$, but neither $A'$ nor $A''$, $i \in \mathcal M_0^{(2)}$
may give surplus to $A'$ but neither $A$ nor $A''$, $i \in \mathcal
M_0^{(3)}$ may give surplus to $A''$ but neither $A$ nor $A'$, and $i
\in \mathcal M_0^{(0)}$ gives no surplus any of the trio.


We must then have
\begin{align}
3 &= \surp_f(V) 
= \sum_{j \in \set{0,1,2}} \sum_{i \in \mathcal M_j } \surp_{g_i}(V)  \\
&\geq
\sum_{i \in \mathcal M_1 } \surp_{g_i}(B_1) 
+ 
\sum_{i \in \mathcal M_2 } \surp_{g_i}(B_2) 
+ 
\sum_{i \in \mathcal M_0 } \surp_{g_i}(V)  \\
&= 1 + 1 + 
\sum_{i \in \mathcal M_0^{(1)}} \surp_{g_i}(V) 
+ 
\sum_{i \in \mathcal M_0^{(2)}} \surp_{g_i}(V) 
+
\sum_{i \in \mathcal M_0^{(3)}} \surp_{g_i}(V) \\
&\geq 2+
\sum_{i \in \mathcal M_0^{(1)}} \surp_{g_i}(A) 
+ 
\sum_{i \in \mathcal M_0^{(2)}} \surp_{g_i}(A') 
+
\sum_{i \in \mathcal M_0^{(3)}} \surp_{g_i}(A'') \\
&= 2 + 1 + 1 + 1 = 5
\end{align}
which is a contradiction.

\end{enumerate}
\end{proof}

\subsubsection{A Non-matroid Case}
\label{sec:non-matroid-case}

Lest one thinks it is only the matroids that give difficulty to SCMMs,
consider the function $f:2^V \to \mathbf R$ where again
$V = \set{\aEL,\bEL,\cEL,\dEL,\eEL,\fEL}$.
\begin{align}
f(A) = \min\Bigl(
\min(|A \cap \set{\aEL,\bEL,\cEL,\dEL}|,3)
+
\min(|A \cap \set{\cEL,\dEL,\eEL,\fEL}|,3),
5
\Bigr)
\label{eq:non_matroid_dsf}
\end{align}
Here, there is an overlap between the two sets $B_1 = \set{\aEL,\bEL,\cEL,\dEL}$
and $B_2 = \set{\cEL,\dEL,\eEL,\fEL}$.  This is not a matroid rank since, for
example, $f(\cEL) = 2$. Also, minimal sets of maximum value are not all
the same size, e.g., $f(\set{\aEL,\cEL,\dEL}) = 5$ while
$f(\set{\aEL,\bEL,\cEL,\eEL}) = 5$.

\begin{lemma}
No SCMM can represent the DSF in
Equation~\ref{eq:non_matroid_dsf}.
\end{lemma}
\begin{proof}
For clarity, we offer the proof as a series of numbered statement groups.
\begin{enumerate}
\item   Assume that for all $A \subseteq V$,
  $f(A) = g(A) = \sum_{i \in \mathcal M} g_i(A)$
  for some index set $\mathcal M$.

\item Assume $\exists i \in \mathcal M$ that offers surplus both to
  $B_1$ and $B_2$. Let $\alpha_i$ be the last linear point in
  $\phi_i$. Then we must have
  $m_i(B_1) > \alpha_i$ and 
  $m_i(B_2) > \alpha_i$, leading to
\begin{align}
m_i(B_1) + m_i(B_2) = m_i(B_1 \cap B_2) + m_i(B_1 \triangle B_2)
\end{align}
where
$B_1 \triangle B_2= (B_1 \setminus B_2) \cup (B_2 \setminus B_1)$ is
the symmetric difference between $B_1$ and $B_2$.  Hence we must have
at least one of $m_i(B_1 \cap B_2) > \alpha_i$ or
$m_i(B_1 \triangle B_2) > \alpha_i$. Either case, however, would cause
an unrecoverable surplus for sets (either $B_1 \cap B_2$ or
$B_1 \triangle B_2$) neither of which should be in surplus.

\item We may partition the index set $\mathcal M$ in to
  $\mathcal M_0, \mathcal M_1, \mathcal M_2$ where $\mathcal M_1$ does
  not give a surplus to $B_2$, $\mathcal M_2$ does not give a surplus
  to $B_1$, and $\mathcal M_0$ gives surplus neither to $B_1$
  nor to $B_2$.

\item This leads to too much surplus, i.e.,

\begin{align}
1 &= \surp_f(V)
= \sum_{i \in \mathcal M} \surp_{g_i}(V)
= 
\sum_{i \in \mathcal M_0} \surp_{g_i}(V)
n+ 
\sum_{i \in \mathcal M_1} \surp_{g_i}(V)
+
\sum_{i \in \mathcal M_2} \surp_{g_i}(V) \\
&\geq
\sum_{i \in \mathcal M_1} \surp_{g_i}(B_1)
+
\sum_{i \in \mathcal M_2} \surp_{g_i}(B_2)
= 2
\end{align}
a contradiction.
\end{enumerate}
\end{proof}

\begin{exercise}
It is left to the reader to show that the following function can not
be represented as an SCMM:
\begin{align}
f(A) =
\min\Bigl(
\sum_{i=1}^4 \min(|A \cap B_i|,3)
,7
\Bigr)
\end{align}
where $V = \set{\aEL,\bEL,\cEL,\dEL,\eEL,\fEL,\gEL,\hEL}$
and where $B_1 = \set{\aEL,\bEL,\cEL,\dEL}$,
$B_2 = \set{\cEL,\dEL,\eEL,\fEL}$,
$B_3 = \set{\eEL,\fEL,\gEL,\hEL}$,
and 
$B_4 = \set{\gEL,\hEL,\aEL,\bEL}$.

It is also possible to construct a truncated matroid rank function of
the kind described in Equation~\eqref{eq:truncated_matroid_rank} that
cannot be represented by an SCMM.
\end{exercise}

Summarizing the results from the above sections, we have the
following.
\begin{theorem}
  The DSF family is strictly larger than that of SCMMs.
\label{sec:dsfs-extend-class}
\end{theorem}

A consequence of this theorem is that in order most generally allow
interaction amongst a hierarchy of concepts, as intuitively argued in
Section~\ref{sec:deep-subm-funct}, it not sufficient to use solely SCMMs.

\subsubsection{More General Conditions on Two-Layer Functions}
\label{sec:more-gener-cond}

In this section, we revisit again the form of DSF in
Equation~\eqref{eq:laminar_rank_no_scmm} where we saw there is no
corresponding SCMM. Let us slightly generalize 
Equation~\eqref{eq:laminar_rank_no_scmm} in the following.
\begin{align}
\label{eq:general_g_scmm_sometimes}
g(A) = \phi(\min(|A\cap\{a,b,c\}|,2)+\min(|A\cap\{d,e,f\}|,2))
\end{align}
where $\phi$ is normalized monotonically non-decreasing concave
function.  Lemma~\ref{lemma:laminar_rank_no_scmm} does not require that
for all $\phi$, the corresponding DSF has no SCMM representation.
Indeed, for certain functions $\phi$ it is possible. While we do not,
in this paper, give a complete characterization of those DSFs that can
or cannot be represented by SCMMs, we do offer the following theorem.
\begin{theorem}
\label{theo:concaveOverSumOfMin}
The function $g(A)$ in Equation~\ref{eq:general_g_scmm_sometimes} is
an SCMM if and only if $-\phi(1)+3.5\phi(2)-4\phi(3)+1.5\phi(4)\geq 0$
and $2\phi(1)+\phi(2)-4\phi(3)+2\phi(4)\geq 0$
\end{theorem}
The proof of the ``if'' part of this theorem follows by considering
the following expression which is clearly an SCMM as long as all of
the coefficient are non-negative.  The ``if'' part of the proof is
fairly easy --- we may simple write $g(A)$ as the form of SCMMs
as follows:
\begin{align}
g(A) &= \left[2\phi(1)+\phi(2)-4\phi(3)+2\phi(4)\right]\min(|A\cap\{a,b,c,d,e,f\}|,1)\\
&+\left[-\phi(1)+3.5\phi(2)-4\phi(3)+1.5\phi(4)\right]\min(|A\cap\{a,b,c,d,e,f\}|,2)\\
&+\left[-\phi(2)+2\phi(3)-\phi(4)\right]\left[\min(|A\cap\{a,b,c\}|,1)+\min(|A\cap\{d,e,f\}|,1)   \right]\\
&+\left[-\phi(3)+\phi(4)\right]\left[\min(|A\cap\{a,b,c\}|,2)+\min(|A\cap\{d,e,f\}|,2)   \right] \\
&+\left[-\phi(2)+2\phi(3)-\phi(4)\right][\modularminA{1,1,0,0.5,0.5,0.5}\\
&\quad+\modularminA{0,1,1,0.5,0.5,0.5} +\modularminA{1,0,1,0.5,0.5,0.5}]\\\
&\quad+\modularminA{0.5,0.5,0.5,1,1,0} +\modularminA{0.5,0.5,0.5,1,0,1}\\
&\quad+\modularminA{0.5,0.5,0.5,1,0,1}]
\end{align}
where $(x_a,x_b,x_c,x_d,x_e,x_f)^T$ is a modular function with
elements $x_a$, $x_b$, $x_c$, $x_d$, $x_e$, and $x_f$. Hence, if all coefficients are
non-negative, then $g$ is an SCMM (in fact, $g$ is a sum of weighted
cardinality truncations, defined in
Lemma~\ref{lemma:sums_weighted_cardinality_truncations}).  The
non-negativity of the coefficients holds whenever the inequalities
stated in the theorem are met.  The ``only if'' part of the theorem is
more involved and thus is given in Appendix~\ref{sec:more-gener-cond-1}.


\subsection{The DSF Family Grows Strictly with the Number of
Layers}
\label{sec:textdsf_k-1-subset}

It is clear that a $k$-layer DSF can easily express a $k-1$ layer
DSF simply by using a linear function at the final unit. Hence, if we
say that $\text{DSF}_k$ is the family of all deep submodular functions
with $k$ layers, we have that
$\text{DSF}_{k-1} \subseteq \text{DSF}_k$.  It is also clear that
$\text{DSF}_{0} \subset \text{DSF}_1$ since $\text{DSF}_{0}$ are
modular functions while $\text{DSF}_1$ are SCMMs. In the previous
section, we demonstrated by example that
$\text{DSF}_{1} \subset \text{DSF}_2$.

In this section, we show that DSFs become strictly more capable as the
allowable number of layers increases, meaning there are $k$-layer
functions that cannot be represented with $k-1$ layers, and hence
$\text{DSF}_{k-1} \subset \text{DSF}_k$ for any $k$.  This result is
similar to some of the recent results from the DNN literature where it
is shown that in some cases, it would require exponentially many
hidden units to implement a network with more layers
\cite{DBLP:journals/corr/EldanS15}. In the DSF case, however, we show
that in some cases, there is no way to represent certain $k$-layer
DSFs with a $k-1$ layer function, which means that the class of DSFs
is strictly increasing with the number of layers. This is different
than standard neural networks where it is shown that even a shallow
neural network is a universal
approximator~\cite{hornik1991approximation}.  In order to do this in
the DSF case, however, we allow the ground set correspondingly to
grow in size with the number of layers.

We begin with a number of definitions and prerequisite lemmas.

\begin{definition}[$(A,B,C)$-function]
We say that polymatroid function $f$ is an $(A,B,C)$-function if $A,B,C \subseteq V$ are
three non-empty disjoint subsets of $V$ and where $f$ satisfies
the following:
\begin{align}
f(A\cup B\cup C) &= f(A\cup B) =f(B\cup C)=f(C\cup A)\\
& =f(A)+f(B)=f(B)+f(C)=f(C)+f(A)
\end{align}
\end{definition}

\begin{definition}[strong $(A,B,C)$-function]
We say that $f$ is a strong $(A,B,C)$-function if 
$f$ is an $(A,B,C)$-function and if $f(A \cup B \cup C) > 0$.
\end{definition}

\begin{lemma}
\label{lemma:abc}
If $f$ is an $(A,B,C)$-function, then $f(A)=f(B)=f(C)$. If $f$ is a strong $(A,B,C)$-function, then $f(A)=f(B)=f(C)>0$.
\end{lemma}
\begin{proof}
$f(A)+f(B)=f(B)+f(C)=f(C)+f(A)$ implies $f(A)=f(B)=f(C)$. If $f$ is strong, we have $f(A\cup B\cup C)>0$. Therefore, $f(A)=\frac{1}{2}f(A\cup B\cup C)>0$.
\end{proof}

A simple example of such a function is a cycle matroid rank function
with $A = \set{ \aEL }$, $B = \set{ \bEL }$, and $C = \set{ \cEL }$,
where $\set{ \aEL,\bEL,\cEL }$ are the edges of a 3-cycle in the cycle
matroids associated graph. Note that in any $(A,B,C)$-function, we
have $I_f(A;B) = I_f(B;C) = I_f(C;A) = 0$. In a strongly
$(A,B,C)$-function, we have $I_f(A;B;C) > 0$. Hence, these functions
have no interaction between any two groups but there is a three-way
interaction amongst the three groups. Like surplus being
zero, $(A,B,C)$-function that are mixtures force properties amongst
the components.
\begin{lemma}
\label{lemma:add}
If $f=\sum_{i=1}^m f_i$ is an $(A,B,C)$-function,  then $f_i$ is an $(A,B,C)$-function for all $i$.
\end{lemma}
\begin{proof}
  First, conditioning on the pair $A,B$, since
  $\sum_{i=1}^m f_i(A|B\cup C)=f(A|B\cup C)=0$ and
  $f_i(A|B\cup C)\geq 0$ for each $i$, we have $f_i(A|B\cup C)= 0$ for
  each $i$. Doing the same for pair $B,C$ and $C,A$, we have
  $f_i(A\cup B)=f_i(B\cup C)=f_i(C\cup A)=f_i(A\cup B\cup C)$.

  Next, since
  $\sum_{i=1}^m f_i(A)+f_i(B)-f_i(A\cup B)=f(A)+f(B)-f(A\cup B)=0$ and
  $f_i(A)+f_i(B)-f_i(A\cup B)\geq 0$ for all $i$, we have
  $f_i(A)+f_i(B)=f_i(A\cup B)$ for all $i$. Doing the
  same for pairs $B,C$ and $C,A$ yields the result.
\end{proof}

\begin{figure}[tb]
\centerline{\includegraphics[page=1,width=0.8\textwidth]{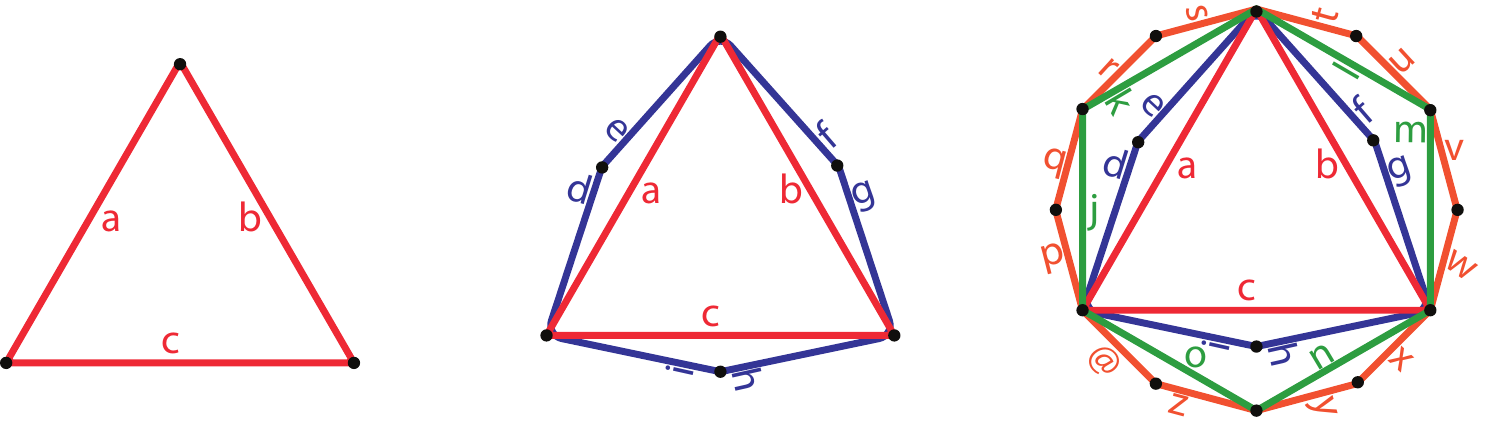}}
\caption{Example cycle matroids
whose rank functions are in $F_k$.
On the left we have
$k=1$ so $|V|=3$ where the example shows a cycle matroid on a graph
which is just a three-cycle. In the middle we have $k=2$ so $|V|=9$, where
$V_1 = \set{ \aEL, \dEL, \eEL }$,
$V_2 = \set{ \bEL, \fEL, \gEL }$,
and 
$V_3 = \set{ \cEL, \hEL, \iEL }$. The figure shows a cycle matroid
rank where each group $V_i$ is itself a three cycle.
On the right shows an example with $k=3$, $|V|=27$ where
$V_i$ for $i \in \set{1,2,3}$ is a set of nine elements,
and 
$V_{i,j}$ for $i,j \in \set{1,2,3}$ is a set of three elements comprising
a three-cycle. These examples demonstrate that $F_k$ is non-empty.
}
\label{fig:fk_cycle_matroid_functions}
\end{figure}

\begin{definition}
\label{definition:sub}
Given a function $f: 2^V \to \mathbb R$, and a subset $V'\subseteq V$,
define the restricted function $f_{V'}:2^{V'}\rightarrow \mathbb{R}$ as
$f_{V'}(X)=f(X)$ for all $X\subseteq V'$.
\end{definition}
\noindent A restricted function $f_{V'}(X)$ has a restricted
ground set, and by stating $f_{V'}(X)$  we assume $X \subseteq V'$.

\begin{lemma}
Let $h$ be polymatroidal, $\phi$ be normalized monotone non-decreasing
concave, and define $h(X) = g(X) + m_\pm(X)$, 
where $g(X) = \phi(h(X))$.
If $g$ is
a strongly $(A,B,C)$-function, then $h_D(X) = \gamma h(X) + m_\pm(X)$
for $D = A \cup B$, $D = B \cup C$, and $D = C \cup A$.
\label{lemma:strongly_w_concave_comp}
\end{lemma}
\begin{proof}
  Since $g$ is strongly $(A,B,C)$, we have $I_g(A;B) = 0$, while
  $g(A) = g(B) > 0$, which by Lemma~\ref{thm:when_concave_is_lin_early} means that $\alpha$, the last
  linear point of $\phi$, must be no less than $h(A,B)$. Hence, for any
  $X \subseteq A \cup B$, $h(X) = \gamma h(X)+ m_\pm(X)$ for some $\gamma > 0$. The
  same holds true for $B \cup C$ and $C \cup A$.
\end{proof}


Given $k\geq 1$ and a ground set $V$ where $|V|=3^k$, we name each element $v\in V$ as
$v_{a_1,a_2,\ldots,a_k}$ where $a_i \in \set{ 1,2,3}$ for $i=1,2,\ldots,k$. Define
$V_{a_1,a_2,\ldots,a_k}=\set{v_{a_1,a_2,\ldots,a_k}}$
and for $1\leq j\leq k-1$,
define $V_{a_1,a_2,\ldots,a_j} = V_{a_1,a_2,\ldots,a_j,1}\cup
V_{a_1,a_2,\ldots,a_j,2}\cup V_{a_1,a_2,\ldots,a_j,3}$.
For example, $V = V_1 \cup V_2 \cup V_3$,
$V_1 = V_{11} \cup V_{12} \cup V_{13}$, 
$V_2 = V_{21} \cup V_{22} \cup V_{23}$, 
$V_{11} = V_{111} \cup V_{112} \cup V_{113}$, 
and so on.

\begin{definition}
\label{definition:F}

We define $F'_k$ as the set of set functions
$f:2^V\rightarrow \mathbb{R}$ where $|V|=3^k$, $f(V)>0$ and $f$ is a
$(V_{a_1,a_2,\ldots,a_j,1}, V_{a_1,a_2,\ldots,a_j,2},
V_{a_1,a_2,\ldots,a_j,3})$-function for all $a_i\in \set{1,2,3}$,
$1\leq i\leq j$, and $0\leq j\leq k-1$.

We also define $F_k$ as the set of set functions $f:2^V\rightarrow \mathbb{R}$
where $|V|=3^k$,  and $f$ is a {\bf strongly}
$(V_{a_1,a_2,\ldots,a_j,1}, V_{a_1,a_2,\ldots,a_j,2},
V_{a_1,a_2,\ldots,a_j,3})$-function for all $a_i\in\{1,2,3\}$,
$1\leq i\leq j$, for all $0\leq j\leq k-1$.
\end{definition}

Figure~\ref{fig:fk_cycle_matroid_functions} shows three examples of
cycle matroids whose ranks are in $F_k$ for $k = 1,2,3$ thus
demonstrating that $F_k$ is non-empty. To show that there are DSFs
who are members of $F_k$, consider the following example.
\begin{example}
  Define $\hat{f}_k:2^{V_k}\rightarrow \mathbb{R}$, where
  $|V_k|=3^k$. Define $\hat{f}_1(X)=\frac{1}{2}\min(|X|,2)$. For
  $k\geq 2$, $V_k$ is partitioned into three sets 
  $V_{k1}$, $V_{k2}$, and $V_{k3}$
  where $|V_{k1}|=|V_{k2}|=|V_{k3}|=3^{k-1}$.
  The level-$k$ function is defined as $\hat{f}_k(X)=\frac{1}{2}\min(\sum_{i=1,2,3}\hat{f}_{k-1}(X\cap
  V_{ki}),2)$. 
\end{example}
Hence, $\hat f_k$ is like a $[0,1]$-normalized laminar matroid rank
function with the laminar family of sets 
$\mathcal F_k = \set{ V_k, V_{k1}, V_{k2}, V_{k3}, V_{k11}, V_{k12}, V_{k13}, V_{k21}, \dots }$. An immediate consequence is the following.
\begin{lemma}
$\hat{f}_k\in F_k$ and $\hat{f}_k$ can be expressed as a $k$-layer DSF. 
\end{lemma}

We also note that the families $F_k$ and $F'_k$ are the same.
\begin{lemma}
$F'_k=F_k$
\label{lemma:F_F_prime}
\end{lemma}

\begin{proof}
Immediately, we have $F_k\subseteq F'_k$

To show the other direction, assume there exists $f\in F'_k$ and
$v\in V$ such that $f(v)=0$ where $v$ is labeled as
$v_{a_1,a_2,\ldots,a_k}$. Then we have
$f(V_{a_1,a_2,\ldots,a_{k-1}})=2\times
f(V_{a_1,a_2,\ldots,a_{k-1,}a_k})=0$,
$f(V_{a_1,a_2,\ldots,a_{k-2}})=2\times
f(V_{a_1,a_2,\ldots,a_{k-2},a{k-1}})=0$, and so on until finally we
have $f(V)=0$ which contradicts with the definition of $F'_k$.  Hence,
for all $f\in F'_k$ and $v\in V$, we have $f(v)>0$ and
by monotonicity $f(A) > 0$ for all $A$. 
Therefore, $f\in F_k$ and $F'_k\subseteq F_k$.
\end{proof}

\begin{lemma}
\label{lemma:fk_sum_decomp}
Given $f\in F_k$, suppose that $f=\sum_{i=1}^m f_i$. If $f_i(V)> 0$,
then $f_i\in F_k$ for all $i$.
\end{lemma}
\begin{proof}
This is immediate when considering lemmas~\ref{lemma:add}
and~\ref{lemma:F_F_prime}.
\end{proof}

\begin{lemma}
\label{lemma:Fbasic}
Given $f\in F_k$, we have $\gamma f\in F_k$ for all $\gamma>0$. If $k\geq 2$, we have $f_{V_i}\in F_{k-1}$, for $i \in \set{ 1,2,3}$, where $V_i$ is defined in Definition~\ref{definition:F}. 
\end{lemma}
\begin{proof}
This is immediate from the definitions.
\end{proof}

\begin{lemma}
\label{lemma:f_k_concave_composition}
For all $f\in F_k$
and $\phi$ be a normalized monotone non-decreasing
concave function. 
If $f=\phi(f')$, then $f'_{V_i}\in F_{k-1}$, for $i \in \set{ 1,2,3}$. 
\end{lemma}
\begin{proof}
Using Lemma~\ref{lemma:strongly_w_concave_comp}, we have $f_{V_i}=\gamma f'_{V_i}$, where $\gamma>0$ is a constant. Also we have $f_{V_i}\in F_{k-1}$ according to second part of lemma~\ref{lemma:Fbasic}. So $f'_{V_i}\in F_{k-1}$ according to first part of lemma~\ref{lemma:Fbasic}.
\end{proof}

For any $f \in F_k$, we have that $f(v | V \setminus \set{v}) = 0$
which follows since if $v = v_{a_1,a_2, \dots, a_{k-1},1}$,
$v' = v_{a_1,a_2, \dots, a_{k-1},2}$, and
$v'' = v_{a_1,a_2, \dots, a_{k-1},3}$,
$0 = f(v|v',v'') \geq f(v|V \setminus \set{ v }) \geq 0$.  Hence, all
members of $F_k$ are totally normalized in this
sense~\cite{cunningham1983decomposition,cu84}.

As mentioned in Section~\ref{sec:deep-subm-funct}, a DSF allows for
the use of an arbitrary final modular function $m_\pm$ at the top
layer. If it is the case that a given $f \in F_k$ is represented as a
DSF, since $f$ is totally normalized and since the polymatroidal part
must have non-negative gain, the final $m_\pm$ must be non-positive as
otherwise we would have $f(v | V \setminus \set{v}) > 0$. Hence, in
order to show that a given $f \in F_k$ can not be represented by a DSF
with fewer than $k$ layers, it is sufficient to show that a function
of the form $f + m_+$, where $f \in F_k$ and $m_+$ is a non-negative
modular function, can not be expressed as a $k-1$ layer DSF having
$m_\pm = 0$.  To this end, we introduce the following class:

\begin{definition}
We define the class of functions $G_k = \set{ f + m_+ | f \in F_k, m_+ \in M_+ }$
where $M_+$ is the set of all non-negative normalized modular functions.
\end{definition}
The addition of a modular function to an $f \in F_k$ does not change
any surplus. Hence, for a $g \in G_k$ with $g = f + m_+$ with
$f \in F_k$, we have that $I_g(A;B) = I_f(A;B)$ for any disjoint sets
$A,B$, and that $\surp_g(A) = \surp_f(A)$ for any set $A$.

The properties of total
normalization~\cite{cunningham1983decomposition,cu84} will be further
useful in the below, so we define functional operators that totally normalize a
given function. Define the functional operator
$\mathcal M : (2^V \to \mathbb R) \to (V \to \mathbb R)$ that
maps from submodular functions to a modular function as follows,
for all $A \subseteq V$:
\begin{align}
(\mathcal M f)(A) = \sum_{a \in A} f(a| V \setminus \set{ a }).
\end{align}
Hence, $\mathcal M f$ is a modular function consisting of elements
which are the smallest possible gain given by submodular $f$.  
We also
define a total normalization functional operator
$\mathcal T: (2^V \to \mathbb R) \to (2^V \to \mathbb R)$ as follows:
\begin{align}
(\mathcal T f)(A) = f(A) - (\mathcal M f)(A).
\end{align}
Then clearly $\mathcal T f$ is a polymatroid function that is totally
normalized (i.e., $(\mathcal T f)(v | V \setminus \set{v}) = 0$), and
we have the identity $f = \mathcal T f + \mathcal M f$, meaning that
any submodular function can be decomposed into a totally normalized
polymatroid function plus a modular
function~\cite{cunningham1983decomposition,cu84}. 
The decomposition is unique because if $f = f' + m$ where
$f'$ is any function having $f'(v | V \setminus \set{v} ) = 0$, 
then $f(v | V \setminus \set{v}) = m(v)$ so we must have
that $m = \mathcal M f$.

The operator $\mathcal M$ is linear,
$\mathcal M(f_1 + f_2) = \mathcal M f_1 + \mathcal M f_2$, as is
$\mathcal T$. Also, in the present case, since $f$ is presumed
polymatroidal, the modular function is non-negative, i.e.,
$(\mathcal M f)(v) \geq 0$ for all $v$.

The next lemma states that if $f$ is representable as a sum, then each
term must either be a member of $G_k$ or must be purely a non-negative
modular function.
\begin{lemma}
Given $f \in G_k$, suppose
that $f = \sum_{i=1}^l f_i$. Then $f_i \in G_k \cup M_+$ for all $i$.
Furthermore, for at least one $i$, we have $f_i \in G_k$.
\label{lemma:gk_sum_decomp}
\end{lemma}
\begin{proof}
  Consider
  $\mathcal M f = \mathcal M \sum_{i=1}^l f_i = \sum_{i=1}^l \mathcal
  M f_i$ and
  $\mathcal T f = \mathcal T \sum_{i=1}^l f_i = \sum_{i=1}^l \mathcal
  T f_i$. For any $h \in F_k$ and $m \in M_+$,
  $\mathcal M(h + m) = m$, and hence
  $\mathcal T f = f - \mathcal M f \in F_k$.  Thus, by
  Lemma~\ref{lemma:fk_sum_decomp}, we have either that
  $\mathcal T f_i$ is identically zero or is otherwise an element of
  $F_k$.  Hence, when considering that
  $f_i = \mathcal M f_i + \mathcal T f_i$, if $\mathcal T f_i$ is
  zero, $\mathcal M f_i + \mathcal T f_i \in M_+$ and if not
  $\mathcal M f_i + \mathcal T f_i \in G_k$. Furthermore,
  since $f \in G_k$ we can not have that for all $i$, $f_i \in M_+$.
\end{proof}


\begin{lemma}
Given an $f \in G_k$, if $f = \phi(f')$, where $\phi$ is normalized
non-decreasing concave, and $f'$ is polymatroidal,
then $f'_{V_i} \in G_{k-1}$, $i \in \set{1,2,3}$.
\label{lemma:g_k_concave_composition}
\end{lemma}
\begin{proof}
Since $f \in G_k$, we have that 
we have $I_f(V_i;V_j) = 0$, 
for $i,j \in \set{1,2,3}$, $i\neq j$, 
while $g(V_i) = g(V_j) > 0$. This, Lemma~\ref{thm:when_concave_is_lin_early}, means that $\alpha$, the last
linear point of $\phi$, must be no less than $f'(V_i,V_j)$. Hence,
$f_{V_i} = \gamma f'_{V_i}$ for $i \in \set{1,2,3}$ and for some
constant $\gamma > 0$. 

Since $f = \mathcal M f + \mathcal T f$ and $f \in G_k$,
$\mathcal T f \in F_k$ and $\mathcal M f \in M_+$. Thus,
$(\mathcal T f)_{V_i} \in \mathcal F_{k-1}$ by
Lemma~\ref{lemma:Fbasic}, and we
also have that $(\mathcal M f)_{V_i} \in M_+$. 
Hence, since $f_X = (\mathcal M f)_X + (\mathcal T f)_X$ for
any $X \subseteq V$, we have
$f'_{V_i} = \frac{1}{\gamma}( (\mathcal M f)_{V_i} + (\mathcal T f)_{V_i}) \in G_{k-1}$.
\end{proof}

\begin{theorem}
\label{theo:Fk_k-1}
Any $f\in G_{k}$ can not be expressed via a $(k-1)$-layer DSF having $m_\pm = 0$.
\end{theorem}
\begin{proof}
We prove this by induction.

To establish the base case, all $f\in G_{1}$ can not be expressed via
a $0$-layer DSF since a $0$-layer DSF is modular while any $f \in G_1$
is not modular since there are sets that have strictly positive
surplus. Hence, the induction step assumes that any $f\in G_{k-1}$ can
not be expressed via a $(k-2)$-layer DSF for $k\geq 2$.

Next, suppose we find a $f\in G_{k}$ where $f$ can be expressed by a
$(k-1)$-layer DSF. Hence, we can express $f=\phi(f')$ where
$\phi(\cdot)$ is concave and where $f'=\sum_{i=1}^m f_i$.  Since $f$
is a $(k-1)$-layer DSF, then for all $i$, $f_i$ is a $(k-2)$-layer
DSF.

We may w.l.o.g., assume that $f_i(V)>0$ for all $i$ (since if for any
$i$ we have $f_i(V) = 0$, then it contributes nothing to the function
for any $A \subseteq V$ by monotonicity and non-negativity).  By
Lemma~\ref{lemma:g_k_concave_composition}, we have that
$f'_{V_1},f'_{V_2},f'_{V_3}\in G_{k-1}$.  For $j \in \set{ 1,2,3}$, we
have that $f'_{V_j} = \sum_{i=1}^m f_{i,V_j}$, and by
Lemma~\ref{lemma:gk_sum_decomp}, for all $i=1,2,\ldots, m$ and
$j \in \set{ 1,2,3}$, we have that $f_{i,V_j}\in G_{k-1} \cup
M_+$. Also, for each $j \in \set{ 1,2,3}$, there is at least one $i$
where $f_{i,V_j}\in G_{k-1}$.  For these instances, by the induction
step, $f_{i,V_j}$ can not be expressed in $(k-2)$-layer DSF. Since
$f_i$ is more complex than $f_{i,V_j}$, $f_i$ also can not be expressed
using a $(k-2)$-layer DSF, which contradicts the
above statement that $f_i$ is a $(k-2)$-layer DSF.

Hence, we can not find an $f\in G_{k}$ that can be expressed as a
$(k-1)$-layer DSF.
\end{proof}

The above results immediate imply our main theorem.
\begin{theorem}
There are $k$-layer DSFs that cannot be expressed using $k'$-layer DSFs
for any $k' < k$.
\end{theorem}

Letting $\text{DSF}_k$ be the family of $k$-layer DSFs, it is
interesting to consider what happens with
$\lim_{k \to \infty} \text{DSF}_k$. To show the above result, we
needed for the ground set to grow exponentially with $k$ which means
that for the flexibility of DSFs to grow, we need an ever increasing
ground set.  It remains an open question to determine if, when the
ground set size is constant and fixed, if $\text{DSF}_k$ comprises
a larger family, or if expressing certain $\text{DSF}_k$s with
$k-1$ layers requires an exponential number of hidden units, analogous
to~\cite{DBLP:journals/corr/EldanS15}.

\subsection{The Family of Submodular Functions is Strictly Larger than DSFs}
\label{sec:dsfs-cannot-do}


Our next result shows that, while DSFs are richer than SCMMs, and the
DSF family grows with the number of layers, they still do not
encompass all polymatroid functions. We show this by proving that the
cycle matroid rank function on $K_4$ is not achievable with DSFs.  We
adopt the idea of the backpropagation style proof in
Lemma~\ref{lemma:laminar_matroid_cannot_do_cycle_graphic} and utilize
the form of DSF given in Eqn.~\eqref{eq:gen_deep_submodular_recursive}
where we strip off the DSF layer-by-layer until we reach a one-layer
DSF that, as is shown, is unable to represent a cycle matroid rank
over $K_4$. In particular, we backpropagate a necessary lack of
surplus, a required linearity, and also a required pairwise surplus, from
the root down to the very first layer. This shows that, for up to size
three sets, the DSF must be similar to a mixture of concave over
modular, and which then is unable to maintain a pairwise surplus
necessary for the cycle matroid rank function. 

The reader is encouraged to review the notation in
Equation~\eqref{eq:gen_deep_submodular_recursive}.  We start with a
number of lemmas that culminate in
Theorem~\ref{thm:dfs_cannot_do_all_polymatroids}.


By applying Lemma~\ref{thm:lindef} and Theorem~\ref{thm:comp_pres_surp}
recursively according to a DSF's DAG, there are some important and
powerful implications for DSF with positive weights. Firstly, if we
ever find an internal network node and corresponding set in surplus,
it means some surplus is preserved all the way to the
root. Correspondingly, any set $A$ not in surplus by the network as a
whole must not be in surplus at any internal node. This allows us to
place constraints at one part of the network to cause consequences at
distant points (i.e., many layers away) elsewhere in the network.  For
a DSF (or SCMM), once a node is in surplus, there is no way to recover
anywhere else in the network (since there are no zero weights).  We
formalize this in the following:
\begin{corollary}[Preservation of Surplus]
If $\surp_{\psi_u}(A) > 0$ for some internal node $u$ in the DSF,
then $\surp_{\psi_v}(A) > 0$
where $v$ is a higher node (closer to the root $\rtnd$). 
In other words, if there
is no surplus at the higher node $v$ for some $A$, there can be no surplus at
any lower internal node in a DSF. This is also true for grouped surplus
(Definition~\ref{def:pairdef}).
\label{thm:presdef}
\end{corollary}
This result immediately follows Theorem~\ref{thm:comp_pres_surp}.
This means that zero surplus at the root $\surp_{\psi_\rtnd}(A) = 0$
on a set $A$ means all internal nodes must also have zero surplus on
$A$. For an SCMM, it means that if one term is in surplus then the sum
must also be in surplus.  This is a crucial result used in
Theorem~\ref{thm:dfs_cannot_do_all_polymatroids}.

\begin{corollary}[Modular on 3-Cycle]
  Let $f: 2^V \to \mathbb R$ be a DSF in the above form using the
  above notation, and assume $f(A) = r(A)$ where $r$ is a cycle
  matroid rank function over the edges of $K_4$.  Then for any
  $v \in \mathbf V$ and any 3-cycle $C = \set{\aEL,\bEL,\cEL}$ having
  $g_v(\aEL) = \psi_v(\mathbf 1_\aEL) = 0$, then
  $\surp_{g_v}(\set{\aEL,\bEL,\cEL}) = 0$ (i.e., $g_v$ is modular at the cycle
  $C$).
\label{thm:forced_modular_on_cycle}
\end{corollary}
\begin{proof}
This follows immediately 
from Lemma~\ref{thm:forced_modular_on_cycle_early}
where the three cycle consists of
edges $\set{\aEL,\bEL,\cEL}$
with $A = \set{ \aEL }$,
$B = \set{ \bEL }$,
$C = \set{ \cEL }$,
and $h = g_v$ which must be polymatroidal in a DSF for any $v \in \mathbf V$.
\end{proof}
%

\begin{lemma}[Linear Part of Hidden Units]
Let $f: 2^V \to \mathbb R$ be a DSF in the above form using
the above notation, and assume $f(A) = r(A)$ where $r$ is a
cycle matroid rank function over the edges of $K_4$. 
We are given any $v \in \mathbf V$, any 3-cycle $C = \set{a,b,c}$,
and any $e' \notin C$ having $g_v(e') > 0$, $g_v(C) > 0$,
and $I_{g_v}(e';C) = 0$. Then any surplus $\surp_{g_v}(C) > 0$ 
given to $C$ is not due to any non-linearity 
in $\phi_v(\cdot)$ and instead is caused by $\varphi_v(\cdot)$.
\label{thm:when_concave_is_lin}
\end{lemma}
\begin{proof}
Thus, since $w_{uv} \geq 0$ for all $u \in \text{pa}(v) \setminus V$,
and the modular part of $\varphi_v$ does not change pairwise surplus,
we may apply Lemma~\ref{thm:when_concave_is_lin_early}
with $g(X) = g_v(X)$, $h(X) = \varphi_v(\mathbf 1_X)$,
$A = \set{ e' }$, and $B = C$,
which means
the linear range of $\phi_v$
must include $[0,\varphi_v(\mathbf 1_C) + \varphi_v(\mathbf 1_{e'})]$.
\end{proof}

\begin{lemma}[Decomposition of sets of three-cycles]
  Let $f: 2^V \to \mathbb R$ be a DSF in the above form using the
  above notation, and assume $f(A) = r(A)$ where $r$ is a cycle
  matroid rank function over the edges of $K_4$.  We are given any $v
  \in \mathbf V$, and a subset $\text{cid}(v) \subseteq \set{1,2,3,4}$ of
  indices of the four three-cycles ($C_1$, $C_2$, $C_3$, and $C_4$) of
  the matroid where $|\text{cid}(v)| \geq 2$ and where
the following is true:
\label{thm:decomposition_of_sets_of_three_cycles}
\begin{enumerate}
\item For $i \in \text{cid}(v)$, $g_v(C_i) > 0$,
\item for $e \in \cup_{i \in \text{cid}(v)} C_i$, $g_v(e) > 0$,
\item for $e \notin \cup_{i \in \text{cid}(v)} C_i$, $g_v(e) = 0$,
\item and for $i \in \text{cid}(v)$, 3-cycle $C_i$ and $e \in C_i$,
  we have $I_{g_v}(e ; C_i \setminus \set{e}) =
  g_v(e) - g_v(e | C_i \setminus \set{e}) = g_v(e)$.

\end{enumerate}
Then we may for all $X$ of size up to three write $g_v(X)$
as 
\begin{align}
g_v(X) = \sum_{u \in U} w_u g_u(X)
\end{align}
with $w_u \geq 0$ and where for all $u \in U = \text{pa}(v) \setminus V$,
there is a set of cycle indices $\text{cid}(u) \subseteq \text{cid}(v)$
having:
\begin{enumerate}
\item For $i \in \text{cid}(u)$, $g_u(C_i) > 0$,
\item for $e \in \cup_{i \in \text{cid}(u)} C_i$, $g_u(e) > 0$,
\item for $e \notin \cup_{i \in \text{cid}(u)} C_i$, $g_u(e) = 0$,
\item and for $i \in \text{cid}(u)$, 3-cycle $C_i$ and $e \in C_i$,
,  we have $I_{g_u}(e ; C_i \setminus \set{e}) =
  g_u(e) - g_u(e | C_i \setminus \set{e}) = g_u(e)$.
\end{enumerate}
If $u$ is a first-layer hidden unit in the DSF then $|\text{cid}(u)| = 1$.
\end{lemma}
\begin{proof}[Proof of Lemma~\ref{thm:decomposition_of_sets_of_three_cycles}.]

For clarity, we offer the proof as a series of numbered statements.

\begin{enumerate}

\item 
  $g_v(\cdot)$ has to be modular on any set up to
  size two, as otherwise an unrecoverable surplus will occur by
  Corollary~\ref{thm:presdef}.  This means that $\phi_v$ has to be
  linear up to any valuation of any size two set (i.e.,
  $\varphi_\rtnd(\mathbf 1_X)$ is still in the linear part of
  $\phi_\rtnd(\cdot)$ for any $X$ with $|X|=2$).

\item 

  For the same reason, the nonlinear part of $\phi_v(\cdot)$ must not
  start before the valuation $\varphi_v(1_X)$ for any $X$ with $|X|=3$
  not in surplus (i.e., with $\surp_{g_v}(X) = 0$, any matroid independent set of
  size three).

\item 

Since $|\text{cid}(v)|\geq 2$, 
for any $i,j \in \text{cid}(v)$, 
corresponding three-cycles $C_i$,$C_j$ and 
any element $e' \in C_j$ where $e' \notin C_i$,
we have 
by Corollary~\ref{thm:presdef}
that 
\begin{align}
I_{g_v}(e'; C_i) = g_v(e') - g_v(e' | C_i ) = 0. 
\end{align}
Therefore, since $g_v(e') > 0$
and $g_v(C_i) > 0$, 
by Lemma~\ref{thm:when_concave_is_lin} 
the non-linear part of $\phi_v$ must not start before the valuation
$\varphi_v(\mathbf 1_{C_i})$ for any $i \in \text{cid}(v)$. 




%

\item For any $i \notin \text{cid}(v)$, $\exists e \in C_i$ with
  $g_v(e) = 0$. By Corollary~\ref{thm:forced_modular_on_cycle},
  this means $g_v$ is modular at $C_i$. 

\item 

Considering the two previous statements, the non-linear part of
$\phi_v(\cdot)$ must not start before the valuation
$\varphi_v(\mathbf 1_X)$ for any set with $|X|=3$. 
Since such an $X$
is still in the linear part of $\phi_v$ we may write $g_v(\cdot)$
as:
\begin{align}
  g_v(X) = 
\alpha_v \varphi_v(\mathbf 1_X) = 
\sum_{u \in \text{pa}(v) \setminus V}
\alpha_v w_{uv} \psi_u( \mathbf 1_X) + \alpha_v \langle  m_v, \mathbf 1_X \rangle
\label{eq:linear_gv_form}
\end{align}
for any $X$ up to size three, for some appropriate positive
constant $\alpha_v \in \mathbb R_+$.

\item 
We are given that for any $i \in \text{cid}(v)$, 
3-cycle $C_i$, and any $e \in C_i$, 
\begin{align}
  I_{g_v}(e; C_i \setminus e) 
  = g_v(e) - g_v(e | C_i \setminus \set{ e }) 
  = g_v(e) > 0.
\end{align}


From the previous statements, however, the surplus of any such 3-cycle
is not addressed by any non-linearity in $\phi_v$ and must instead
be handled by $\varphi_v$ which, since $g_v(e) > 0$,
means that
\begin{align}
0 
= g_v(e | C_i \setminus \set{e})
= \sum_{u \in \text{pa}(v) \setminus V} \alpha_v w_{uv}
  \psi_u(\mathbf 1_e | \mathbf 1_{C_i \setminus \set{e}}) + m_v(e) 
\label{eq:allgainszero}
\end{align}
Since
$\psi_u(\mathbf 1_e | \mathbf 1_{C_i \setminus \set{e}}) \geq 0$, for
all $u \in \text{pa}(v) \setminus V$, this requires
$0 = \psi_u(\mathbf 1_e | \mathbf 1_{C_i \setminus \set{e}}) = g_u(e |
C_i \setminus \set{e})$. Since $m_v(e) \geq 0$. this also implies that
$m_v(e) = 0, \forall e \in C_i$. Since $g_v(e) = 0$ for
$e \notin \cup_{i \in \text{cid}(v) } C_i$, we have that
$m_v(e) = 0, \forall e \in V$.  Hence, the above establishes that for
all $u \in \text{pa}(v) \setminus V$:
\begin{align}
I_{g_u}(e; C_i \setminus \set{e})
= g_u(e) - g_u(e | C_i \setminus \set{e}) = g_u(e)
\label{eq:u_equal_cond}
\end{align}





Next we need to consider whether $g_u(e) = 0$ or not.

\item If there is a $u \in \text{pa}(v) \setminus V$ and
  corresponding $i \in \text{cid}(v)$, 3-cycle $C_i$ having
  $\psi_u(\mathbf 1_e) = 0$ for some $e \in C_i$, then
  Lemma~\ref{thm:forced_modular_on_cycle} means that $\psi_u()$ must
  be modular at $C_i$.  But then we must have $\psi_u(\mathbf 1_{e'})
  = 0$ for $e' \in C_i \setminus {e}$ as otherwise, by modularity,
  we'd get $\psi_u(\mathbf 1_{e'} | \mathbf 1_{C_i \setminus \set{ e'
    }}) = \psi_u(\mathbf 1_{e'}) > 0$ 
  violating the requirement of
  Equation~\eqref{eq:allgainszero}.


\item Thus, this means that for every such $u$ and every 
  $i \in \text{cid}(v)$ and 3-cycle $C_i$,
  we have either $\forall e \in C_i, \psi_u(\mathbf 1_{e}) = 0$ or
  alternatively $\forall e \in C_i, \psi_u(\mathbf 1_{e}) > 0$, and in this
  latter case $u$ must give $C_i$ a positive surplus 
  (to satisfy Equation~\eqref{eq:allgainszero}).
  Any
  $u$ giving no surplus to any of the 3-cycles in $\text{cid}(v)$ thus must have
  $\forall e \in V, \psi_u(\mathbf 1_{e}) = 0$ and so
  can be removed from the network without effect (which we assume
  in the below).

\item Hence, for all $u$ there exists a set $\text{cid}(u) \subseteq
  \text{cid}(v)$ where for all $i \in \text{cid}(u)$, three-cycle
  $C_i$, and $e \in C_i$, we have $g_u(e) > 0$, $g_u(C_i) > 0$.  For
  $e \notin \cup_{i \in \text{cid}(u)} C_i$, $g_u(e) = 0$,

\item If $u$ is one of the first layer hidden unit nodes, then $g(A) =
  \phi_u(w_u(A))$ is a simple concave over modular function $w_u : V
  \to \mathbb R_+$.  Suppose that for this $u$, we have
  $|\text{cid}(u)| > 1$, then taking $i,j \in \text{cid}(u)$, $i\neq
  j$, $i,j \in \text{cid}(v)$, corresponding three-cycles $C_i$,$C_j$
  and any element $e' \in C_j$ where $e' \notin C_i$, we require by
  Corollary~\ref{thm:presdef} that $I_{g_u}(e'; C_i) = g_u(e') -
  g_u(e' | C_i ) = 0$.  By Lemma~\ref{thm:when_concave_is_lin}, the
  non-linear part of $\phi_u$ must not start before the valuation of
  $w_u(C_i)$, meaning $\phi_u(w_u(C_i))$ is modular on the cycle,
  contradicting Equation~\eqref{eq:allgainszero}. Hence, we must have
  $|\text{cid}(u)| = 1$ for first layer hidden nodes.

\end{enumerate}

\end{proof}

\begin{theorem}[DSFs are unable to represent the cycle matroid rank function on edges of $K_4$]
\label{thm:sccms_cannot_do_cycle_matroid_ranko}
\end{theorem}
\begin{proof}

Let $f: 2^V \to \mathbb R$ be a DSF in the above form.
We may, w.l.o.g., assume all weights are strictly positive, as the
  summations below will be based on $u \in \text{pa}(v)$,
  so we assume that for all $u \in \text{pa}(v)$,
  $w_{uv} > 0$.

  Consider, in Eqn.~\eqref{eq:gen_deep_submodular_recursive} , the top
  layer concave function along with the arbitrary modular function, and
  suppose that $f(A) = \psi_\rtnd(\mathbf 1_A) + m_\pm(A) = r(A)$ for
  all $A$ where $r: 2^V \to \mathbb Z_+$ is a cycle matroid rank
  function on $K_4$. Hence, $g_\rtnd(A) = \psi_\rtnd(\mathbf 1_A) =
  r(A) - m_\pm(A)$ which is an assuredly polymatroidal part of
  $f(A)$.



Let $C_1$, $C_2$, $C_3$, and $C_4$ be the four three-cycles of
  the matroid.  Note that for all $i$, we have $\surp_f(C_i) > 0$ for
  all $i$, and $f(C_i) > 0$. Also, for all $e \in V$, $f(e) > 0$.
  Hence, define $\text{cid}(\rtnd) = \set{1,2,3,4}$. By Theorem~\ref{thm:decomposition_of_sets_of_three_cycles},
  for any set $X$ with $|X| \leq 3$, we may write $g_\rtnd(X)$ as follows:
\begin{align}
  g_\rtnd(X) = \sum_{u \in U} w_u g_u(X)
\end{align}
where $\text{cid}(u) \subseteq \text{cid}(\rtnd)$, and where 
for all $u \in U$, 
$i \in \text{cid}(u)$, we have $g_u(C_i) > 0$, 
$g_u(e) > 0$ for $e \in \cup_{i \in \text{cid}(u)} C_i$,
and 
$g_u(e) = 0$ for $e \notin \cup_{i \in \text{cid}(u)} C_i$.
Hence we may write $g_\rtnd(X)$ as:
\begin{align}
  g_\rtnd(X) = \sum_{u \in U: |\text{cid}(u)| = 1} w_u g_u(X)
+ \sum_{u \in U: |\text{cid}(u)| > 1} w_u g_u(X)
\end{align}
For any $u \in U$ with $\text{cid}(u) > 1$, by 
Theorem~\ref{thm:decomposition_of_sets_of_three_cycles},
we may,
for any set $X$ of size $|X| \leq 3$, write it as:
\begin{align}
  g_u(X) = \sum_{u' \in U'} w_{u'} g_{u'}(X)
\end{align}
where $\text{cid}(u') \subseteq \text{cid}(u)$. Thus, we have
\begin{align}
  g_\rtnd(X) &= \sum_{u \in U: |\text{cid}(u)| = 1} w_u g_u(X) \\
&+ \sum_{u' \in U' : |\text{cid}(u')| = 1 } w_{u'} g_{u'}(X)
+ \sum_{u' \in U' : |\text{cid}(u')| > 1 } w_{u'} g_{u'}(X)
\end{align}
This process may continue recursively,
applying Theorem~\ref{thm:decomposition_of_sets_of_three_cycles} each time,
until we reach all
units in the bottom layer of the DSF. We are
guaranteed termination since the DSF is itself
finite size. Also, since the bottom layer
consists of single concave composed with modular functions,
all have $\text{cid}(\cdot) = 1$. Hence, for $X$ with $|X| \leq 3$, the entire DSF
can be expressed as:
\begin{align}
  g_\rtnd(X) = \sum_{u \in U^{(\ell)} } w_u g_u(X)
\end{align}
where $\text{cid}(u) = 1$ and where we
may partition 
$U^{(\ell)}$ in to four disjoint sets corresponding
to the four cycles, where in each index set
we have surplus only of one of the cycles. This
means that it is not possible to achieve, for
a cycle $C$ and $e \in C$, 
\begin{align}
I_{g_\rtnd}(e ; C \setminus \set{e}) =
g_\rtnd(e) - g_\rtnd(e | C \setminus \set{e}) = g_\rtnd(e) = 1
\end{align}
since some of the terms in the sum are non-zero
meaning $g_\rtnd(e | C \setminus \set{e}) > 0$,
thus contradicting that $f(X) = r(X)$ for all $X \subseteq V$.
\end{proof}

The above results therefore imply the following.
\begin{corollary}[$\text{SCCMs} \subset \text{DSFs} \subset \text{Submodular Functions}$]
\label{thm:dfs_cannot_do_all_polymatroids}
The family of SCMMs is smaller than that of DSFs,
and the family of DSFs is smaller than the family of all submodular
functions. That is, let $C_n$ be the set of all submodular functions
over ground set $V$ of size $n$ and let $\text{DSF}_k$ be the family
of DSFs with $k$ layers on $V$, and $\text{SCCM}$ be the family
of SCCMs on $V$ with an arbitrary number of component functions. Then, for any $k$,
$\text{SCCM} \subset \text{DSF}_k \subset \mathcal C_n$.
\end{corollary}


While DSFs do not comprise all submodular functions, a consequence of
Theorem~\ref{thm:submod_antitone_super_on_vector_poly} is that the
input to a DSF can be any set of polymatroid functions. Let $f$ be a
DSF with $k$ inputs and a ground set $V = \set{1,2,\dots, k}$. Then we
can consider the standard way to utilize a DSF, in the context of
Theorem~\ref{thm:submod_antitone_super_on_vector_poly}, as one where
the $i$th input is a function $g_k(A) = \mathbf 1_{k \in A}$ which is
modular. Theorem~\ref{thm:submod_antitone_super_on_vector_poly} allows
for any polymatroid function to be used as input to a DSF, not just an
indicator function, and hence the DSF can be used to add interactions
between and perhaps improve these functions in some way. Hence, if
several of the $g_k$ are cycle matroid rank functions, and if the DSF
is learnt, the resulting family is expanded to include at least those
matroid ranks used as input.  It remains an open question to see if
there is a small finite fixed set of input polymatroid functions that
can be cascaded into a DSF in order to achieve all submodular
functions.

It is also worth noting that in~\cite{zhang1997non} it is shown that
the entropy function $f(A) = H(X_A)$ when seen as a set function must
satisfy inequalities that are not required for an arbitrary
polymatroid function, thus implying that entropy also does not
comprise all submodular function. An additional open problem,
therefore, is to compare the family of DSFs to that of entropy
functions.

\section{Applications in Machine Learning and Data Science}
\label{sec:applications}

In this section, we describe a number of possible DSF applications in
machine learning and data science.

\subsection{Learning DSFs}
\label{sec:learn-deep-subm}

As mentioned in Section~\ref{sec:introduction}, recent
studies~\cite{goemans2009approximating,balcan2011learning,feldman2013representation,feldman2013optimal}
show that learning submodular functions can be easier or harder
depending on the learning setting.

A general outline of various learning settings is given
in~\cite{kearns1994toward,feldman2013representation} --- here, we give
only a very brief overview. To start, learning may involve several
families of functions $\mathcal F$, $\mathcal H$, and $\mathcal T$
members of which are mappings from $2^V$ to $\mathbb R$.  There is
some true function $f \in \mathcal F$ to be learnt based on
information obtained via samples of the form $(A,f(A))$ for
$A \subseteq V$. One wishes to produce an approximation
$\tilde f \in \mathcal H$ to $f$ that is good in some way. Learning
submodular functions has been studied under a number of possible
variants. For example, there is typically a probability distribution
$\textbf{Pr}$ over subsets of $V$ (i.e.,
$\textbf{Pr}(\mathbf S = A) \geq 1$ and
$\sum_{A \subseteq V} \textbf{Pr}(\mathbf S = A) = 1$ where
$\mathbf S$ is a random variable). A set of samples
$\mathcal D = \set{ (A_i,f(A_i) }_i$ is obtained via this
distribution. The distribution $\textbf{Pr}$ might be unknown
\cite{balcan2011learning}, or might be known (and in such case, might
be assumed to be
uniform~\cite{feldman2013representation,feldman2013optimal}).  The
quality of learning could be judged over all $2^n$ points or over some
fraction, say $1-\beta$, of the points, for $\beta \in [0,1]$. In
general, there is no specificity on the particular set of points, or
the particular kind of points, that should be learnt as long as at
least a (probability distribution measured) fraction $1-\beta$ of them
are learnt.  Learning itself happens with some probability
$1-\delta$. I.e., there is some probability $\delta$ that the learning
will not succeed.  While learning asks for a function in
$\tilde f \in \mathcal H$ that is good, we might judge $\tilde f$
relative only to the best function $\hat f \in \mathcal T$ (the
touchstone class).  For example, in agnostic learning
\cite{kearns1994toward}, we acknowledge that it might be difficult to
show that learning is good relative to all of $\mathcal F$ (say due to
noise) but still feasible to show that learning is good relative to
the best within $\mathcal T$.  Also, there are a variety of ways to
judge goodness.  In~\cite{balcan2011learning}, goodness is judged
multiplicatively, meaning for a set $A \subseteq V$ we wish that
$\tilde f (A) \leq f(A) \leq g(n)f(A)$ for some function $g(n)$, and
this is typically a probabilistic condition (i.e., measured by
distribution $\textbf{Pr}$, goodness, or
$\tilde f (A) \leq f(A) \leq g(n)f(A)$, should happen on a fraction at
least $1-\beta$ of the points). Alternatively, goodness may also be
measured by an additive approximation error, say by a norm. I.e.,
defining
$\text{err}_p(f,\tilde f) =\| f - \tilde f \|_p = ( E_{A \sim
  \textbf{Pr}} [ {| f(A) - \tilde f(A) |}^p ] )^{1/p}$, we may wish
$\text{err}_p(f,\tilde f) < \epsilon$ for $p =1$ or $p=2$.  In the PAC
(probably approximately correct) model, we probably ($\delta > 0$)
approximately ($\epsilon > 0$ or $g(n) > 1$) learn ($\beta = 0$) with
a sample or algorithmic complexity that depends on $\delta$ and
$g(n)$.  In the PMAC (probably mostly approximately correct)
model~\cite{balcan2011learning}, we also ``mostly'' $\beta > 0$ learn. In
agnostic learning, $\mathcal F \supseteq \mathcal H = \mathcal T$.
Let $\mathcal C_n$ be the space of all submodular functions. In some
cases $\mathcal F \supseteq \mathcal C_n = \mathcal H$ so we wish to
learn the best submodular approximation to a non-submodular
function. In other cases,
$\mathcal F = \mathcal C_n \subseteq \mathcal T \subseteq \mathcal H$
meaning we are allowed to deviate from submodularity as long as the
error is small.

In the machine learning community, $\mathcal H$ may be
a parametric family of submodular functions. 
For example,
given a fixed set of component submodular functions, say
$\set{ f_i }_{i=1}^\ell$ one may with to learn only the weights of a mixture
$\set{ w_i }_i$ to produce $f = \sum_i w_i f_i$ where $w_i \geq 0$ for
all $i$ to ensure submodularity is preserved. What is learnt is only
the coefficients of the mixture, not the components, so the
flexibility of the family is determined by the diverseness and
quantity of components used. Empirically, experiments that learn 
submodularity for various data science
applications~\cite{sipos2012large,hui2012-submodular-shells-summarization},
has been more successful than simply hand-designing a
fixed submodular function. This is true both for image
\cite{sebastian2014-submod-image-sum} and document
\cite{hui2012-submodular-shells-summarization} summarization tasks.
There also has been some initial work on learnability bounds in
\cite{hui2012-submodular-shells-summarization}.
Learning just the mixture coefficients of a mixture of submodular
functions, while keeping the component functions themselves fixed, is
only as flexible as the set of component functions allows,
however. Given a small (or indiscriminately selected and hence
potentially redundant) number of components, the family over which one
can learn might be limited. As a result, one might need add a very
large number of components before one obtains a sufficiently powerful
family.


An alternative approach to learning a mixture that alleviates to some
extent the above problem is to learn over a richer parametric family,
and this is where DSFs hold promise.  An approach to learning DSFs,
therefore, is to learn within its parametric family, so
$\mathcal H = \text{DSF}_k$ for some finite $k$ and where
$f_w \in \text{DSF}$ is parameterized by the vector $w$ that
determines the topology (e.g., number and width of layers) of the
network, the numeric parameters (set of matrices) within that
topology, and the set of concave functions $\set{ \phi_u }_u$.  As
shown in the present paper, DSFs represent a strictly larger family than
SCMMs. Therefore, even in the mixture case above where the components
may also be learnt, there are DSFs that are unachievable by SCMMs.  In
addition, by Theorem~\ref{thm:submod_antitone_super_on_vector_poly}, a
DSF rather than a mixture can be applied to a fixed set of input
submodular components (e.g., some of which might be simple
indicators of the form $g_u(A) = \mathbf 1_{u \in A}$ and others could
be cycle matroid rank functions in order to reduce any chance of the
unachievability mentioned in
Theorem~\ref{thm:sccms_cannot_do_cycle_matroid_ranko}).  Even in cases
where a DSF can be represented by an SCMM, DSFs may be a far more
parsimonious representation of classes of submodular functions and
hence a more efficient family over which to learn, analogous to
results in DNNs showing the need for exponentially many hidden units
for shallow networks to implement a network with more layers
\cite{DBLP:journals/corr/EldanS15}.

Suppose $f \in \mathcal C_n$ is a target submodular function,
$f_w \in \text{DSF}_k$ is a parameterized $k$-layer DSF,
$\mathcal D = \set{ (S_i, y_i) }_i$ is a training set consisting of
subsets $S_i \subseteq V$ and valuations $y_i = f(S_i)$ for the target
function and that is drawn from distribution $\textbf{Pr}$. 
An empirical risk minimization (ERM), or regression,
style of learning is obtained a standard way:
\begin{align}
  \min_{w \in \mathcal W} J(w) = \sum_i L( y_i, f_w(S_i)) + \| w\|
\end{align}
where $L(\cdot,\cdot)$ is a loss function and $\|w \|$ is a norm on
the parameters. Obvious candidates for the loss would be squared loss,
or L1 loss, and the norm can also be chosen to prefer smaller values
for $w$. Given the objective $J(w)$ one may proceed using, for
example, projected stochastic gradient descent, where at each step we
project the weights $w$ into $\mathcal W$ which corresponds to the
non-negative orthant for parameters other than $m_\pm$ to ensure
submodularity is retained.  Under this approach, and with an
appropriate regularizer, it may be feasible to obtain generalization
bounds in some form \cite{shalev2010learnability} as is often found in
statistical machine learning settings.  Note that, depending on the
loss $L$ used, this approach may be tolerant of noisy estimates of the
function, where, say, $y_i = f_w(S_i) + \epsilon$ and where $\epsilon$
is noise, somewhat analogous to how it is possible to optimize a noisy
submodular function~\cite{hassidim2016submodular}. Alternatively, one
could analyze it under an agnostic learning setting.


Under many distribution assumptions, such as when $\textbf{Pr}$ is the
uniform distribution over $2^V$, then as the training set gets larger,
we approach the case where there are $O(2^{|V|})$ distinct samples,
and the goal is to learn the function at all points. For large ground
sets, certain learning settings might become infeasible in practice
due to the curse of dimensionality.  As mentioned above, there are
learning settings that ask only for a fraction $1-\beta$ of the points
to be learnt, but without a mechanism to specify which fraction.


In many practical learning situations, however, access to an oracle
function $h(A)$, or training data that utilizes $h$'s evaluations,
might not be available. Even if $h$ available, such a learning setting
might be overkill for certain applications, as we might not need a
submodular function $f_w$ to be accurate at all points
$A \subseteq V$.  One example is in summarization
applications~\cite{hui2012-submodular-shells-summarization,sebastian2014-submod-image-sum}
where we wish to learn a submodular function $f_w$ that, when
maximized subject to a cardinality constraint, produces a set that is
valuated highly by the true submodular function relative to other sets of
that size.  Such a set should be diverse and high quality. In this
case, one does not need $f_w$ to be an accurate surrogate for $f$
except on sets $A$ for which $f$ is large.  More precisely, instead of
trying to learn $f$ everywhere, we seek only to learn the parameters
$w$ of a function
so that if $B \in \argmax_{A \subseteq V: |A| \leq k} f_{w}(A)$, then
$h(B) \geq \alpha h(A^*)$ for some $\alpha \in [0,1]$ where
$A^* \in \argmax_{A \subseteq V: |A| \leq k} h(A)$. This setting puts
fewer constraints on what is needing to be learnt than the regression
approach and hence should correspondingly be easier.  This is somewhat
analogous to discriminative learning where the entire distribution
over input and output variables is not needed and instead only a
conditional distribution (or a deterministic mapping from input to
output) is required.



The max-margin
approach~\cite{sipos2012large,hui2012-submodular-shells-summarization,sebastian2014-submod-image-sum}
is appropriate to this problem and is applicable to learning DSFs.
Given an unknown but desired non-negative submodular function
$f \in \mathcal C_n$, we are given a set of representative
sets $\mathcal{S} = \set{S_1, S_2, \dots}$, with $S_i \subseteq V$ and
where each $S \in \mathcal{S}$ is scored highly by $f(\cdot)$.  Unlike
the regression approach, we do not need the actual evaluations $f(S_i)$.
It might be,
for example, that the sets are selected summaries chosen by a human
annotator from a larger set. A matroid analogy is to learn a matroid
using a set of independent sets of a particular
size, say $\ell$. If $M' = (V, \mathcal I')$ is a matroid of rank
$\ell' > \ell$, then $M = (V, \mathcal I)$ is also a matroid where
$\mathcal I = \set{ I \in \mathcal I' : |I| \leq \ell }$.


In max-margin approach, we learn the parameters $w$ of $f_{w}$ in an attempt
to make, for all $S\in \mathcal{S}$, $f_{w} (S)$ high, while for $A \in 2^V$,
$f_{w}(A)$ is lower by a given loss. More precisely, we ask that
for $S\in \mathcal{S}$ and $A\in 2^V$,
$f_{w}(S) \geq f_{w}(A) + \ell_S(A)$.  The loss is
chosen so that $\ell_S(S) = 0$, so that $\ell_S(A)$ is very small
whenever $A$ is close to $S$ (e.g., if $A$ is also a good summary),
and so that $\ell_S(A)$ is large when $A$ is considered much worse
(e.g. if $A$ is a poor summary).  Achieving the above is done by
maximizing the loss-dependent margin, and reduces to finding
parameters so that
$f_{w}(S) \geq \max_{A\in 2^V}[ f_{w}(A) + \ell_S(A)]$
is satisfied for $S\in\mathcal{S}$.  The task of finding the
maximizing set is known as loss-augmented inference
(LAI)~\cite{taskar2005learning,yu2009learning}, which for general $\ell(A)$ is
NP-hard. With regularization, the optimization becomes:
\begin{align}
  \underset{w \in \mathcal W}{\text{min}}\ \sum_{S\in \mathcal{S}}
  \mathcal L \left(
  \underset{A \in 2^V }{\text{max}}\left[ f_w(A) + \ell_S(A)
  \right] - f_w(S) \right) + \frac{\lambda}{2}||w||^2_2 .
\end{align}
where $\mathcal L$ is a classification loss function such as the logistic 
($\mathcal L(x) = \log(1+\exp(-x))$) or
hinge ($\mathcal L(x) = \max(0, x)$) loss.  If it is the case that $f_w(S)$ is
linear in $w$ (such as when $w$ are mixture parameters in an SCMM as
was done
in~\cite{sipos2012large,hui2012-submodular-shells-summarization,sebastian2014-submod-image-sum}),
and if the maximization can is done exactly, then this constitutes a
convex minimization procedure. In general, however, there are several
complications.

Firstly, the LAI problem $\max_{A\in 2^V}[ f_{w}(A) + \ell_S(A)]$ may
be hard. Given a submodular function for the loss, as was done
in~\cite{hui2012-submodular-shells-summarization}, then the greedy
algorithm offers the standard $1-1/e$ approximation guarantee for LAI.
On the other hand, a submodular function is not always natural for the
loss. Recall above that $\ell_S(A)$ should be large when $A$ is
considered a poor set relative to $S$ (e.g. if $A$ is a poor
summary). If it is the case that one may get an assessment of $A$, say
via a surrogate $\tilde f$ of the ground truth function $f$, then one
may use $\ell_S(A) = \kappa - \tilde f(A)$ but this, to the extent
that $\tilde f$ needs to represent $f$, approaches the labeling needs
of the ERM/regression approach above. If $\tilde f$ is submodular,
then $\kappa - \tilde f$ is supermodular, and in this case solving
$\max_{A\in 2^V\backslash\mathcal{S}}[ f(A) + \ell(A)]$ involves
maximizing the difference between two submodular functions, and the
submodular-supermodular
procedure~\cite{narasimhan2005-subsup,rkiyeruai2012} can be used
although this procedure does not have guarantees in general.

Secondly, when $f_w$ is not linear in $w$, the above problem is not
convex. Given the enormous success of deep neural networks in
addressing non-convex optimization problems, however, this should not
be daunting. Indeed, given an estimation to
$\tilde A \in \argmax_{A\in 2^V}[ f_{w}(A) + \ell_S(A)]$, we can
easily obtain an approximate subgradient of weights
$dw \in \partial_w ( f_w(\tilde A) - f_w(S) + \lambda/2 \|w\|_2^2)$ to
be used in a projected stochastic subgradient descent procedure.  For
a DSF, this subgradient can be easily computed using backpropagation,
similar to the approach of~\cite{Pei2014}.  Like in the mixtures case,
we must use projected descent to ensure $w \in \mathcal W$ and
submodularity is preserved.  Recall, however that the weights
corresponding to $m_{\pm}(A)$ may be left negative if they so choose.
Preliminary experiments in learning DSFs in this fashion were reported
in \cite{brian2016nips-dsf-def-learning} and show encouraging results.

As an additional benefit, many of the concave functions mentioned in
Section~\ref{sec:feat-based-funct} are parameterized themselves, and
these parameters may also be the target of stochastic gradient based
learning. In such case, not only the weights but also the concave
functions of a DSF may be learnt.

Given the ongoing research on the non-convex learning of DNNs, which
have achieved remarkable results on a plethora of machine learning
tasks \cite{LeCun2015,goodfellow2016deep}, and given the similarity
between DSFs and DNNs, we may leverage the same DNN learning
techniques to learn DSFs. This includes stochastic gradient descent,
convolutional linear maps, momentum, dropout, batch normalization,
unsupervised pre-training, learning rate scheduling such as
AdaGrad/Adam, convolutional matrix patterns, mini-batching, and so
on. In some cases these methods might need to be modified (e.g.,
stochastic projected gradient descent to ensure the function remains
submodular). Moreover, the suitability of fast GPU computing to the
matrix-matrix multiplications necessary to evaluate DSFs should also
be a benefit. Lastly, the many toolkits that support DNN training (such
as Tensorflow, Theano, Torch, Caffe, CNTK, and so on), and that
include automatic symbolic differentiation and semi-differentiation
(for non-differentiable functions) for backpropagation-style parameter
learning can easily be used to train DSF. All of these techniques and
software may be leveraged to DSF's benefit, and is true both for the
regression and max-margin setting.







\subsubsection{Training and Testing on Different Ground Sets, and Multimodal Submodularity}
\label{sec:training-test-set}

\begin{figure}[tbh]
\centerline{\includegraphics[page=4,width=0.8\textwidth]{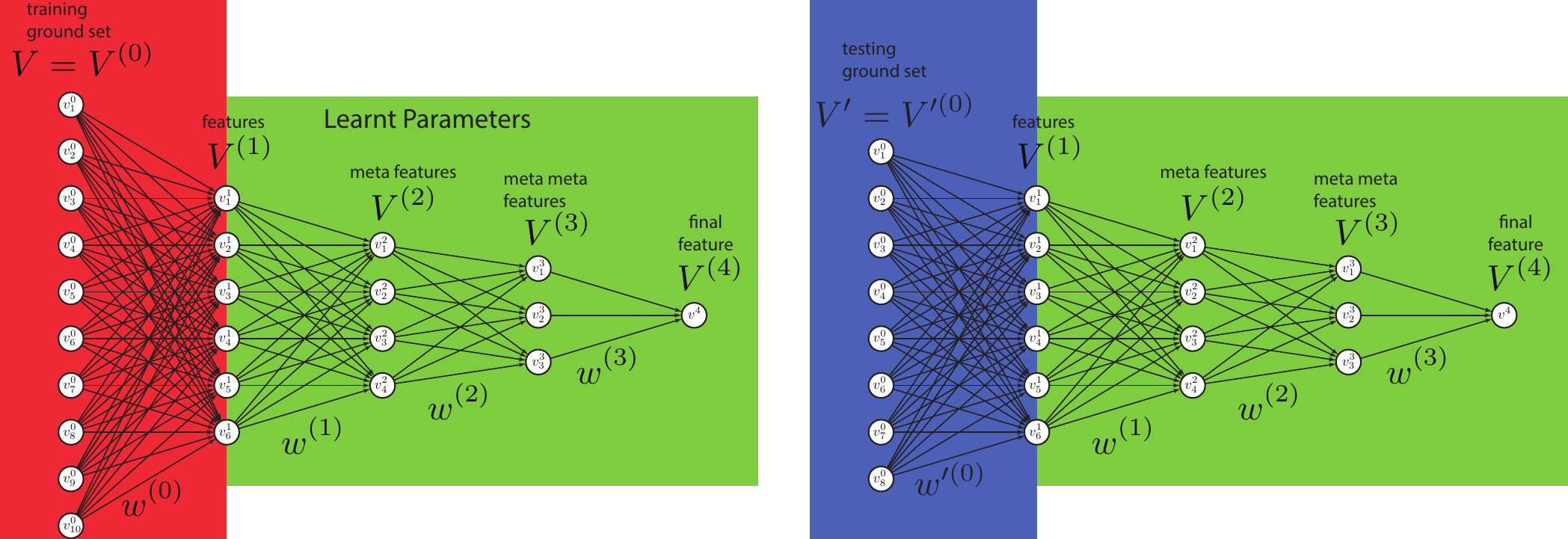}}
\caption{
Left: Training DSF where
the first layer weights $w^{(0)}$ act to embed 
input space and ground set $V = V^{(0)}$
into feature space $V^{(1)}$. The learn
weights are highlighted (green) while the
mapping parameters (red) are a embedding
transformation.
Right: Using the learnt 
parameters (green) $\set{ w^{(i)} }_{i> 0}$, 
we instantiate a DSF on training
objects using an mapping (blue)
from a distinct ground set $V' = V'^{(0)}$
into the same feature space $V^{(1)}$.
}
\label{fig:dsf_train_test}
\end{figure}


In the training process in machine learning, one trains with a
training set and then evaluates or tests on a distinct set having no
overlap with the training set. When training submodular functions,
this means that the training set might consist of multiple ground
sets, and the test set might consist of ground sets that were not seen
during training.  A data set might consist of
$\mathcal D = \set{ (V_i, S_i, y_i) }_i$ where $V_i$ is a ground set,
$S_i \subseteq V_i$ and, when available, $y_i = f_i(S_i)$ is an
evaluation of $S_i$ by a ground-set-specific submodular function
$f_i$. Hence, there may be no instance where two ground sets are the
same, so $V_i \neq V_j$ for $i\neq j$, nor might there be ground set
commonality between training and test data sets.  The reason this
occurs can be explained using a document summarization
example~\cite{hui2012-submodular-shells-summarization}. A training set
consists of pairs, each of which is pile of documents (comprised of a
set of sentences) and a subset of those sentences corresponding to a
summary. Multiple training samples consists of different piles of
documents and their corresponding summaries, and then a test set
consists of a different pile of documents and summaries thereof.  In
this section, we discuss how to addresses this problem for DSFs via a
strategy that
generalizes~\cite{hui2012-submodular-shells-summarization,sebastian2014-submod-image-sum}.

Let $V$ be a training set where each $v \in V$ is a data object.  Any
particular element $v \in V$ may be represented by a vector of
non-negative weights
$(w^{(0)}_1(v),w^{(0)}_2(v), \dots, w^{(0)}_{|U|} )$.  Each object
$v \in V$ is hence embedded in non-negative $|U|$-dimensional space
corresponding to low-level features $U$ for the object. For example,
if $v$ is a sentence, $w^{(0)}_u(v)$ might counts the number of times
an n-gram $u$ appears in sentence $v$.  Alternatively, $w^{(0)}_u(v)$
might be automatically obtained via representation learning in a
DNN-based auto-encoder, or there can be a mix of features obtained via
representation learning and hand-crafting, using any of the
feature-engineering methods discussed in
Section~\ref{sec:feat-based-funct}. For each feature, we can define a
modular function $m_u(A) = \sum_{a \in A} w^{(0)}_u(a)$ that
measures feature $u$'s weight for any set $A \subseteq V$.  The entire
training set, therefore can be seen a matrix $w^{(0)}$ to be used as
the first layer in DSF (e.g., $w^{(0)}$ in
Figure~\ref{fig:dsf_train_test} left (red)) that is fixed during the
training of subsequent layers (Figure~\ref{fig:dsf_train_test} left
(green)). As long as $w^{(0)}$ is non-negative, submodularity is
preserved and if $w^{(0)}$ is constant, it allows all later layers
(i.e., $w^{(2)}, w^{(3)}, \dots$) to be learnt generically over any
heterogeneous set of objects that can be represented in the same
feature space, including multimodal data objects (e.g., consisting of
images, videos, and text sentences). Any training process remains
ignorant that this is happening since it sees the data only post
feature representation. In fact, one can view this, in light of
Theorem~\ref{thm:submod_antitone_super_on_vector_poly}, as a fixed
layer consisting of an SCMM that embeds data objects into
feature space corresponding to the components of the SCMM.

Once training has occurred, and if there is an analogous process to
transform distinct (and possibly different types) of test data into
the same feature space, it is possible to use the learnt DSF even for
a different ground set. In Figure~\ref{fig:dsf_train_test} right
(red), we have a different transformation $w'^{(0)}$ into the same
feature space $V^{(1)}$ which can use the DSF (green) learnt during
training. This process analogous to the ``shells''
of~\cite{hui2012-submodular-shells-summarization}. In that case,
mixtures were learnt over fixed components, some of which were graph
based (and hence required $O(n^2)$ calculation for element-pair
similarity scores). Via featurization in the first layer of a DSF,
however, we may learn a DSF over a training set, preserving
submodularity, avoid any $O(n^2)$ cost, and test on any new data
represented in the same feature space. Alternatively, one could
combine the shells approach and $w'^{(0)}$ into a vector of
polymatroid functions and then apply
Theorem~\ref{thm:submod_antitone_super_on_vector_poly}.

\subsection{Deep Supermodular Functions and Deep Differences}
\label{sec:learn-diff-dsfs}

All of the results in this paper assume that the hidden units in a DSF
are concave. If we replace these concave functions in
Equation~\eqref{eq:gen_deep_submodular_recursive} then we get a class
we could call {\em Deep Supermodular Functions} (DSUFs). The results
in this paper, hence, generalize to show that DSUFs correspond to a
larger class than just sums of convex functions composed with
non-negative modular functions.

In~\cite{narasimhan2005-subsup,rkiyeruai2012} it was shown that any
set function $h : 2^V \to \mathbb R$ can be represented as a
difference between two submodular functions. If we take
$f_1,f_2 \in \text{DSF}$ then the class of functions
$\text{DDSF} = \set{ h : h = f_1 - f_2, f_1,f_2 \in \text{DSF} }$ can
be seen as a class of deep differences of submodular
functions. Considering the class
$\text{DSSUF} = \set{ h : h = f + g, f \in \text{DSF}, g \in
  \text{DSUF} }$ can be seen as a class of deep submodular plus
supermodular functions.  Given that DSFs do not comprise all
submodular functions, it is unlikely that $\text{DSSUF}$s comprise all
set functions. However, these can be useful classes of functions to
learn over using, say, the deep learning methods mentioned in
Sections~\ref{sec:learn-deep-subm}. A key advantage of learning over
this family is that the framework never looses the decomposition into
two submodular functions or a submodular and supermodular
function. For example, after learning, we can utilize submodular
level-set constrained submodular optimization of the kind developed
in~\cite{rishabh2013-submodular-constraints} for optimization. 
Learning under such a decomposition, moreover, might
reveal substitutive (via $f$) and complementary (via $g$)
properties of the data.

It may also be useful to define a class of deep
``cooperative-competitive'' energy functions for use in a
probabilistic model. For example, one can define probability
distributions $p$ over binary vectors with
$p(x) = \frac{1}{Z}\exp( f_{w_1}(x) - f_{w_2}(x) )$ where $f_{w_1}$
and $f_{w_2}$ are both deep submodular, or
$p(x) = \frac{1}{Z}\exp( f_{w_1}(x) + g_{w_2}(x))$ where $f_{w_1}$ is
deep submodular and $g_{w_2}$ is deep supermodular. If $f_{w_1}$ and
$g_{w_2}$ have decomposition properties with respect to a graph, then
these could be called deep cooperative-competitive graphical models.




\subsection{Deep Multivariate Submodular Functions}
\label{sec:deep-k-submodular}

Submodular functions have been generalized in a variety of ways to
domains other than just subsets of a finite set $V$ (i.e., binary
vectors). In Section~\ref{sec:antit-maps-superd}, we discussed the
negativity of the off-diagonal Hessian as a way of defining submodular
functions on lattices. Other ways to generalize submodularity
considers discrete generalizations of properties such as midpoint
convexity over integer lattices~\cite{murota2003discrete}.

In this section, we consider certain submodular generalizations to
multi-argument functions. For example, a set function $f(A, B)$ with
two arguments $A \subseteq V$ and $B \subseteq V$ is a biset
function. If the domain is of the form
$2^{2V} \triangleq \lbrace (A, B) : A \subseteq V,\ B \subseteq V
\rbrace$,
we may define the class of functions known as simple bisubmodular:
\begin{definition}[Simple Bisubmodularity~\cite{singh2012-bisubmod}]
  $f : 2^{2V} \rightarrow \mathbb{R}$ is simple bisubmodular
  iff for each $(A,B) \in 2^{2V}$,
  $(A',B') \in 2^{2V}$ with $A \subseteq A'$, $B \subseteq B'$
  we have for $s \notin A'$ and $s \notin B'$:
  \begin{eqnarray*}
  &f(A + s, B) -  f(A, B) \geq f(A' + s, B') -  f(A', B'),&\\
  &f(A, B+s) -  f(A, B) \geq f(A', B' + s) -  f(A', B').&
  \end{eqnarray*}
\end{definition}
\noindent An equivalent way to define simple bisubmodularity is as follows.
\begin{proposition}
 \label{ref:prop_equiv_simp_bisubmod}
  The function $f : 2^{2V} \rightarrow \mathbb{R}$ is simple bisubmodular
  whenever 
  $\forall (A,B),(A',B') \in 2^{2V}$,
  \begin{align}
  f(A,B) + f(A', B') \geq
  f(A \cup A', B \cup B') + f(A \cap A', B \cap B')
  \end{align}%
\end{proposition}
\noindent If the domain is of the form 
$3^V \triangleq \lbrace (A ,B) : A \subseteq V,\ B \subseteq V,\ A \cap B =
\emptyset \rbrace$, then we can define directed bisubmodularity as follows:
\begin{definition}[Directed Bisubmodularity~\cite{bisubmod}]
  Biset function $f: 3^V \to \mathbb{R}$ is directed bisubmodular whenever
\begin{align}
    f(A, B) + f(A', B') \geq  f(A \cap A', B \cap B') +
    f((A \cup A') \setminus (B \cup B'), (B \cup B') \setminus (A \cup A')).
\end{align}
\end{definition}
Directed bisubodularity functions have been generalized to what is
known as $k$-submodular functions in
\cite{kolmogorov2011submodularity,DBLP:journals/corr/HuberK13a}.  More recently, simple
bisubmodularity~\cite{singh2012-bisubmod} has been generalized to
multivariate submodular functions~\cite{santiagomultivariatesub2016}.
A multivariate submodular (or what we will call a $k$-multi-submodular) function
$f: (2^V)^k \to \mathbb R$ is defined as a function
such that 
for all $(X_1,X_2, \dots, X_k),(Y_1,Y_2, \dots, Y_k) \in (2^V)^k$, we have that:
\begin{align}
f(X_1,X_2, \dots, X_k)
+ f(Y_1,Y_2, \dots, Y_k)
\geq
f(X_1 \cup Y_1, X_2 \cup X_2, \dots, X_k \cup Y_k)
+ 
f(X_1 \cap Y_1, X_2 \cap X_2, \dots, X_k \cap Y_k)
\end{align}
These are not the same as k-submodular functions~\cite{DBLP:journals/corr/HuberK13a}
but for $k=1$ we obtain standard submodular functions
and for $k=2$ we obtain simple bisubmodular functions.

A DSF with $k' > k$ layers can be used to instantiate a
$k$-multi-submodular function. Consider a layered-DSF with $k'$ layers
corresponding to sets $V^{(0)}$, $V^{(1)}$, \dots, $V^{(k')}$.  Choose
a size $k$ subset of these layers, say
$\sigma_1, \sigma_2, \dots, \sigma_k$ where $\sigma_j \in [0,k'-1]$ for all $j$,
$\sigma_1 = 0$, and all of the $\sigma_j$'s are distinct, w.l.o.g.,
$0 = \sigma_1 < \sigma_2 < \dots < \sigma_k \leq k'-1$. Given an
$f \in \text{DSF}_{k'}$, we ordinarily obtain a valuation $f(A)$ using
a subset $A \subseteq V^{(0)}$ of the ground set.  Now, consider
$f: (2^V)^k \to \mathbb R$ where $A_1 \subseteq V^{(\sigma_1)}$,
$A_2 \subseteq V^{(\sigma_2)}$, \dots, $A_k \subseteq V^{(\sigma_k)}$
and the value of $f(A_1,A_2, \dots, A_k)$ is obtained as:
\begingroup\makeatletter\def\f@size{10}\check@mathfonts
\def\maketag@@@#1{\hbox{\m@th\small\normalfont#1}}%
\begin{align}
\label{eq:deep_submodular_polymatroid_b}
\phi_{v^{k'}}&
\Biggl(
\sum_{v^{k'-1} \in \bar V^{(k'-1)}}
w^{(k')}_{v^{k'}}(v^{k'-1})
\phi_{v^{k'-1}}
\biggl(
\dots
\!\!\!\!
\sum_{v^1 \in \bar V^{(1)}}
w_{v^2}^{(2)}(v^1)
\phi_{v^1}
\Bigl(
\sum_{v^0 \in \bar V^{(0)}}
w_{v^1}^{(1)}(v^0)
\phi_{v^0}
\bigl(
\sum_{a \in \bar V^{(0)}} w_{v^1}^{(0)}(a)
\bigr)
\Bigr)
\biggr)
\Biggr) \\
&+ m_\pm^{(1)}(A_1) 
+m_\pm^{(2)}(A_2) 
+ \dots
+m_\pm^{(k)}(A_k) 
\end{align}\endgroup
where $\bar V^{(i)} = V^{(i)} \cap A_{\sigma^{-1}_i}$ whenever
$\exists j \in [0,k'-1] : i = \sigma_j$ and otherwise $\bar V^{(i)} = V^{(i)}$, and
where $m_\pm^{(j)} : V^{(\sigma_j)} \to \mathbb R$, for each $j$, is
an arbitrary modular function.  In other words, $A_j$ acts as a set of
binary triggers to activate a set of units at layer $j$ in the DSF.
If we hold all but layer $j$ fixed, then $A_j$ can be seen as the set
of units to provide the values for the vector $\mathbf b_{A_j}$ in
Corollary~\ref{thm:submod_antitone_super} and as a result, we get as a
result that the function is submodular in $A_j$.
$k$-multi-submodularity then follows from a generalization of
Proposition~\ref{ref:prop_equiv_simp_bisubmod} to
$k$-multi-submodularity.

Deep $k$-multi-submodular functions should be useful in a number of
applications, for example representing information jointly in a set of
features and data items (and could be useful for simultaneous
feature/data subset selection).


%


\subsection{Simultaneously Learning Hash and Submodular Functions}
\label{sec:hash-functions}

One of the difficulties in training DSFs is obtaining a sufficient
amount of training data. It would be useful therefore to have an
strategy to easily and cheaply obtain as much training data as
desired.  In the spirit of the empirical success of DNNs, this section
suggests one strategy for doing this.

The goal is to learn a map from a vector $x \in \mathbb R^d$ to a
$b$-bit vector via a function $h_\theta: \mathbb R^d \to \set{0,1}^b$,
anywhere $h_\theta$ is parameterized by $\theta$. The reason for doing
this is to take data objects (e.g., images, documents, music files,
etc.) that are represented in the input space $R^d$ and map them to
binary space $\set{-1,1}^b$ where $b < d$ and, moreover, since the
space is binary, operations such nearest neighbor search are
faster.  
There are existing approaches that can learn this mapping
automatically, sometimes using neural networks (e.g.,
\cite{grauman2013learning}). Often,
$h_\theta: \mathbb R^d \to \set{-1,1}^b$ rather than
$h_\theta: \mathbb R^d \to \set{0,1}^b$, but this should not be of any
consequence.

This section describes a strategy for learning hash functions that
utilizes DSFs, the \lovasz{} extension, and the submodular Hamming
metric~\cite{gillenwater2015nips}. Let $f: 2^V \to \mathbb R$ be a
submodular function and let $\lex f$ be its \lovasz{} extension.
Also, let $d_f(A,B) = f(A \triangle B)$ be the submodular Hamming
metric between $A$ and $B$ parameterized by submodular function $f$.
We are given a large (and possibly unlabeled) data set
$\mathcal D = \set{ x_i }_{i \in D}$ and a corresponding distance
function between data pairs ($d(x_i,x_j)$ is the distance between item
$x_i \in \mathbb R^d$ and $x_j \in \mathbb R^d$). The goal is to
produce a mapping $h_\theta: \mathbb R^d \to \set{0,1}^b$ so that
distances in the ambient space $d(x_i,x_j)$ are preserved in the
binary space. One approach adjusts $h_\theta$ to ensure
that
$d(x_i,x_j) = \sum_{\ell=1}^b \mathbf 1_{ h_\theta(x_i)(\ell) \neq
  h_\theta(x_j)(\ell)}$.
That is, we adjust $h_\theta(x_i)$ so that the Hamming distance
preserves the distances in the ambient space.

In general, this problem is made more difficult by the rigidity of the
Hamming distance. In order to relax this constraint, we can use a
submodular Hamming metric parameterized by a DSF $f_w$ (which itself
is parameterized by $w$).  Hence, the hashing problem can be seen as
finding $\theta$ and $w$ so that the following is true as much as
possible.
\begin{align}
d(x_i,x_j) = 
d_{f_w}(h_\theta(x_i),h_\theta(x_j))
\end{align}
The function $h_\theta$ maps to binary vectors, and $d_{f_w}$ is a
function on two sets. This makes it difficult to pass derivatives
through these functions in a back-propagation style learning
algorithm. To address this issue, we can further relax this problem in
the following way:
\begin{itemize}

\item Given $A,B \subseteq V$, the Hamming distance is
  $|A \triangle B|$ and we can represent this as
  $(\mathbf 1_A \otimes (\mathbf 1_V - \mathbf 1_B) + \mathbf 1_B \otimes
  (\mathbf 1_V - \mathbf 1_A))(V)$
where $\otimes : \mathbb R^n \times \mathbb R^n \to \mathbb R^n$ is the vector element multiplication operator
(i.e., $[x \otimes y](j) = x(j)y(j)$).
In other words, we define a vector $z_{A \triangle B} \in \set{0,1}^V$
with
\begin{align}
z_{A \triangle B} &= \mathbf 1_A \otimes (\mathbf 1_V - \mathbf 1_B) + \mathbf 1_B \otimes
  (\mathbf 1_V - \mathbf 1_A) \\
&= \mathbf 1_A + \mathbf 1_B - 2 \mathbf 1_A \otimes \mathbf 1_B
\end{align}
and $|A \triangle B| = z_{A \triangle B}(V) = \sum_{i \in V} z_{A \triangle B}(i)$.
Hence, the submodular hamming metric is
$f(A \triangle B) = \lex f(z_{A \triangle B})$, which 
holds since the \lovasz{} extension is tight at the vertices of the
hypercube.

\item For two arbitrary vectors $z_1,z_2 \in [0,1]^V$, we can define
a relaxed form of metric
as follows: $d(z_1,z_2) = \lex f(z_1 + z_2 - 2z_1 \otimes z_2)$,
and for a DSF, this can be expressed as
$d_{\lex f_w}(z_1,z_2) = \lex f_w(z_1 + z_2 - 2z_1 \otimes z_2)$.

\item Let us suppose that $\tilde h_\theta : \mathbf R^d \to [0,1]^b$
  is a mapping from real vectors to vectors in the hypercube (e.g.,
  $\tilde h_\theta$ might be expressed with a deep model with a final
  layer of $b$ sigmoid units at the output to ensure that each output is between
  zero and one). Then we can construct a distortion between $x_i$ and
  $x_j$ via
\begin{align}
  d_{w,\theta}(x_i,x_j) \triangleq
  d_{\lex f_w}(\tilde h_\theta(x_i), \tilde h_\theta(x_j))
  = \lex f_w( \tilde h_\theta(x_i) +\tilde h_\theta(x_j)
   - 2 \tilde h_\theta(x_i) \otimes \tilde h_\theta(x_j))
\end{align}
Hence, $d_{w,\theta}$ is a parametric family of distortion functions
that uses two maps, one via the DNN $\tilde h_\theta$ and another via
the DSF $f_w$ using the Lovasz extension $\lex f_w$.

\item Assuming the original unlabeled data set $\mathcal D$ is large,
  and the distance function in the ambient space is accurate, it may
  be possible to learn both $w$ and $\theta$ by forming an objective
  function to minimize:
\begin{align}
J(w,\theta)
= \sum_{i,j \in D} 
\|
d(x_i,x_j)
- d_{w,\theta}(x_i,x_j) \|.
\end{align}
Learning ($\min_{w,\theta}J(w,\theta)$)
can utilize stochastic gradient steps and the
entire arsenal of DNN training methods.


\end{itemize}

The approach learns both the mapping function $\tilde h_\theta$ and
the submodular function $f_w$ simultaneously in a way that preserves
the original distances. It may therefore be that $\tilde h_\theta$ can
be used as a feature transformation (i.e. a way to map data objects
$x$ into feature space via $\tilde h_\theta$), and at the same time we
obtain a submodular function $f_w$ over those features that, perhaps,
can useful for summarization, all without needing labeled training
data as in Section~\ref{sec:learn-deep-subm}.

\section{Conclusions and Future Work}
\label{sec:concl-future-work}

In this paper, we have provided a full characterization of our
newly-proposed class of submodular functions, DSFs. We have introduced
the antitone gradient as a way of establishing subclasses of
submodular functions. We have shown that DSFs constitute a strictly
larger family than the family of submodular functions obtained by
additively combining concave composed with modular functions (SCMMs).
We have also shown that DSFs do not comprise all submodular functions.
This was all done in the special context of matroid rank functions,
and also in a more general context.

As mentioned at various points within the paper, there are several
interesting open problems associated with DSFs.  An immediate task is
to further develop practical strategies for successfully empirically
learning DSFs, as was initiated in~\cite{bilmes2016dsf_nips}. A second
task is to establish generalization bounds for learning DSFs in an ERM
framework. A third task asks if there is a finite set of ``boot''
submodular functions that, when cascaded into a DSF as in
Theorem~\ref{thm:submod_antitone_super_on_vector_poly}, lead to a
family that comprises all polymatroid functions. And lastly, it remains
to compare the DSF family with the family of all entropy
functions~\cite{zhang1997non}.

\section{Acknowledgments}

Thanks to Brian Dolhansky for helping with building an initial
implementation of learning DSFs that was used in
\cite{bilmes2016dsf_nips}.  Thanks also to Reza Eghbali and Kai Wei
for useful discussions, and to Jan Vondrak for suggesting the use of
surplus and deficit as an analysis strategy.  This material is based
upon work supported by the National Science Foundation under Grant
No. IIS-1162606, the National Institutes of Health under award
R01GM103544, and by a Google, a Microsoft, a Facebook, and an Intel
research award.  Thanks also to the Simons Institute for the Theory of
Computing, Foundations of Machine Learning Program.  This work was
supported in part by TerraSwarm, one of six centers of STARnet, a
Semiconductor Research Corporation program sponsored by MARCO and
DARPA.

\addcontentsline{toc}{section}{References}
\bibliography{dsf_nips2016}
\bibliographystyle{plain}

\appendix

\section{More General Conditions on Two-Layer Functions: Proofs}
\label{sec:more-gener-cond-1}



\begin{proof}[Proof of Theorem~\ref{theo:concaveOverSumOfMin}]
We begin with the ``only if'' part. In the proof, we always assume the ground set $V = \set{a,b,c,d,e,f}$.
\begin{definition}

  Consider the bijection $p:V\rightarrow V$. Let
  $A_p = \set{p(v)|v\in A}$.  Notationally, we may write a given $p$
  as $(v_1,v_2,\ldots,v_k)\rightarrow (u_1,u_2,\ldots,u_k)$ where
  $u_i,v_i \in V$ with $u_i = p(v_i)$. Let $P_{A}$ be the set of all
  one-to-one maps that are an identity for $v \in V \setminus A$, that
  is $p(v)=v$ for all $v\in V\setminus A$.  Corresponding to
  Theorem~\ref{theo:concaveOverSumOfMin}, in the below, assume
  $V=\set{a,b,c,d,e,f}$.  We next define a number of operators that
  allow us to study the partial permutation symmetry of a set
  function.
\end{definition}
\begin{definition}
\label{def:symmetryOperation}
For any submodular function $h$, let:
\begin{itemize}
\item $E_{B}$ be an operator such that $E_B h(A) = \frac{1}{|P_B|}(\sum_{p\in P_B} h(A_p))$; 
\item $E'$ be an operator such that $E'h(A) = \frac{1}{2} [h(A)+h(A_{(a,b,c,d,e,f)\rightarrow(d,e,f,a,b,c)}]$;
\item and  $E$ be an operator such that  $Eh(A) = E'E_{\set{d,e,f}}E_{\set{a,b,c}}h(A)$.
\end{itemize}
\end{definition}
Immediately, we have the following lemma.
\begin{lemma}
\label{lemma:persevedProjection}
$E g(A) = g(A)$ for all $A\subseteq V$. Also, $E$ is a linear
operation, that is $E(h_1+h_2) = Eh_1 + Eh_2$. Lastly, if $h$ is an SCMM, $Eh$
is also an SCMM.
\end{lemma}
\begin{lemma}
\label{lemma:Ehisdeterminedbyn1n2}
For any $A,B\subseteq V$, if $[|A\cap\{a,b,c\}|= |B\cap\{a,b,c\}| \text{ AND } |A\cap\{d,e,f\}|= |B\cap\{d,e,f\}|]$ OR  $[|A\cap\{a,b,c\}|= |B\cap\{d,e,f\}|\text{ AND }|A\cap\{d,e,f\}|= |B\cap\{a,b,c\}|]$, then $Eh(A) = Eh(B)$.
\end{lemma}
This means that $Eh(A)$ is fully determined by 
the unordered pair $\set{ |A\cap \set{a,b,c}|, |A\cap \set{d,e,f}| }$.

\begin{definition}
\label{def:unorderedPair}
For any $h : 2^V \to \mathbb R$, define $Eh(n_1,n_2) = Eh (A)$, where
$n_1 = |A\cap \set{a,b,c}|$, $n_2 = |A\cap \set{d,e,f}|$, and
$0\leq n_2\leq n_1\leq 3$.
\end{definition}
Since this section shows
the ``only if'' part of Theorem~\ref{theo:concaveOverSumOfMin},
we have $g(A)$ is an SCMM, thus 
by Lemma~\ref{lemma:sums_mod_truncs},
$g(A) = \sum_i  \min(m_i(A),\beta_i) + m_{\pm}(A)$, 
where $m_i\geq 0$ is non-negative modular and $\beta_i>0$. Immediately, we have 
\begin{align}
Eg(A) &= \sum_i E\min(m_i(A),\beta_i) + Em_{\pm}(A)\\
g(A) &= \sum_i E g_i(A) + Em_{\pm}(A)
\end{align}
according to lemma~\ref{lemma:persevedProjection}, where
$g_i(A) = \min(m_i(A),\beta_i)$. Moreover, we assume
$m_i(V)>\beta_i>0$ for each i; otherwise $g_i$ is modular and can be
merged into the final modular term.  Furthermore, we assume that
$m_i(v)\leq \beta_i$ for all $v\in V$ and all $i$. If
$m_i(v)>\beta_i$, it means that $\min(m_i(A),\beta_i) = \beta_i$
whenever $v$ is selected in $A$. In such case, we can let
$m_i(v) = \beta_i$ which have the same function value for all
$A$. Therefore we have
\begin{lemma}
\label{lemma:finalGainIsLessThanMv}
$g_i(v|V\setminus\set{v})<g_i(v) = m_i(v)$ for all $i$ and $v$ s.t. $m_i(v)>0$.
In other words, $I_{g_i}(v; V \setminus v) > 0$ for all $i$ with $g_i(v) > 0$.
\end{lemma}
\begin{proof}
This follows since $m_i(V)$ passes the linear part of $g_i$ but $m_i(v)$ does not.
\end{proof}
\begin{lemma}
\label{lemma:finalGainIsGainOfbc}
 $g_i(a|\set{b,c}) = g_i(a|\set{b,c,d,e,f})$ for all $i$.
\end{lemma}
\begin{proof}
  We have that $0 \leq I_g(a; A) \leq I_g(a; B)$ for all $A \subseteq B$.
  Hence $I_g(a; \set{b,c,d,e,f}) = 0$ implies
  $I_g(a; \set{b,c}) = 0$. Hence,
  for all $i$, 
  $I_{g_i}(a; \set{b,c,d,e,f}) = I_{g_i}(a; \set{b,c}) = 0$,
  implying 
 $g_i(a|\set{b,c}) = g_i(a|\set{b,c,d,e,f})$ for all $i$.
\end{proof}

\begin{definition}
We define the following functions:
\begin{itemize}
\item $f_0(A) = |A|$;
\item $f_1(A) =\min(|A\cap\{a,b,c,d,e,f\}|,1)$;
\item $f_2(A) = \min(|A\cap\{a,b,c,d,e,f\}|,2)$;
\item $f_3(A)= \min(|A\cap\{a,b,c\}|,1)+\min(|A\cap\{d,e,f\}|,1)$;
\item $f_4(A) =\min(|A\cap\{a,b,c\}|,2)+\min(|A\cap\{d,e,f\}|,2)$;
\item and $f_5(A) = E \modularminA{1,1,0,0.5,0.5,0.5}$, 
where $(x_a,x_b,x_c,x_d,x_e,x_f)^T$ is a modular function with elements 
$x_a$, $x_b$, $x_c$, $x_d$, $x_e$, $x_f$.
\end{itemize}
\end{definition}
Immediately, we notice that $Ef_i = f_i$ for all $i$.
	\begin{lemma}
		\label{lemma:onlyabc2f0f3f4}
		For a normalized monotonically non-decreasing submodular $h$, if $h(\set{d,e,f}) = 0$, then $Eh$ is a conical combination of   $f_0 ,f_3,f_4$ 
	\end{lemma}
		\begin{proof}
			Let $x = \frac{1}{3}(h(a)+h(b)+h(c))$ and $y = \frac{1}{3}(h(\set{a,b})+h(\set{b,c})+h(\set{a,c}))$ and $z=h(\set{a,b,c})$. Then $Eh$ can actually be written as $\frac{1}{2}[(z-y)f_0 + (2x-y ) f_3+ (2y-z-x) f_4] $ where $z-y, 2x-y, 2y-z-x\geq 0$ according to submodularity.
		\end{proof}
\begin{lemma}
We say a function is fully curved if $f(v | V \setminus v)$ for some $v$.
\label{lemma:fullycurveGiisf0f3f4}
For i such that $g_i$ is not fully curved, $Eg_i$ is a conical
combination of $f_0,f_3,f_4$.
\end{lemma}
\begin{proof}
Without lose of generality, we assume
  $g_i(a|\set{b,c,d,e,f})>0$. Immediately we have $m_i(a)>0$ and
  $m_i(\set{b,c,d,e,f})<\beta_i$.  According to
  lemma~\ref{lemma:finalGainIsLessThanMv} and
  lemma~\ref{lemma:finalGainIsGainOfbc}, we have
  $I(a;\set{b,c}) = g_i(a) - g_i(a|\{b,c\}) = g_i(a) -
  g_i(a|V\setminus\set{a})>0$. Thus
  $m_i(\set{a,b,c})\geq\beta_i$. Therefore
  $0 = g_i(a|\set{b,c}) - g_i(a|\set{b,c,d,e,f})=g_i(\set{a,b,c})-
  g_i(\set{b,c}) -g_i(\set{a,b,c,d,e,f})+ g_i(\set{b,c,d,e,f}) =
  \beta_i - m_i(\set{b,c}) -\beta_i + m_i(\set{b,c,d,e,f}) =
  m_i(\set{d,e,f})$. So we have that $m_i(\set{d,e,f}) =0 $ and $g_i$
  only involves $a,b,c$.  According to
  lemma~\ref{lemma:onlyabc2f0f3f4}, $Eg_i$ is a conical combination of
  $f_0,f_3$ and $f_4$.
\end{proof}

\begin{lemma}
  $m\pm$ is not necessary, that is if we find one SCMM expansion of
  $g$, we can also find another SCMM expansion with $m\pm = 0$.
\end{lemma}
\begin{proof}
  For some $i$, $Eg_i$ is fully curved and for the other $i$,
  $Eg_i = g'_i +m'_i$ where $g'_i$ is a fully curved SCMM and $m'_i$
  is modular according to lemma~\ref{lemma:fullycurveGiisf0f3f4}. So
  we can group all $m'_i$ and $m_\pm$ together.  If a fully curved
  submodular function is another fully curved submodular function plus
  modular, the only possibility is that the modular term equals 0. So
  the final modular vanishes.
\end{proof}
      
So actually, we can ignore the final modular functions at the
expansion of $g$.
$g(A) = \sum_i \min(m_i(A),\beta_i) = \sum_i E\min(m_i(A),\beta_i) =
\sum_i E g_i(A) $, where all term are non-negative and fully curved
now.

Consider the quality $g_i(a|\set{b,c})$, it is non-negative for each
$i$ and 0 for $g$. So for each $i$, we have $g_i(a|\set{b,c})=0 $. In
fact, $g_i$ is fully curved on $\set{a,b,c}$ and $\set{d,e,f}$

\begin{lemma}
  For a normalized monotonically non-decreasing submodular $h$, if $h$
  is fully curved on $\set{a,b,c}$ and $\set{d,e,f}$ , then $Eh$ is
  determined by 5 values, $Eh(1,0)$, $Eh(2,0)$, $Eh(1,1)$, $Eh(2,1)$
  and $Eh(2,2)$.
\end{lemma}

\begin{proof}
  According to lemma~\ref{lemma:Ehisdeterminedbyn1n2} and
  definition~\ref{def:unorderedPair}, $Eh(A)$ is determined by
  $Eh(1,0)$, $Eh(1,1)$, $Eh(2,0)$, $Eh(2,1)$, $Eh(2,2)$, $Eh(3,2)$ and
  $Eh(3,3)$. But $Eh(n_1,n_2) = Eh(\min(n_1,2),\min(n_2,2))$ according
  to the saturate properties. So $Eh(1,0)$, $Eh(2,0)$, $Eh(1,1)$,
  $Eh(2,1)$ and $Eh(2,2)$ are the only free variables remained.
\end{proof}

\begin{lemma}
  $Eh(n_1,n_2)=Eh_1(n_1,n_2)+Eh_2(n_1,n_2)$ if $h = h_1 +h_2$.
\end{lemma}
So in fact $Ef$ is a 5-dimensional-vector.  Here we calculate the
5-dimensional-vector for $f_1,f_2,f_3,f_4,f_5$, see
table~\ref{table:5-vectorforallfi}.

\begin{table}	
\begin{tabular}{cccccc}
& $Ef(1,0)$ & $Ef(2,0)$ & $Ef(1,1)$ & $Ef(2,1)$ & $Ef(2,2)$ \\ \hline
$f_1$ & 1 & 1 & 1 & 1 & 1 \\ \hline
$f_2$ & 1 & 2 & 2 & 2 & 2 \\ \hline
$f_3$ & 1 & 1 & 2 & 2 & 2 \\ \hline
$f_4$ & 1 & 2 & 2 & 3 & 4 \\ \hline
$f_5$ & $\frac{7}{12}$ & 1 & $\frac{5}{6}$ & 1 & 1 \\ \hline 
$E\min(|A\cap\set{a,b,d,e}|,1)$ & $\frac{2}{3}$ & 1 & $\frac{8}{9}$& 1 & 1 \\ \hline
$E\min(|A\cap\set{a,b,c,d,e}|,1)$ & $\frac{5}{6}$ & 1 & 1 & 1& 1 \\ \hline
\end{tabular}
\caption{Function values}
\label{table:5-vectorforallfi}
\end{table}

\begin{lemma}
For all $i$, $Eg_i$ is a conical combination of   $f_1$, $f_2$, $f_3$, $f_4$, $f_5$.
\end{lemma}

\begin{proof}
  For $i$ s.t. $g_i(\set{a,b,c})=0$ or $g_i(\set{d,e,f}) = 0$, $Eg_i$
  is a conical combination of $f_0 ,f_3,f_4$ according to
  lemma~\ref{lemma:onlyabc2f0f3f4}. Moreover, $f_0$ is not necessary
  since $g_i$ is fully curved.

  For other $i$, if $m_i(a)+m_i(b)<\beta_i$, then $m_i(c) = 0$;
  otherwise $g_i(c|\set{a,b})>0$. But in this case
  $0 = g_i(a|\{b,c\}) = m_i(a)$ and $0 = g_i(b|\{a,c\})= m_i(b)$ which
  contradicts with $g_i(\set{a,b,c})>0$. So
  $m_i(\set{a,b})\geq \beta_i$.  Similarly, we have
  $m_i(\set{b,c}),m_i(\set{c,a}),
  m_i(\set{d,e}),m_i(\set{e,f}),m_i(\set{d,f})\geq \beta_i$.

  So $Eg_i(2,0) = Eg_i(2,1) = Eg_i(2,2) = \beta_i$. And the undecided
  parameters are $Eg_i(1,0)$ and $Eg_i(\set{1,1})$.

  It is easy to check that
  $E g_i = [Eg_i(1,1) + 2Eg_i(1,0)- 2\beta_i]f_1+\frac{1}{2}[5
  Eg_i(1,1) - 2Eg_i(1,0)- 3\beta_i]f_2+ [6\beta_i-6
  Eg_i(1,1)]f_5$.

Next we will show that all coefficients are non-negative.

\begin{lemma}
  Given $g_i(A) = \min(m_i(A),\beta_i)$, if
  $m_i(a)+m_i(b),
  m_i(b)+m_i(c),m_i(c)+m_i(a),m_i(d)+m_i(e),m_i(e)+m_i(f),m_i(f)+m_i(d)\geq
  \beta_i$, we have $Eg_i(1,1) + 2Eg_i(1,0)- 2\beta_i\geq 0$,
  $5 Eg_i(1,1) - 2Eg_i(1,0)- 3\beta_i\geq 0$, $Eg_i(1,1)\leq \beta_i$
\end{lemma}

\begin{proof}

  Let $x_i$ be the weight of each elements. Without lose of
  generality, we assume that $\beta_i\geq x_1\geq x_b\geq x_c\geq 0$,
  $\beta_i\geq x_d\geq x_e\geq x_f\geq 0$ and $x_c\geq x_f$. So
  $x_a, x_b, x_d, x_e\geq \frac{1}{2}\beta_i$.
      
  $Eg_i(1,0)= \frac{1}{6}\sum_i x_i\geq \frac{2}{3}\beta_i$ and $Eg_i(1,1) = \frac{1}{9}[\sum_{v\in\set{a,b,c}}\sum_{u\in\set{d,e,f}} g_i(\set{v,u})] = \frac{2}{3}\beta_i+\frac{1}{9}[\min(x_a+x_f,\beta_i)+\min(x_b+x_f,\beta_i)+\min(x_c+x_f),\beta_i)]\geq \frac{2}{3}\beta_i$.	

 Therefore, $Eg_i(1,1) + 2Eg_i(1,0)- 2\beta_i\geq 0$ and $Eg_i(1,1)\leq \beta_i$.
      
 For $5 Eg_i(1,1) - 2Eg_i(1,0)- 3\beta_i\geq 0$, if $x_c+x_f\geq \beta_i$, we have $Eg_i(1,1)= \beta_i$ and $Eg_i(1,0)\leq \beta_i$. So $5 Eg_i(1,1) - 2Eg_i(1,0)- 3\beta_i\geq 0$.
      
 If $x_c+x_f\leq \beta_i$, $5Eg_i(1,1) + 2Eg_i(1,0)- 3\beta_i$ is
 growing when $x_f$ increased. So we can let $x_f = 0$ for the worst
 case.  Therefore
 $5 Eg_i(1,1) + 2Eg_i(1,0)- 2\beta_i
 =5(\frac{2}{3}\beta_i+\frac{1}{9}[x_a+x_b+x_c])
 -\frac{1}{3}[x_a+x_b+x_c+x_d+x_e]- 3\beta_i$ which is increasing with
 respect to $x_a,x_b,x_c$ and deceasing with respect to $x_d,x_e$.
 Further more, we have $\frac{3}{2}\beta_i \leq x_a+x_b+x_c $ and
 $x_d+x_e\leq 2\beta_i$.  So
 $5 Eg_i(1,1) + 2Eg_i(1,0)- 2\beta_i\geq 0$
      
\end{proof}

Therefore, we have shown that $Eg_i$ is a conical combination of $f_1$,
$f_2$, $f_3$, $f_4$, $f_5$ for all $i$. Therefore $g = \sum_i Eg_i$ is
a conical combination of $f_1$, $f_2$, $f_3$, $f_4$, $f_5$.

\end{proof}
The 5-vector related to $Eg$ is
$(\phi(1),\phi(2),\phi(2),\phi(3),\phi(4))$. So according to
table~\ref{table:5-vectorforallfi}, the unique expression to expand
$g$ on $f_1$, $f_2$, $f_3$, $f_4$, $f_5$ is
$g(A) = [2\phi(1)+\phi(2)-4\phi(3)+2\phi(4)]f_1
+[-\phi(1)+3.5\phi(2)-4\phi(3)+1.5\phi(4)]f_2
+[-\phi(2)+2\phi(3)-\phi(4)]f_3 +[-\phi(3)+\phi(4)]f_4 +
6[-\phi(2)+2\phi(3)-\phi(4)]f_5 $

This expression is valid if and only if
$-\phi(1)+3.5\phi(2)-4\phi(3)+1.5\phi(4)\geq 0$ and
$2\phi(1)+\phi(2)-4\phi(3)+2\phi(4)\geq 0$; other coefficients are
always non-negative according to concavity and monotonicity.

The ``if'' part is straight forward according to the above expansion
as we saw after the statement of the theorem.
\end{proof}

\section{Sums of Weighted Cardinality Truncations is Smaller than SCMMs}

\label{sec:sums-weight-card}

In this section, show
Lemma~\ref{lemma:sums_weighted_cardinality_truncations}, namely that
$G = \{\sum_{B\subseteq V}\sum_{i=1}^{|B|-1}\alpha_{B,i}\min(|A\cap
B|,i), \; \forall B, i, \alpha_{B,i}\geq 0\}\subset$ SCMM.
We assume the reader is familiar with the notation
in Appendix~\ref{sec:more-gener-cond-1}.

\begin{lemma}
$f_5(A) \notin \{\sum_{B\subseteq V}\sum_{i=1}^{|B|-1}\alpha_{B,i}\min(|A\cap B|,i)|\alpha_{B,i}\geq 0\}$
\end{lemma} 
\begin{proof}
  Assume that
\begin{align}
  f_5(A) = \sum_{B\subseteq V}
  \sum_{i=1}^{|B|-1}\alpha_{B,i}\min(|A\cap B|,i)
  = \sum_{B\subseteq V}\sum_{i=1}^{|B|-1}\alpha_{B,i}E\min(|A\cap
  B|,i).
\end{align}
Note that $f_5$ is fully curved on $\set{a,b,c}$ and
  $\set{d,e,f}$, and these hold for all terms.  So for $B$ and $i$, if
  $i\geq |B\cap \set{a,b,c}|$ or $i\geq |B\cap \set{d,e,f}|$,
  $\alpha_{B,i} = 0$.  Therefore the remaining terms are $f_1$, $f_2$,
  $f_3$, $f_4$, $E\min(|A\cap\set{a,b,d,e}|,1)$ and
  $E\min(|A\cap\set{a,b,c,d,e}|,1)$ Therefore, for all these
  functions, $Ef(2,0) \leq \frac{9}{8}Ef(1,1)$, but
  $Ef_5(2,0) = \frac{6}{5}Ef_5(1,1)$
  (table~\ref{table:5-vectorforallfi}).  So it is impossible to find a
  conical combination of $\min(|A\cap B|,i)$ that equals $f_5$.

\end{proof}

\end{document}